\documentclass{article}
\PassOptionsToPackage{numbers,sort&compress}{natbib}

\usepackage[final,main]{neurips2026}

\usepackage[utf8]{inputenc}
\usepackage[T1]{fontenc}
\usepackage{url}
\usepackage{booktabs}
\usepackage{amsfonts}
\usepackage{amsmath}
\usepackage{amssymb}
\usepackage{mathtools}
\usepackage{amsthm}
\usepackage{nicefrac}
\usepackage{microtype}
\usepackage[dvipsnames,table]{xcolor}
\usepackage{graphicx}
\usepackage{subcaption}
\usepackage{adjustbox}
\usepackage{multirow}
\usepackage{enumitem}
\usepackage{wrapfig}
\usepackage{pifont}
\usepackage{algorithm}
\usepackage{algorithmic}
\usepackage{thmtools}
\usepackage{thm-restate}
\usepackage{hyperref}
\usepackage[capitalize,noabbrev]{cleveref}

\definecolor{tablegray}{rgb}{0.9, 0.9, 0.9}

\newcommand{\pms}[2]{#1{\scriptstyle\pm}#2}

\newcommand{\esci}[3]{\ensuremath{\left(\pms{#1}{#2}\right)\ifnum#3=0\relax\else\mathrm{e}{#3}\fi}}
\newcommand{\LCFM}{\mathcal{L}_{\mathrm{CFM}}}

\theoremstyle{plain}
\newtheorem{theorem}{Theorem}
\newtheorem{lemma}{Lemma}
\newtheorem{proposition}{Proposition}
\newtheorem{corollary}{Corollary}
\theoremstyle{definition}
\newtheorem{definition}{Definition}

\theoremstyle{remark}
\newtheorem{remark}{Remark}

\newif\ifarxiv
\arxivtrue

\title{Improving Function Space Flow Matching with Kernel Optimal Transport}

\author{%
  Fred Xu \\
  University of California, Los Angeles \\
  \texttt{fredxu@cs.ucla.edu} \\
  \And
  Thomas Markovich \\
  Block, Inc. \\
  \texttt{tmarkovich@squareup.com} \\
  \AND
  Barbora Barancikova \\
  Imperial College London \\
  \texttt{b.barancikova23@imperial.ac.uk} \\
  \And
  Yizhou Sun \\
  University of California, Los Angeles \\
  \texttt{yzsun@cs.ucla.edu} \\
}

\begin{document}

\maketitle

\begin{abstract}
Generative models for function-valued data, such as time series and solutions of partial differential equations, must learn distributions over infinite-dimensional spaces. Functional Flow Matching (FFM) extends Flow Matching to this setting, learning a velocity field whose flow transports a Gaussian prior to the data distribution, but it inherits the independent endpoint pairing of standard Flow Matching: in each batch, prior and data samples are matched arbitrarily, so the conditional bridge must traverse both the shared global structure of the dataset and instance-specific residuals. In function space this limitation is harder to fix than in finite dimensions, since optimal transport (OT) on function spaces is delicate to formulate and a flat Euclidean surrogate ignores the geometry that distinguishes function-valued data. We propose kernel Functional Flow Matching (kFFM), which replaces the independent pairing by entropic OT under a kernel-induced cost, the coupling underlying the Hilbert Sinkhorn Divergence (HSD), while leaving the FFM neural-operator architecture unchanged. Theoretically, we prove that the kernel cost and the HSD objective are uniformly bounded and well-posed on Banach ambient spaces, derive an explicit error decomposition against quadratic-cost OT on compact metric spaces that isolates an irreducible kernel-cost mismatch term, and prove a discretization-invariance bound whose rate is governed by Sobolev regularity. Empirically, kFFM improves distributional matching over FFM, diffusion, adversarial, and finite-dimensional OT baselines on time-series and PDE benchmarks, with statistically significant paired-seed gains over FFM and improvements that persist under non-kernel and physics-based diagnostics, including a turbulent Navier--Stokes benchmark. Bounded kernel costs already outperform raw $L^2$ Sinkhorn, and function-space-aware kernels (signature, Sobolev RBF) give further gains on rough or path-valued data.
\end{abstract}

\section{Introduction}

Flow matching \citep{lipman2023flowmatchinggenerativemodeling} is a powerful framework for generative modeling that learns the velocity field of a probability path between a reference and a data distribution. Time series and PDE snapshots are naturally represented as functions, motivating \citet{kerrigan2023functionalflowmatching} to extend this framework to function space through Functional Flow Matching (FFM), which defines a path of Gaussian measures in a Hilbert space and parameterizes the velocity field with a neural operator. FFM enjoys resolution-invariance and competitive performance on sequence and PDE datasets, but inherits the independent endpoint coupling of base FM. Under this coupling, each Gaussian prior sample is paired with an arbitrary data function in a minibatch, and the conditional bridge must traverse both the shared global structure of the dataset and instance-specific residuals; the velocity field is then forced to discover the global structure and match each particular instance simultaneously. A natural fix is geometry-aware coupling: pairing each prior sample with a data sample that is already close in a meaningful geometry, so the bridge primarily transports the residual. This idea has improved finite-dimensional flow matching \citep{tong2024improvinggeneralizingflowbasedgenerative}, but transferring it to function space is non-trivial: direct optimal transport (OT) on function spaces is delicate to formulate and compute \citep{Oxtoby1946InvariantMeasures, DaPrato2006InfiniteDimensionalAnalysis, hosseini2024conditionaloptimaltransportfunction, divol2024optimaltransportmapestimation}, and a flat $L^2$ surrogate ignores the function-space geometry (Sobolev structure, path roughness) that distinguishes one function-valued sample from another.

We resolve this through a kernel-induced Hilbert space surrogate. Working in a Reproducing Kernel Hilbert Space \citep{Sriperumbudur2010Hilbert}, the entropic coupling underlying the Hilbert Sinkhorn Divergence \citep{Li_2021_CVPR} provides a tractable pairing rule whose kernel encodes prior knowledge about function-space geometry: Sobolev structure for PDE snapshots via the RBF kernel on $H^k$, and $p$-variation roughness for time series via the signature kernel \citep{cass2024lecturenotesroughpaths, chevyrev2022signaturemoments}. Our theory certifies this surrogate coupling (well-posedness with uniform bounds, an error decomposition with an irreducible kernel-mismatch term, and discretization invariance); it makes no straightness or sampling-efficiency claim in function space (Section~\ref{sec:exp}, Appendix~\ref{app:straightness}), and the sample-quality gains are an empirical claim that we test with paired-seed statistics, non-kernel metrics, physics diagnostics, and a turbulent Navier--Stokes stress test. The kernel-induced surrogate also outperforms raw $L^2$ Sinkhorn: relative to a flat-$L^2$ minibatch-OT baseline (CFM-OT($L^2$)) with the same Sinkhorn solver and neural-operator backbone, kFFM lowers the distributional error by 75\% on Gene Expression and 66\% on Economy. Function-space-aware kernels (Sobolev RBF, signature) add further gains on rough or path-valued data. Code for every experiment is publicly available.\footnote{\url{https://github.com/heraclixus/KFFM}} Our main contributions are the following:
\begin{enumerate}[topsep=0pt, itemsep=1pt, parsep=1pt, partopsep=0pt]
    \item We propose kernel Functional Flow Matching (kFFM), to our knowledge the first method to perform OT endpoint coupling in function-space flow matching: a kernel-induced cost couples prior and data samples before flow matching, while the FFM neural-operator architecture is unchanged.
    \item We establish three theoretical results for the function-space setting: uniform boundedness and well-posedness of the kernel cost and HSD on Banach ambient spaces; an explicit error decomposition against quadratic-cost OT on compact metric spaces (kernel mismatch, entropic bias, covering); and a discretization-invariance theorem with an $N^{-2(\alpha-k)/d}$ rate that ties the function-space objective to the truncated grid representation the model actually computes.
    \item We show empirically that kFFM improves population-level distributional quality on time series and PDE snapshots over functional generative baselines and a finite-dimensional OT coupling baseline, with paired-seed significance, gains that survive non-kernel and physics-based metrics including a turbulent Navier--Stokes benchmark, and a training overhead of only 5--7\% at the largest resolution we train. The coupling geometry matters differently across data classes: signature kernels for paths, RBF-type costs for fields.
\end{enumerate}

\section{Background}

For the FFM construction we follow \citet{kerrigan2023functionalflowmatching} and work in a separable Hilbert space $\mathcal{F}$ of functions $f : \Omega \to \mathbb{R}$, with $\Omega\subset\mathbb{R}^d$ compact, equipped with its Borel $\sigma$-algebra $\mathcal{B}(\mathcal{F})$; $\nu\in \mathcal{P}(\mathcal{F})$ denotes the data measure. Section \ref{sec:ffm_background} works directly with $\mathcal{F}$, while later subsections and Section \ref{Sec:method} use $\mathcal{X}$ for the metric sample space on which the OT coupling is defined. Additional background is deferred to Appendix \ref{app:a}.

\subsection{Flow Matching in Function Space}\label{sec:ffm_background}

Flow Matching \citep{lipman2023flowmatchinggenerativemodeling} learns the velocity field whose flow transports a prior distribution to the data distribution along a given probability path. In the Gaussian-measure formulation of FFM, a vector field $v : [0, 1] \times \mathcal{F} \to \mathcal{F}$ induces the flow $\partial_t \psi_t(g) = v_t(\psi_t(g))$, $\psi_0(g) = g$, and the path $\mu_t = [\psi_t]_\# \mu_0$ from $\mu_0 \in \mathcal{P}(\mathcal{F})$. FFM constructs this path by averaging endpoint-conditioned Gaussian bridges over $f \sim \nu$: if $\mu_0^f = \mu_0$ and $\mu_1^f$ is concentrated around $f$, then $\mu_t(A) = \int \mu_t^f(A)\,d\nu(f)$ satisfies $\mu_1 = \nu$. For Gaussian bridges with $\mu_0 = \mathcal{N}(0, C_0)$ and $\mu_t^f = \mathcal{N}(m_t^f, (\sigma_t^f)^2C_0)$, the conditional velocity is
\begin{equation*}
    v_t^f(g) = \frac{(\sigma_t^f)'}{\sigma_t^f}(g - m_t^f) + \frac{d}{dt}m_t^f, \qquad g \sim \mu^f_t,
\end{equation*}
and the standard FFM parametrization $m_t^f = tf$, $\sigma_t^f = 1 - (1 - \sigma_{\min})t$ corresponds to Gaussian OT between the bridge marginals, while the endpoint pairing itself remains independent \citep{tong2024improvinggeneralizingflowbasedgenerative}. A neural operator $v_\theta(g,t)$, here a Fourier Neural Operator (FNO) \citep{li2021fourierneuraloperatorparametric}, is trained with the Conditional Flow Matching (CFM) loss
\begin{equation}\label{eq:cfm}
    \mathcal{L}_{CFM}(\theta):=\mathbb{E}_{t\sim U[0,1], f\sim \nu, g\sim \mu_t^f} \left[ \|v_t^f(g)-v_\theta(g,t) \|_{\mathcal{F}}^2\right],
\end{equation}
which under suitable regularity conditions \citep{lipman2023flowmatchinggenerativemodeling,kerrigan2023functionalflowmatching} has the same minimizers as the unconditional loss $\mathcal{L}_{FM}(\theta):=\mathbb{E}_{t\sim U[0,1], g\sim \mu_t} [ \|v_t(g)-v_\theta(g,t) \|_{\mathcal{F}}^2]$; the absolute-continuity assumption behind this equivalence is spelled out in Appendix \ref{app:cameron_martin}. In Section \ref{Sec:method} we keep the same probability path but replace the independent pairing of its endpoints by an OT-coupled pairing.

\subsection{Optimal Transport and Kernel-based Optimal Transport}

For $(\mathcal{X},d)$ a Polish space, $\mu, \nu \in \mathcal{P}(\mathcal{X})$, and a cost $c$, the static OT problem is $\mathcal{W}(\mu, \nu) := \inf_{\pi \in \Pi(\mu, \nu)} \int c\,d\pi$ \citep{Villani2008OptimalTO}, and its entropic regularization \citep{peyre2020computationaloptimaltransport}
\begin{align}\label{eq:ot_entropic_reg}
    W_{\epsilon}(\mu, \nu) := \inf_{\pi\in \Pi(\mu, \nu)} \int c(x,y)\,d\pi(x,y) + \epsilon\,\mathrm{KL}(\pi \,\|\, \mu \otimes \nu)
\end{align}
underlies Sinkhorn methods; the debiased Sinkhorn divergence subtracts the corresponding self-costs \citep{feydy2018interpolatingoptimaltransportmmd, sejourne2019sinkhorndivergencesunbalancedoptimal, flamary2021pot}. OT-based couplings improve finite-dimensional flow matching by computing a static coupling between reference and data samples before interpolation \citep{tong2024improvinggeneralizingflowbasedgenerative, kornilov2024optimalflowmatchinglearning}. Adapting this idea to function-valued data is non-trivial: no Lebesgue-like reference measure exists on function space \citep{Oxtoby1946InvariantMeasures, eldredge2016analysisprobabilityinfinitedimensionalspaces}, so densities are ill-posed, and Gaussian-measure and Cameron--Martin assumptions \citep{DaPrato2006InfiniteDimensionalAnalysis} are limiting when data are non-Gaussian (PDE states, rough paths). This motivates embedding the measures into a Reproducing Kernel Hilbert Space (RKHS) \citep{Sriperumbudur2010Hilbert}, where the coupling is computed under a kernel-induced cost.

\begin{definition}[Hilbert Sinkhorn Divergence \citep{Li_2021_CVPR}] \label{def:hsd}
Let $(\mathcal{X},d)$ be a metric space, $\mu, \nu \in \mathcal{P}(\mathcal{X})$, and $\kappa:\mathcal{X}\times \mathcal{X}\rightarrow \mathbb{R}$ a measurable symmetric positive-definite kernel with RKHS $\mathcal{H}_\kappa$ and canonical feature map $\phi: \mathcal{X}\rightarrow \mathcal{H}_\kappa$. The entropic OT functional on the embedded space is
\begin{align}\label{eq:hsd_h}
    \mathsf{OT}^{\mathcal{H}_\kappa}_{\epsilon}(\phi_\# \mu, \phi_\# \nu)
    :=
    \inf_{\pi_\phi \in \Pi(\phi_\#\mu, \phi_\#\nu)}
    \Big[
    \int\|u-v\|_{\mathcal{H}_\kappa}^2\,d\pi_\phi(u,v)
    + \epsilon\,\mathrm{KL}\!\left(\pi_\phi\,\middle\|\,\phi_\#\mu\otimes\phi_\#\nu\right)
    \Big],
\end{align}
and the Hilbert Sinkhorn divergence is the debiased quantity
$S_\epsilon(\phi_\# \mu, \phi_\# \nu) := \mathsf{OT}^{\mathcal{H}_\kappa}_{\epsilon}(\phi_\# \mu, \phi_\# \nu) -\tfrac12 \mathsf{OT}^{\mathcal{H}_\kappa}_{\epsilon}(\phi_\# \mu, \phi_\# \mu) -\tfrac12 \mathsf{OT}^{\mathcal{H}_\kappa}_{\epsilon}(\phi_\# \nu, \phi_\# \nu)$.
\end{definition}
By the kernel trick, the functional in Equation \eqref{eq:hsd_h} can be written directly on $\mathcal{X}$ with the RKHS-induced cost $c_{\kappa}(x,y) := \|\phi(x)-\phi(y)\|^2_{\mathcal{H}_\kappa} = \kappa(x,x)+\kappa(y,y)-2\kappa(x,y)$ \citep[restated as Proposition~\ref{prop:equiv_hsd_rd} in Appendix~\ref{app:b}]{Li_2021_CVPR}:
\begin{align}\label{eq:hsd_x}
    \mathsf{OT}_{\kappa,\epsilon}(\mu,\nu)
    :=
    \inf_{\pi\in\Pi(\mu,\nu)}
    \Big[
    \int_{\mathcal{X}\times \mathcal{X}} c_{\kappa}(x,y)\, d\pi(x,y)
    +\epsilon\,\mathrm{KL}\!\left(\pi\,\middle\|\,\mu\otimes\nu\right)
    \Big],
\end{align}
with minimizers corresponding under pushforward, and $S_{\kappa,\epsilon}(\mu,\nu) := \mathsf{OT}_{\kappa,\epsilon}(\mu,\nu) - \tfrac12 \mathsf{OT}_{\kappa,\epsilon}(\mu,\mu) - \tfrac12 \mathsf{OT}_{\kappa,\epsilon}(\nu,\nu)=S_{\epsilon}(\phi_\#\mu, \phi_\#\nu)$. Previous work takes $\mathcal{X}\subset\mathbb{R}^d$, where common kernels apply directly, the divergence enjoys statistical consistency with finite-sample bounds (restated in Appendix \ref{app:b}), and the discrete entropic OT problem is convex \citep{feydy2018interpolatingoptimaltransportmmd}. Section \ref{Sec:method} generalizes $\mathcal{X}$ to function-valued sample spaces.

\subsection{Path Signature and Signature Kernel}\label{sec:signature_background}

Let $C_p(\mathbb{R}^d) = C_p([0, T], \mathbb{R}^d)$ denote the continuous paths from $[0,T]$ to $\mathbb{R}^d$ with finite $p$-variation, $p \in [1, 2)$. The signature of $x \in C_p(\mathbb{R}^d)$ is the sequence of iterated Young integrals $S(x) = (1, S^{(1)}(x), S^{(2)}(x), \ldots)$ with
$S^{(m)}(x) = \int_{0 < t_1 < \cdots < t_m < T} dx_{t_1} \otimes \cdots \otimes dx_{t_m} \in (\mathbb{R}^d)^{\otimes m}$, an element of the free tensor algebra $T((\mathbb{R}^d)) = \prod_{n=0}^\infty (\mathbb R^d)^{\otimes n}$.
\begin{definition}[Signature Kernel \citep{cass2024lecturenotesroughpaths}]\label{def:sig_ker}
    The signature kernel $\kappa_{\mathrm{sig}}: C_p(\mathbb{R}^d)\times C_p(\mathbb{R}^d)\rightarrow \mathbb{R}$ is
    $\kappa_{\mathrm{sig}}(x,y) = \langle S(x), S(y) \rangle := \sum_{m=0}^\infty \langle S^{(m)}(x), S^{(m)}(y) \rangle_{(\mathbb{R}^d)^{\otimes m}}$, where $\langle \cdot, \cdot \rangle_{(\mathbb{R}^d)^{\otimes m}}$ is the Hilbert--Schmidt inner product.
\end{definition}
The signature kernel is symmetric and positive definite and hence induces a unique RKHS $\mathcal{H}_{\mathrm{sig}}$; on compact subsets, the time-augmented signature kernel is universal and characteristic \citep{cass2024lecturenotesroughpaths}. We compute it with \texttt{pySigLib} \citep{shmelev2025pysiglibfastsignaturebased}.

\section{Functional Flow Matching with Kernel Optimal Transport}\label{Sec:method}

Throughout this section $\mathcal{X}$ is a general metric space, the sample space on which data and reference measures live and on which the kernel OT coupling acts.

\subsection{Coupling with Kernel Optimal Transport}\label{sec:coupling}

kFFM first couples prior and data samples through entropic OT under a kernel-induced cost, Equation \eqref{eq:hsd_x}. The kernel reflects the metric space where the data live. On the Sobolev space $\mathcal{X} = H^k(\Omega)$ (functions with $k$ weak derivatives in $L^2$, the natural home of PDE solutions \citep{Evans2010}), the RBF kernel on $H^k$ rewards pairs close in a derivative-aware norm; on path space $\mathcal{X} \subset C_p([0,T],\mathbb{R}^d)$, $p\in[1,2)$ (smooth and rough time series \citep{cass2024lecturenotesroughpaths}), the signature kernel rewards reparametrization-invariant path similarity. Well-posedness for each choice is established in Section \ref{sec:theory} and validated in Section \ref{sec:exp}; with empirical measures the minibatch problem is solved by the Sinkhorn algorithm \citep{flamary2021pot}.

Given a plan $\pi$ between prior draws $f_0 \sim \mu_0$ and data samples $f_1 \sim \mu_1$, kFFM uses the same probability path as FFM, with the paired prior sample in the role of the base noise:
\begin{gather*}
    g_t^{f_0,f_1} = t f_1 + \sigma_t f_0, \qquad \sigma_t = 1-(1-\sigma_{\min})t, \qquad (f_0,f_1)\sim\pi, \\
    v_t^{f_0,f_1}\big(g_t^{f_0,f_1}\big) = \partial_t g_t^{f_0,f_1} = f_1-(1-\sigma_{\min})f_0 = \frac{f_1-(1-\sigma_{\min})\,g_t^{f_0,f_1}}{\sigma_t},
\end{gather*}
so that $g_0=f_0\sim\mu_0$, $g_1=f_1+\sigma_{\min}f_0$, and the conditional velocity is the FFM conditional field of Section \ref{sec:ffm_background} evaluated with the coupled pair; the FFM path $tf+\sigma_t\xi$ is recovered exactly for the independent coupling $\pi=\mu_0\otimes\mu_1$.

\textbf{What the coupling changes.} Up to the $\sigma_{\min}$ term, the regression target is the displacement $u := f_1-f_0$, the \emph{residual} between the paired data and prior samples, which carries all the endpoint information. Under independent coupling this residual is the difference of two unrelated draws, a large and high-variance regression target; under kernel-OT coupling each $f_1$ is paired with a geometrically close $f_0$, so the network regresses smaller, structured residuals. The neural operator (an FNO) is unchanged relative to FFM; only the minibatch coupling, and through it the law of the conditional paths, is modified. Algorithm \ref{alg:mbotcfm} summarizes training; implementation defaults and good practices are collected in Appendix \ref{app:practices}. The reference distribution $\mu_0$ is a Gaussian process (GP) or Gaussian random field \citep{gardner2018gpytorch}, or white noise on the grid where the data are rougher than a smooth GP prior (the base measure is part of the selected configuration, Section \ref{sec:quality}); in practice both measures are replaced by their empirical estimators, with the coupling computed from a kernel-induced cost matrix \citep{Li_2021_CVPR, flamary2021pot}:
\begin{align}\label{eq:emp_est}
    \mu_n = \sum_{i=1}^n \hat{\mu}_i\delta_{x_i}, \qquad \nu_n=\sum_{j=1}^n \hat{\nu}_j\delta_{y_j}, \qquad
    C^{\kappa}_{ij} = \kappa(x_i,x_i) + \kappa(y_j,y_j)-2\kappa(x_i,y_j).
\end{align}

\begin{algorithm}[t]
  \caption{Functional Flow Matching with Kernel Optimal Transport}
  \label{alg:mbotcfm}
  \begin{algorithmic}[1]
    \REQUIRE Prior distribution $\mu_0=\mathcal N(0,C_0)$, data distribution $\mu_1$, kernel function $\kappa$, entropy regularization $\epsilon > 0$, batch size $b$, initial network $v_{\theta}$, $\sigma_{\min}>0$.
    \STATE \textbf{Training loop (repeat until convergence):}
      \STATE \quad Sample batches of size $b$: $f_0 \sim \mu_0$, $f_1 \sim \mu_1$
      \STATE \quad Solve the entropic kernel OT problem in Equation \eqref{eq:hsd_x} with $\kappa, \epsilon$ (cost matrix in Equation \eqref{eq:emp_est}) to get coupling $\pi$.
      \STATE \quad $(f_0, f_1) \sim \pi$
      \STATE \quad $t \sim \mathcal{U}(0, 1),\quad \sigma_t \leftarrow 1-(1-\sigma_{\min})t$
      \STATE \quad $g \leftarrow t f_1 + \sigma_t f_0$ \hfill \textit{// the paired prior sample plays the role of the base noise}
      \STATE \quad $v_t^{f_0,f_1}(g) \leftarrow \big(f_1-(1-\sigma_{\min})\,g\big)/\sigma_t \;=\; f_1-(1-\sigma_{\min})f_0$
      \STATE \quad $\mathcal{L}_{CFM}(\theta) \gets \left\| v_t^{f_0,f_1}(g) - v_\theta(g,t)\right\|_{\mathcal{F}}^2$
      \STATE \quad $\theta \gets \mathrm{Update}\!\left(\theta, \nabla_\theta \LCFM(\theta)\right)$
    \STATE \textbf{Output:} $v_{\theta}$.
  \end{algorithmic}
\end{algorithm}

\textbf{Where the kernel enters.} The kernel enters at exactly one place: the cost $c_\kappa$ used by minibatch Sinkhorn to \emph{pair} prior and data samples. No sample or model input is projected into the RKHS; the FNO acts on the original discretized functions, the paths are constructed in the original space, and generation happens there. The training coupling is the entropic plan induced by $c_\kappa$, computed per minibatch with log-domain Sinkhorn on a median-normalized cost and sampled from the plan (Appendix \ref{app:practices}); the debiasing self-terms of the HSD (Definition \ref{def:hsd}) do not depend on $\pi$ and cannot change the argmin plan, so HSD is our \emph{analysis} object (what Theorems \ref{thm:regularity}--\ref{thm:approx_error_correct} control), not the training objective. Finally, at a fixed batch size $b$ minibatch OT induces the expected minibatch plan $\bar\pi_b=\mathbb{E}[\hat\pi_b]$ \citep{fatras2020minibatchwasserstein, fatras2021minibatchot}, a valid coupling of the true marginals for every $b$, so the CFM consistency argument applies and $b$ only changes which conditional paths are used; as $b\to\infty$, $\bar\pi_b$ converges to the population entropic kernel-OT plan.

\subsection{Theoretical Results for Sobolev and Path Space}\label{sec:theory}

The theory certifies the surrogate coupling; it does not by itself explain why sample quality improves, which Section \ref{sec:exp} establishes empirically. It answers three questions (proofs in Appendix \ref{app:b}): is the kernel-OT objective well-defined and well-conditioned on the function spaces used by kFFM; how far is the RKHS-induced cost from quadratic-cost OT, and which parts of the gap are reducible; and does the continuum formulation survive the Fourier truncation used by the FNO and the Sinkhorn solver?

\begin{restatable}[Well-posedness and uniform boundedness]{theorem}{regularitythm}\label{thm:regularity} Let $(\mathcal X,\|\cdot\|_{\mathcal X})$ be a real Banach space and let
$d_{\mathcal X}(f,g):=\|f-g\|_{\mathcal X}$.
Let $\mu,\nu\in\mathcal P(\mathcal X)$ and let $\kappa:\mathcal X\times\mathcal X\to\mathbb R$ be measurable, symmetric, and positive definite with RKHS $\mathcal H_\kappa$.
Assume additionally that $\kappa$ is bounded:
$B_\kappa:=\sup_{f\in\mathcal X} \kappa(f,f) < \infty$.
Then the kernel cost satisfies $0\le c_\kappa(f,g)\le 4B_\kappa$ for all $f,g\in\mathcal X$, and for every $\varepsilon>0$ the Hilbert Sinkhorn divergence
$S_{\kappa,\varepsilon}(\mu,\nu)$ (Definition~\ref{def:hsd}) is finite for all $\mu,\nu\in\mathcal P(\mathcal X)$, with $|S_{\kappa,\varepsilon}(\mu,\nu)|\le 8B_\kappa$.
\end{restatable}

Finiteness per se is not the binding constraint: under a second-moment assumption on $\nu$ the $L^2$ cost is already integrable, and for the Gaussian prior $\mathbb{E}\|f_0\|^2=\operatorname{tr}(C_0)<\infty$. Theorem \ref{thm:regularity} delivers more. (a) \emph{Uniform} boundedness with no assumption on the measures, which conditions the entropic solver: the Gibbs kernel $e^{-c_\kappa/\varepsilon}$ stays bounded away from $0$ and $1$ across batches, so Sinkhorn converges independently of outlier pairs, unlike unbounded $L^2$ costs on heavy-tailed batches (Section \ref{sec:efficiency} quantifies the resulting speed-up). (b) Generality: the statement holds on Banach spaces without Hilbert structure, which licenses the signature kernel on path space, our strongest configuration. (c) Geometry injection: the kernel is the interface through which domain structure enters the coupling. For Sobolev data in $H^{k}(\Omega)$ we use $\kappa_{\mathrm{RBF}}^{(k)}(x,y) := \exp(-\|x-y\|^2_{H^k}/2\sigma^2)$, which is bounded and positive definite \citep{Christmann_kernel_2010, ziegel2022characteristickernelshilbertspaces}; for paths we use the signature kernel (Definition \ref{def:sig_ker}), universal and characteristic on compact sets \citep{cass2024lecturenotesroughpaths}, bounded there by continuity, and used for two-sample tests on path data \citep{chevyrev2022signaturemoments}, applied to time-augmented paths (Appendix \ref{app:a}, Definition \ref{def:time_augmentation}) to avoid its tree-like equivalence.

The entropic OT functional transports with the cost induced by the feature map rather than the original metric cost. On compact subsets of the ambient space ($B_R^{H^\alpha}$ is compact in $H^k$ for $\alpha>k$; compact sets of paths for the signature kernel, with truncation discussed in Appendix \ref{app:b}), the next theorem separates the resulting kernel-cost mismatch from entropic and covering effects.

\begin{restatable}[Error decomposition]{theorem}{approximationthm}
\label{thm:approx_error_correct}
Let $(\mathcal{X},d)$ be a compact metric space and let $\mu,\nu\in\mathcal P(\mathcal X)$. Define the target quadratic cost $c(x,y)=d(x,y)^2$. Let $\kappa:\mathcal X\times\mathcal X\to\mathbb R$ be a bounded positive-definite kernel with RKHS-induced cost
$c_\kappa(x,y):=\kappa(x,x)+\kappa(y,y)-2\kappa(x,y)$,
and let
$\Delta_\kappa:=\sup_{x,y\in\mathcal X}|c_\kappa(x,y)-c(x,y)|$.
Let $\mathsf{OT}_{\kappa,\varepsilon}$ denote the entropic OT functional in Equation \eqref{eq:hsd_x} with cost $c_\kappa$, and let $M:=\operatorname{diam}(\mathcal X)<\infty$. Then for every $\varepsilon>0$ and every $\delta\in(0,M]$:
\begin{align}\label{eq:bound_abstract}
\left|\mathsf{OT}_{\kappa,\varepsilon}(\mu,\nu)-\mathcal W(\mu,\nu)\right|
\leq
\Delta_\kappa +
2\varepsilon\log \mathcal N(\mathcal{X},\delta;d)
+
8M\delta,
\end{align}
where $\mathcal N(\mathcal{X},\delta;d)$ is the $\delta$-covering number of $\mathcal{X}$ in the metric $d$.
\end{restatable}

Theorem \ref{thm:approx_error_correct} is an error decomposition, not an approximation guarantee: the entropic and covering terms vanish as $\varepsilon,\delta\to0$, but $\Delta_\kappa$ is \emph{irreducible} and does not shrink with more samples or finer grids. For the RBF kernel, $c_{\mathrm{RBF}}=2(1-e^{-\|f-g\|^2/2\sigma^2})$ is a rescaled $L^2$ cost up to relative error $O((D/\sigma)^2)$, with $D$ the batch diameter; with median-normalized costs $\Delta_\kappa$ is small at large bandwidth, and on smooth Navier--Stokes fields the induced plans agree across costs (total variation $\le0.02$; Appendix \ref{app:mismatch}). For the signature kernel $\Delta_\kappa$ is large by design: the theorem then controls the deviation from Wasserstein with the \emph{chosen} cost, not proximity to $L^2$ OT, and the pair rankings differ measurably from $L^2$ (rank correlation $0.33$--$0.37$ on AEMET and gene expression), which is where the largest gains occur. In low-regularity regimes the bound degrades on two fronts: $\Delta_\kappa$ grows because a smooth kernel cost saturates on rough functions, and the covering number blows up as the effective smoothness approaches the embedding threshold, making the entropic term large for any practical $\varepsilon$. The benchmarks of Section \ref{sec:quality} sit in the smooth regime where the bound is informative; Section \ref{sec:turbulent} stress-tests the turbulent regime empirically.

Theorems \ref{thm:regularity} and \ref{thm:approx_error_correct} are stated at the continuum level, while the model operates on a finite-rank spectral representation: PDE snapshots are represented in a Fourier subspace of dimension proportional to $N$, which both the FNO and the Sinkhorn solver consume. Our final theorem quantifies the resulting gap at a rate governed by Sobolev regularity; the path-space analogue for signature truncation is discussed in Appendix \ref{app:b}.

\begin{restatable}[Discretization invariance]{theorem}{discretizationthm}\label{thm:discretization}
Let $\Omega\subset\mathbb R^d$ be a bounded periodic domain and let $\alpha>k\ge 0$. Set $\mathcal X:=B_R^{H^\alpha}:=\{f\in H^\alpha(\Omega):\|f\|_{H^\alpha}\le R\}$, viewed as a compact subset of $H^k(\Omega)$, and let $\mu,\nu\in\mathcal P(\mathcal X)$. For the Fourier basis $\{e_n\}_{n\in\mathbb Z^d}$, let $V_N:=\mathrm{span}\{e_n:|n|^2\le N^{2/d}\}$ and let $P_N:H^\alpha(\Omega)\to V_N$ denote the orthogonal projection; hence $\dim V_N\asymp N$. Write $\mu_N:=(P_N)_\#\mu$ and $\nu_N:=(P_N)_\#\nu$. Define the Sobolev-RBF kernel $\kappa_{\mathrm{RBF}}^{(k)}(f,g):=\exp(-\|f-g\|^2_{H^k}/2\sigma^2)$ on $H^k$, and write $\mathsf{KOT}^{(k)}_{\varepsilon}:=\mathsf{OT}_{\kappa_{\mathrm{RBF}}^{(k)},\varepsilon}$ for the corresponding entropic kernel-OT functional from Theorem \ref{thm:approx_error_correct}. Then for every $\varepsilon>0$,
\[
\bigl|\mathsf{KOT}^{(k)}_{\varepsilon}(\mu,\nu) - \mathsf{KOT}^{(k)}_{\varepsilon}(\mu_N,\nu_N)\bigr|
\;\le\;
\frac{4R^2}{\sigma^2}\, N^{-2(\alpha-k)/d}.
\]
Equivalently, in terms of grid spacing $h=N^{-1/d}$, the discretization error decays as $O(h^{2(\alpha-k)})$.
\end{restatable}

Combining Theorems \ref{thm:approx_error_correct} and \ref{thm:discretization} bounds the gap between the computed finite-rank kernel-OT objective and continuum quadratic OT in the $H^k$ metric by four interpretable terms: kernel-cost mismatch, entropic regularization, covering, and discretization (Corollary \ref{cor:composed} in Appendix \ref{app:b_composed}). Only the discretization term vanishes with projection rank $N$, at rate $N^{-2(\alpha-k)/d}$ driven by the regularity gap $\alpha-k$, the formal counterpart of the empirical resolution-invariance in Section \ref{sec:exp}; the mismatch term depends only on kernel design (zero for the linear kernel, finite for RBF and controlled by data scale and bandwidth $\sigma$). The theorem relies on compact Sobolev embedding \citep{Evans2010}; our smooth PDE benchmarks sit in that regime \citep{CamliyurtKukavicaVicol2020_AnalyticityBoundary}.

\section{Experiments}\label{sec:exp}

\textbf{Datasets.} We use two classes of data. \emph{Sequence datasets}: AEMET \citep{FebreroBandeOviedoDeLaFuente2012fdausc}, gene expression \citep{Orlando2008GlobalControl}, and the economics dataset \citep{BoltVanZanden2020MaddisonStyle2020Update} from \citet{kerrigan2023functionalflowmatching}, plus financial time series from the Heston stochastic volatility model \citep{Heston1993ClosedFormSV}, treated as realizations of measures on path space $C_p(\mathbb{R})$. \emph{PDE snapshots}: the 1D Korteweg--de Vries (KdV) equation, the incompressible 2D Navier--Stokes equation \citep{Li2022LearningChaoticDynamics}, and their stochastic counterparts \citep{salvi2022neuralstochasticpdesresolutioninvariant}, treated as measures on Sobolev space $H^k(\Omega)$; Section \ref{sec:turbulent} adds a turbulent 2D Navier--Stokes benchmark.

\textbf{Baselines.} Denoising Diffusion Operator (DDO) with NCSN noise scale \citep{Lim2023aScoreBasedFunctionSpace}, functional DDPM \citep{Kerrigan2023DiffusionInfiniteDimensions}, Generative Adversarial Neural Operator (GANO) \citep{Rahman2022GANO}, and the original FFM \citep{kerrigan2023functionalflowmatching}. We also include a direct finite-dimensional OT baseline, CFM-OT($L^2$) \citep{tong2024improvinggeneralizingflowbasedgenerative}, which uses the same minibatch Sinkhorn solver and FNO backbone as our method but applies entropic OT to flattened grid vectors with the unbounded $L^2$ cost and the white-noise base of finite-dimensional flow matching; kFFM's Euclidean--RBF configuration instead uses the bounded RBF cost on Euclidean distances. kFFM variants are named by cost, kFFM-Sig (signature), kFFM-Sob (Sobolev-RBF), kFFM-RBF (Euclidean--RBF), kFFM-Euc (raw $L^2$), and tagged by base measure, gp (GP prior) or wn (white noise). Datasets, architectures, and training budgets, identical across methods, are detailed in Appendix \ref{app:c}.

\textbf{Metrics.} The maximum mean discrepancy with an RBF kernel (MMD-RBF) is the primary population-level metric, computed with the unbiased estimator (so small negative values occur). To avoid judging an RBF-type coupling by an RBF-MMD alone, we also report non-kernel metrics (sliced Wasserstein distance and marginal $W_1$), the pointwise diagnostics of \citet{kerrigan2023functionalflowmatching} (autocorrelation error on sequences, log-spectral distance on PDEs, mean/variance summaries in Appendix \ref{app:c}), and, on Navier--Stokes, physics diagnostics that capture non-Gaussian structure: enstrophy-distribution $W_1$, vorticity-PDF $W_1$, skewness and excess-kurtosis errors of the pooled vorticity, and log-spectral error. Results aggregate 10 seeds (20 where marked); all methods share seeds, so tests are paired by seed.

\subsection{Distributional Quality}\label{sec:quality}

Table \ref{tab:kernel_vs_l2_main} reports MMD-RBF against all baseline families. The ``kFFM (selected)'' column reports, per dataset, the configuration (OT cost and Gaussian base measure) selected on a held-out validation split by sliced Wasserstein distance, a non-kernel criterion (Appendix \ref{app:selection}); pointwise diagnostics by kernel variant appear in Appendix \ref{app:detailed}. kFFM is the strongest method on every row, and the selected configurations follow a practitioner-usable pattern: the signature kernel with a GP base on path-space datasets, where function-space geometry matters most (except short-horizon Heston, where the top configurations are statistically indistinguishable and the rule picks the bounded RBF cost with a white-noise base); the Sobolev-RBF cost with a GP base on the KdV family; and the bounded Euclidean--RBF cost with a white-noise base on the Navier--Stokes family, whose rough fields are poorly matched by a smooth GP prior.

\begin{table}[t]
\small
\setlength{\tabcolsep}{3pt}
\renewcommand{\arraystretch}{0.98}
\centering
\caption{Population-level distributional quality measured by MMD-RBF (lower is better; unbiased estimator, so small negative values occur). ``kFFM (selected)'' is the per-dataset configuration chosen on a held-out validation split by sliced Wasserstein distance (Appendix~\ref{app:selection}); each cell is tagged (cost/base) with Sig = signature kernel, RBF = Euclidean--RBF kernel, Sob-RBF = Sobolev-RBF kernel, gp = Gaussian-process prior, wn = white noise. The kFFM and FFM columns are the matched-seed runs of Table~\ref{tab:paired_tests} ($n$ as given there); the remaining baselines are averaged over 10 seeds. ``Gain (\%)'' is relative to the strongest non-kFFM baseline in the row and is omitted (--) where that baseline sits at the noise floor of the unbiased estimator (non-positive MMD), where ratios are not meaningful; Table~\ref{tab:paired_tests} reports paired tests instead.}
\label{tab:kernel_vs_l2_main}
\resizebox{\linewidth}{!}{%
\begin{tabular}{lccccccc}
\toprule
Dataset & kFFM (selected) & CFM-OT($L^2$) & FFM & DDPM & DDO/NCSN & GANO & Gain (\%) \\
\midrule
Heston & \textbf{kFFM (RBF/wn)}: $-1.30 \times 10^{-4}$ & $-8.26 \times 10^{-5}$ & $3.10 \times 10^{-4}$ & $1.10 \times 10^{-3}$ & $5.53 \times 10^{-2}$ & $3.92 \times 10^{-2}$ & -- \\
AEMET & \textbf{kFFM (Sig/gp)}: $-5.10 \times 10^{-3}$ & $-4.11 \times 10^{-3}$ & $-3.24 \times 10^{-3}$ & $-3.99 \times 10^{-3}$ & $2.44 \times 10^{-1}$ & $5.24 \times 10^{-1}$ & -- \\
KdV & \textbf{kFFM (Sob-RBF/gp)}: $-2.56 \times 10^{-3}$ & $-1.32 \times 10^{-3}$ & $-1.16 \times 10^{-3}$ & $4.84 \times 10^{-2}$ & $2.37 \times 10^{-1}$ & $4.12 \times 10^{-1}$ & -- \\
Economy & \textbf{kFFM (Sig/gp)}: $2.89 \times 10^{-3}$ & $8.46 \times 10^{-3}$ & $7.30 \times 10^{-3}$ & $9.97 \times 10^{-3}$ & $1.57 \times 10^{-1}$ & $8.98 \times 10^{-1}$ & $\mathbf{60.4}$ \\
Gene Expr. & \textbf{kFFM (Sig/gp)}: $7.11 \times 10^{-3}$ & $2.80 \times 10^{-2}$ & $1.78 \times 10^{-2}$ & $5.32 \times 10^{-2}$ & $3.42 \times 10^{-2}$ & $4.17 \times 10^{-1}$ & $\mathbf{60.1}$ \\
Stoch.\ KdV & \textbf{kFFM (Sob-RBF/gp)}: $1.00 \times 10^{-4}$ & $6.58 \times 10^{-4}$ & $8.80 \times 10^{-4}$ & $3.84 \times 10^{-4}$ & $4.48 \times 10^{-2}$ & $2.39 \times 10^{-2}$ & $\mathbf{74.0}$ \\
Stoch.\ NS & \textbf{kFFM (RBF/wn)}: $2.80 \times 10^{-3}$ & $4.22 \times 10^{-3}$ & $9.68 \times 10^{-2}$ & $1.20 \times 10^{-1}$ & $2.95 \times 10^{-1}$ & $2.52 \times 10^{-1}$ & $\mathbf{33.6}$ \\
Navier--Stokes & \textbf{kFFM (RBF/wn)}: $1.63 \times 10^{-3}$ & $1.78 \times 10^{-3}$ & $1.31 \times 10^{-1}$ & $3.80 \times 10^{-1}$ & $7.35 \times 10^{-1}$ & $3.00 \times 10^{-1}$ & $\mathbf{8.4}$ \\
\bottomrule
\end{tabular}}
\end{table}

\textbf{Uncertainty and paired tests.} Since unbiased MMD values can be tiny or negative, percentage gains are uninformative near the noise floor. Table \ref{tab:paired_tests} therefore reports per-seed mean$\pm$std and seed-paired tests against FFM at matched architecture and training budget: kFFM improves on FFM on every dataset at level $0.05$ under both a paired $t$-test and a Wilcoxon signed-rank test, and wins on every shared seed for the four largest-gap datasets. In the GP-base rows only the coupling changes; in the white-noise rows (Heston, Navier--Stokes family) the base measure changes too, and at that base the effect of the kernel cost is isolated by the comparison against CFM-OT($L^2$) in Table \ref{tab:kernel_vs_l2_main}.

\begin{table}[t]
\small
\setlength{\tabcolsep}{4pt}
\centering
\caption{kFFM (selected) vs.\ FFM with uncertainty, at matched architecture and training budget. MMD-RBF$\times10^{3}$ as mean$\pm$std over $n$ shared seeds (fixed before outcomes); $p$-values from a seed-paired $t$-test and a Wilcoxon signed-rank test. In the gp rows only the coupling changes at a fixed prior; in the wn rows the selected kFFM also uses a white-noise base whereas FFM keeps its GP prior (at that base, the comparison isolating the kernel cost is against CFM-OT($L^2$) in Table~\ref{tab:kernel_vs_l2_main}). Economy averages the population and GDP series within each seed. kFFM wins on every shared seed for the first four rows.}
\label{tab:paired_tests}
\begin{adjustbox}{max width=\linewidth}
\begin{tabular}{lccccc}
\toprule
Dataset & $n$ & FFM & kFFM (selected) & paired $t$ $p$-value & Wilcoxon $p$-value \\
\midrule
Gene Expr.     & 10 & $17.8\pm3.4$   & $\mathbf{7.11\pm2.0}$ (Sig/gp)      & $2.0\times10^{-5}$ & $0.002$ \\
Economy        & 10 & $7.30\pm1.8$   & $\mathbf{2.89\pm0.97}$ (Sig/gp)     & $2.5\times10^{-5}$ & $0.002$ \\
Stoch.\ NS     & 10 & $96.8\pm21$    & $\mathbf{2.80\pm1.3}$ (RBF/wn)      & $1.7\times10^{-7}$ & $0.002$ \\
Navier--Stokes & 10 & $131\pm60$     & $\mathbf{1.63\pm2.6}$ (RBF/wn)      & $8.6\times10^{-5}$ & $0.002$ \\
AEMET          & 20 & $-3.24\pm2.1$  & $\mathbf{-5.10\pm2.6}$ (Sig/gp)     & $0.015$ & $0.014$ \\
Heston         & 20 & $0.31\pm0.80$  & $\mathbf{-0.13\pm0.30}$ (RBF/wn)    & $0.021$ & $0.011$ \\
KdV            & 20 & $-1.16\pm2.1$  & $\mathbf{-2.56\pm1.0}$ (Sob-RBF/gp) & $0.039$ & $0.027$ \\
Stoch.\ KdV    & 20 & $0.88\pm1.1$   & $\mathbf{0.10\pm0.33}$ (Sob-RBF/gp) & $0.037$ & $0.037$ \\
\bottomrule
\end{tabular}
\end{adjustbox}
\end{table}

\textbf{Non-kernel metrics.} The gains survive the non-kernel metrics (Table \ref{tab:nonkernel}): on the same runs, kFFM improves sliced Wasserstein distance and marginal $W_1$ over FFM on every dataset, by $18$--$20\%$ on Gene Expression and Economy, where each comparison is significant when paired by seed, and by $2.9$--$4.1\times$ on Navier--Stokes and Stoch.\ NS; for the signature configurations the coupling and evaluation kernels already differ. The pointwise diagnostics also expose what MMD understates: CFM-OT($L^2$) buys its pointwise-marginal fit with $105\times$ and $66\times$ worse autocorrelation error than kFFM-Sig on Economy and Gene Expression (Table \ref{tab:structure_main}); function-level structure is where the kernel coupling is strongest.

\begin{table}[t]
\small
\setlength{\tabcolsep}{3.5pt}
\centering
\caption{Non-kernel metrics for kFFM (selected configuration in parentheses) vs.\ FFM on the runs of Table~\ref{tab:paired_tests}; lower is better. Sequence rows: mean$\pm$std over 10 shared seeds (Economy averages the population and GDP series within each seed), each comparison significant when paired by seed (Wilcoxon $p=0.002$); Navier--Stokes rows are means. Appendix Table~\ref{tab:nonkernel_full} adds the autocorrelation and log-spectrum errors.}
\label{tab:nonkernel}
\begin{tabular}{lcccc}
\toprule
 & \multicolumn{2}{c}{Sliced-$W$} & \multicolumn{2}{c}{Marginal-$W_1$} \\
\cmidrule(lr){2-3}\cmidrule(lr){4-5}
Dataset & FFM & kFFM & FFM & kFFM \\
\midrule
Gene Expr.\ (Sig/gp) & $0.106\pm0.006$ & $\mathbf{0.087\pm0.006}$ & $0.074\pm0.005$ & $\mathbf{0.059\pm0.004}$ \\
Economy (Sig/gp) & $0.0291\pm0.0019$ & $\mathbf{0.0233\pm0.0012}$ & $0.0171\pm0.0016$ & $\mathbf{0.0137\pm0.0009}$ \\
Navier--Stokes (RBF/wn) & $0.315$ & $\mathbf{0.081}$ & $0.239$ & $\mathbf{0.058}$ \\
Stoch.\ NS (RBF/wn) & $1.06$ & $\mathbf{0.365}$ & $0.795$ & $\mathbf{0.249}$ \\
\bottomrule
\end{tabular}
\end{table}

\textbf{Bounded cost or function-space geometry?} Both effects are real and dominate on different data classes. On smooth fields the RBF coupling is a monotone transform of the $L^2$ cost (rank correlation $\approx1$ on training batches; Appendix \ref{app:mismatch}), so it reuses the $L^2$ geometry while compressing outliers, which buys conditioning, $\epsilon$-robustness, and faster Sinkhorn, and matches or improves the unbounded-$L^2$ coupling at the same base measure (Stoch.\ NS: $2.80$ vs.\ $4.22\times10^{-3}$, Table \ref{tab:kernel_vs_l2_main}). On paths, no Euclidean cost, raw or bounded, recovers the signature gains (matched runs, MMD$\times10^{3}$; Gene Expr.: kFFM-Euc $20.0$, kFFM-RBF $20.7$ vs.\ kFFM-Sig $7.1$; Economy: $7.55$/$7.60$ vs.\ $2.9$), and the signature kernel genuinely re-ranks pairs (rank correlation $0.33$--$0.37$ with $L^2$); a synthetic example in Appendix \ref{app:toy} isolates both failure modes of $L^2$ pairing. Qualitative samples and super-resolution results are shown in Appendix \ref{app:c} (Figures \ref{fig:aemet_main} and \ref{fig:ns_main}).

\textbf{Kernel choice and sensitivity.} Re-aggregating the one-at-a-time sweeps of Appendix \ref{app:sensitivity}: results are indistinguishable within seed-to-seed variation for RBF bandwidths $\sigma\in[0.1,2]$ on every dataset, flat across three orders of magnitude of the Sinkhorn regularization ($\epsilon\in[0.001,1]$) for bounded costs (the raw $L^2$ cost degrades by $\approx10\times$ at $\epsilon\ge0.5$ on long volatility paths), and unchanged across the signature hyperparameters (dyadic order $0$--$3$, static bandwidth $0.1$--$5$); over batch sizes $b\in\{64,\dots,512\}$, kFFM tracks FFM across the $8\times$ range (Appendix \ref{app:practices}). The kernel \emph{family} is the only choice that matters, and the rule we recommend is simple and consistent with the theory: match the kernel to the data class (signature for path-like data; Sobolev or Euclidean RBF for smooth fields, where the Euclidean RBF is a reasonable default; RBF or $L^2$ on very long paths, where signature-kernel memory grows with path length), set the bandwidth by the median heuristic, and validate with a non-kernel metric on a held-out split (Appendix \ref{app:selection}). The properties the theorems require (boundedness, characteristicness) hold for every kernel we consider, so the theory constrains the family only weakly and the selection rule does the rest.

\subsection{Turbulent Regime and Physics Diagnostics}\label{sec:turbulent}

The Navier--Stokes benchmark above ($\nu=10^{-3}$) is smooth, the regime where our theory is informative. To test the low-regularity regime where Theorem \ref{thm:approx_error_correct} is weakest, we run the canonical turbulent FNO benchmark \citep{li2021fourierneuraloperatorparametric}: 2D Navier--Stokes at $\nu=10^{-5}$ ($1200$ trajectories, from which we take the spun-up snapshots at $t\ge10$; $\mathrm{Re}\approx2000$; seven methods, five seeds, identical FNO backbone and training budget). Table \ref{tab:turbulent_ns} reports MMD, sliced-W, and the physics diagnostics; Figure \ref{fig:turbulent_stats} shows the pooled vorticity PDF and energy spectrum, and vorticity samples are in Appendix \ref{app:turbulent}. (i) kFFM (selected) attains the best value on five of seven metrics and ties the best on MMD; held-out sliced-W rejects the GP-base arms ($^\dagger$ rows; $0.19$ vs.\ $\ge0.60$), and among the tied white-noise arms the rule keeps the Euclidean-RBF configuration selected on standard NS; CFM-OT($L^2$) at the same base overlaps kFFM within std, with means favoring kFFM on five of seven metrics. (ii) OT-coupled flow matching is far more robust than the score-based and diffusion alternatives: $23$--$32\times$ better vorticity-PDF $W_1$, $5.7$--$6.9\times$ better enstrophy $W_1$, and $17$--$43\times$ better vorticity-kurtosis error (Figure \ref{fig:turbulent_stats}). (iii) Against FFM run with its GP prior, kFFM improves every metric except kurtosis (a tie), with seed-paired significance (MMD $73\times$ lower, $p=3\times10^{-4}$; sliced-W $3.1\times$, $p=1.3\times10^{-6}$; log-spectrum $3.2\times$, $p=0.005$; vorticity-PDF $2.9\times$, $p=0.007$; skewness $18\times$, $p=0.016$). This gain combines the base-measure change and the coupling: the three GP-base arms are indistinguishable, so in this regime the white-noise base does most of the work, and at that base the kernel and raw-$L^2$ costs tie, as in (i). All methods, ours included, still leave headroom on sliced-W in this regime (Section \ref{sec:conclusion}).

\begin{table}[t]
\footnotesize
\setlength{\tabcolsep}{3.5pt}
\centering
\caption{Turbulent 2D Navier--Stokes ($\nu=10^{-5}$, $1200$ trajectories, $\mathrm{Re}\approx2000$): distributional and physics diagnostics, lower is better, mean$\pm$std over 5 seeds with identical FNO backbone and training budget. Rows marked $^\dagger$ use the GP base measure; kFFM and CFM-OT use the white-noise base chosen by the validation rule (Appendix~\ref{app:selection}). Best per column in bold (on MMD, kFFM and CFM-OT tie within std). Unscaled values and full configurations are in Appendix Table~\ref{tab:turbulent_ns_full}.}
\label{tab:turbulent_ns}
\begin{adjustbox}{max width=\linewidth}
\begin{tabular}{lccccccc}
\toprule
Method & MMD & Sliced-W & Vort-PDF $W_1$ & Enstrophy $W_1$ & Skew err. & Kurt err. & log-Spec \\
 & ($\times10^{-3}$) & & ($\times10^{-2}$) & & ($\times10^{-2}$) & ($\times10^{-2}$) & ($\times10^{-2}$) \\
\midrule
kFFM (selected)   & $\mathbf{1.48\pm0.6}$ & $\mathbf{0.193\pm0.01}$ & $\mathbf{3.8\pm1}$ & $\mathbf{0.155\pm0.02}$ & $\mathbf{0.27\pm0.3}$ & $7.8\pm1$ & $\mathbf{2.0\pm0.8}$ \\
CFM-OT($L^2$)     & $\mathbf{1.47\pm1.0}$ & $0.194\pm0.01$ & $4.1\pm1$ & $0.157\pm0.01$ & $0.40\pm0.2$ & $7.5\pm1$ & $2.2\pm0.7$ \\
kFFM-Sob$^\dagger$ & $111\pm20$ & $0.612\pm0.05$ & $10.4\pm3$ & $0.194\pm0.05$ & $4.6\pm2$ & $\mathbf{6.8\pm5}$ & $6.5\pm1$ \\
kFFM-Euc$^\dagger$ & $109\pm20$ & $0.623\pm0.03$ & $10.4\pm4$ & $0.197\pm0.06$ & $5.0\pm3$ & $7.7\pm4$ & $6.7\pm1$ \\
FFM$^\dagger$      & $108\pm20$ & $0.601\pm0.02$ & $10.7\pm4$ & $0.201\pm0.06$ & $4.9\pm3$ & $7.2\pm5$ & $6.5\pm1$ \\
DDO/NCSN          & $569\pm9$  & $1.39\pm0.06$ & $122\pm1$ & $1.07\pm0.02$ & $0.83\pm1$ & $129\pm1$ & $93.7\pm2$ \\
DDPM              & $380\pm20$ & $1.18\pm0.1$ & $84.8\pm3$ & $0.889\pm0.02$ & $7.9\pm5$ & $337\pm200$ & $11.7\pm3$ \\
\bottomrule
\end{tabular}
\end{adjustbox}
\end{table}

\begin{figure}[t]
\centering
\includegraphics[width=\linewidth]{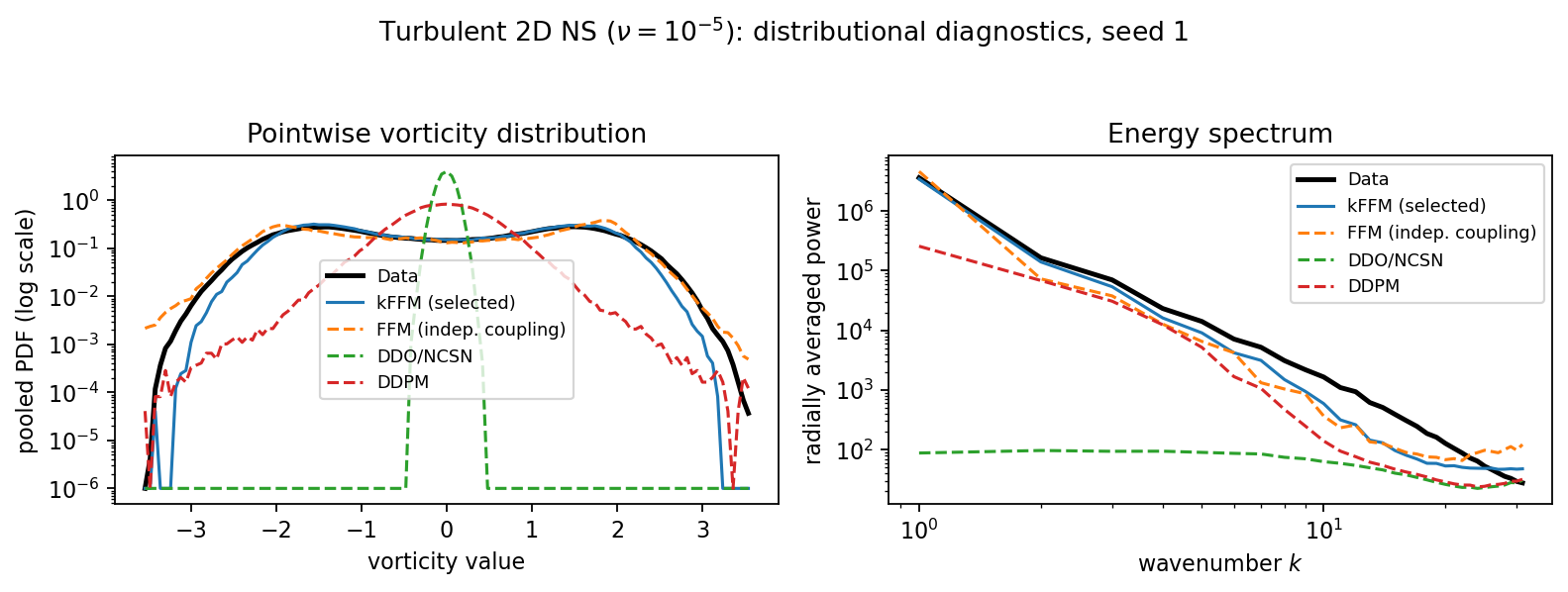}
\caption{Turbulent 2D Navier--Stokes ($\nu=10^{-5}$, seed 1): pooled vorticity PDF (log scale) and radially averaged energy spectrum. The data PDF is bimodal and heavy-tailed; kFFM (selected) tracks it, DDO/NCSN collapses to a narrow near-Gaussian spike, and DDPM is unimodal with the wrong shape. The flow-matching spectra track the data across scales; DDO/NCSN's is nearly flat.}
\label{fig:turbulent_stats}
\end{figure}

\subsection{Cost, Efficiency, and Convergence}\label{sec:efficiency}

\textbf{Compute.} Measured on every benchmark (Appendix \ref{app:compute}), the end-to-end training overhead of the per-minibatch Sinkhorn solve over FFM is $5$--$7\%$ at the largest resolution we train ($128^2$ Navier--Stokes); the $O(b^2)$ coupling cost is independent of model size and resolution, so its relative weight shrinks as the backbone grows. Peak coupling memory is $80$--$151$\,MB ($1.5$--$9\%$ of the FNO's forward-plus-backward peak), and per-batch coupling time is comparable to one FNO step. The bounded RBF cost is about $8\times$ faster than unbounded-$L^2$ Sinkhorn at $b=512$ and trains $15$--$25\%$ faster end-to-end, the conditioning effect of Theorem \ref{thm:regularity}; giving FFM the extra budget as additional epochs does not close the gap.

\textbf{No straightness claim.} We claim no straightness theorem: Appendix \ref{app:straightness} (Proposition \ref{prop:coupling_only}, Corollary \ref{cor:straightness_invariance}) shows that kFFM changes only the endpoint coupling while the affine interpolation template is fixed, and Hilbert-space Gaussian geometry does not inherit the Euclidean straightness picture that motivates finite-dimensional OT-CFM \citep{yun2025gaussianoptimaltransportbreniers}. We therefore make no efficiency claim; Appendix \ref{app:convergence} reports only the training-time convergence of the target metric, where kFFM ends training at a better MMD-RBF than FFM on every dataset.

\section{Related Work}\label{sec:related}

Minibatch OT coupling for flow matching was introduced by \citet{tong2024improvinggeneralizingflowbasedgenerative} and \citet{pooladian2023multisampleflowmatchingstraightening}, with follow-ups refining the transport objective \citep{kornilov2024optimalflowmatchinglearning, yue2025oatfmoptimalaccelerationtransport} and analyses of minibatch plans \citep{fatras2020minibatchwasserstein, fatras2021minibatchot}; all assume $\mathbb{R}^d$ sample spaces. In function space, FFM \citep{kerrigan2023functionalflowmatching} lifts flow matching to Hilbert spaces via Gaussian measures, alongside functional diffusion \citep{Kerrigan2023DiffusionInfiniteDimensions, Lim2023aScoreBasedFunctionSpace} and adversarial neural operators \citep{Rahman2022GANO}. Concurrent function-space flow matching changes the objective, the task, or the prior, but not the coupling: Functional Rectified Flow \citep{zhang2025functionalrectifiedflow} straightens flows by re-training on model-generated endpoint pairs across training rounds, a mechanism complementary to and composable with our within-batch, data-geometry coupling; Operator Flow Matching \citep{lee2025operatorflowmatching} and flow matching with GP priors \citep{kollovieh2025flowmatchinggp} target conditional time-series forecasting with an independent prior-data coupling. Kernel embeddings of measures \citep{Sriperumbudur2010Hilbert, gretton2008kernelmethodtwosampleproblem} underlie MMD and two-sample testing, recently on path space with the signature kernel \citep{chevyrev2022signaturemoments}; Sinkhorn divergences interpolate between OT and MMD \citep{feydy2018interpolatingoptimaltransportmmd}, and the Hilbert Sinkhorn Divergence \citep{Li_2021_CVPR} kernelizes them. To our knowledge, ours is the first work to perform OT endpoint coupling in function-space flow matching; Appendix \ref{app:d} gives an extended discussion, including MMD gradient flows.

\section{Conclusion and Limitations}\label{sec:conclusion}

kFFM replaces the independent endpoint pairing of FFM with a geometry-aware kernel OT coupling while keeping the neural-operator backbone. Theoretically, the kernel-OT surrogate is uniformly bounded and well-posed on Banach spaces, admits an explicit error decomposition against quadratic-cost OT that isolates an irreducible kernel-cost mismatch, and is discretization-invariant at rate $N^{-2(\alpha-k)/d}$. Empirically, kFFM improves distributional matching on time-series and PDE benchmarks with paired-seed significance, and the gains survive non-kernel metrics, physics diagnostics, and a turbulent Navier--Stokes stress test, at a training overhead of $5$--$7\%$ at the largest resolution we train.

\textbf{Limitations.} Several deserve plain statement. First, our theory concerns the surrogate coupling: it certifies well-posedness, an error decomposition, and discretization invariance, but says nothing about straightness or sampling efficiency in function space, on which we make no claim, and it does not by itself explain the sample-quality gains, which remain an empirical finding. Second, the theory is least informative in low-regularity regimes: on the turbulent benchmark the prior-data $L^2$ distances more than double relative to the standard one (median $91$ vs.\ $40$), pushing the RBF cost into saturation, where $\Delta_\kappa$ grows and the covering term of Theorem \ref{thm:approx_error_correct} is large for any practical $\epsilon$; there the empirical results carry the method. Third, turbulent small-scale statistics remain an open problem for function-space generative models. On current evidence the failure sits mainly with the score-based and diffusion baselines (vorticity-kurtosis errors of $1.3$--$3.4$ and enstrophy $W_1\approx1$), but all methods, ours included, leave headroom on sliced-W ($0.19$ for the best method), and in this regime the white-noise base rather than the coupling does most of the work; OT-coupled flow matching is currently the strongest option. Finally, signature-kernel memory grows with path length ($0.63$\,GB at $100$ time points vs.\ $12.5$\,GB at $512$), so very long paths fall back to RBF or $L^2$ costs. Future work includes stronger function-space OT formulations \citep{yun2025gaussianoptimaltransportbreniers, kornilov2024optimalflowmatchinglearning} and base measures and costs designed for rough, turbulent data.

\begin{ack}
Barbora Barancikova is supported by UK Research and Innovation [UKRI Centre for Doctoral Training in AI for Healthcare grant number EP/S023283/1]. Fred Xu is partially supported by NSF grant 2531008 and was supported by Block, Inc.\ during an internship.
\end{ack}

\typeout{MAINTEXTEND=\thepage}
\bibliographystyle{plainnat}
\bibliography{references}

\clearpage
\appendix
\typeout{APPENDIXPAGE=\thepage}

\section{Additional Technical Background}\label{app:a}

This appendix collects the mathematical background on which the paper relies: probability measures on function spaces, reproducing kernel Hilbert spaces and kernel embeddings, path signatures, and optimal transport.

\subsection{Probability Measure in Function Space}

In this section we provide the definitions and necessary theoretical results for probability measure in function space, highlighting the analytical difficulties in general. We give high-level intuition and omit the proofs; for details we refer the reader to the textbooks of \citet{rudin1974functional} and \citet{DaPrato2006InfiniteDimensionalAnalysis}.

\begin{definition}[Hilbert Space]
    A Hilbert space $\mathcal{H}$ is an inner-product space that is complete in the norm induced by its inner product, defined as $\langle \cdot, \cdot \rangle: \mathcal{H}\times \mathcal{H}\rightarrow \mathbb{F}$, where $\mathbb{F} = \mathbb{R}$ for real Hilbert space. 
\end{definition}

Importantly, the Hilbert space's inner product structure gives it a metric space with the norm $||x|| :=\sqrt{\langle x,x\rangle}$ and metric $d(x,y) = ||x-y||, \quad \forall x,y\in \mathcal{H}$. Let $B(x,r)$ be the open balls defined using the metric $d$, then the Borel $\sigma$-algebra of $\mathcal{H}$, denoted $\mathcal{B}(\mathcal{H})$, is generated by the open balls defined in this way:
\begin{align*}
    \mathcal{B}(\mathcal{H}) := \sigma(\{B(x,r):x\in \mathcal{H}, r>0\}).
\end{align*}

Another desirable property of the Hilbert space is eigenfunction spectral decomposition for bounded linear operators:

\begin{theorem}
    Let $T:\mathcal{H}\rightarrow \mathcal{H}$ be a compact, self-adjoint, bounded linear operator. Then there exists an orthonormal set of eigenvectors $\{e_n\}_{n \in I}$ and real eigenvalues $\{\lambda_n\}_{n\in I}$ such that:
    \begin{enumerate}
        \item $Te_n = \lambda_ne_n$
        \item $\lambda_n\rightarrow 0$ if $I$ is infinite.
        \item $\lambda_i \neq 0$ has finite multiplicity.
        \item The closure of span of non-zero eigenvectors produces closure of $Ran(T)$. 
        \item The spectral decomposition of $T$ is:
        \begin{align*}
            Tx = \sum_{i\in I}\lambda_i\langle x,e_i\rangle e_i, \qquad x\in \mathcal{H}
        \end{align*}
    \end{enumerate}
\end{theorem}

In the main text, we adopt a measure-theoretic view of probability, which is crucial for function spaces. This is due to the following result, known as ``no Lebesgue measure in infinite dimensions.'' The statement usually takes the form: If $\mathcal{X}$ is an infinite-dimensional separable normed linear space, then there is no nontrivial, translation-invariant, $\sigma$-additive Borel measure $\mu$ on $\mathcal{X}$ that is finite on every nonempty open ball \citep{yamasaki1985measures}. The implication of this result is that there is no universal definition of \textbf{probability density} as in $\mathbb{R}^d$, and only the more general Radon-Nikodym derivative can be defined between two measures that are absolutely continuous.

\begin{definition}[Singular and Absolute Continuity]
    For a measurable space $(\Omega,\mathcal{F})$, two measures $\mu, \nu$ are mutually singular, written $\mu \perp \nu$, if there exists $A\in \mathcal{F}$, such that $\mu(A) =0$ and $\nu(\Omega \setminus A)=0$. $\mu$ is said to be absolutely continuous w.r.t $\nu$, denoted $\mu \ll \nu$, if $\nu(A)=0 \Longrightarrow\mu(A)=0$. If $\mu \ll  \nu$ and $\nu \ll \mu$, then we say that $\mu \sim \nu$, or that they are equivalent.
\end{definition}

\begin{definition}[Radon-Nikodym Derivative] If $\mu \ll  \nu$ and $\mu,\nu$ are $\sigma-$finite, then there exists a measurable function $f: \Omega\rightarrow [0,\infty]$ such that for all $A\in \mathcal{F}$:
\begin{align*}
    \mu(A) = \int_A fd\nu,
\end{align*}

where $f$ is called the Radon-Nikodym derivative of $\mu$
 w.r.t $\nu$, denoted $\frac{d\mu}{d\nu}$. When the measure space is $(\mathbb{R}^d, \mathcal{B}(\mathbb{R}^d))$ and $\nu$ is the Lebesgue measure, $f$ is also called the probability density function of distribution $\mu$.    
\end{definition}

Though there is no infinite-dimensional Lebesgue measure, and hence no universal way to define density, the most feasible and well-understood way to have a Lebesgue-like measure is to define Gaussian measures: 

\begin{definition}[Gaussian Measure]
    Let $a\in \mathcal{H}$ and $Q \in L_1^+(\mathcal{H})$, the space of positive, symmetric, and trace class linear operators, then the Gaussian measure $\mu:=N_{a,Q}$ on $(\mathcal{H},\mathcal{B}(\mathcal
    H))$ with mean $a$ and covariance operator $Q$ is defined by the Fourier transform: 
    \begin{align*}
        \widehat{N_{a,Q}}(h) = \exp\{i\langle a,h\rangle -\frac{1}{2} \langle Qh, h\rangle\}
    \end{align*}
\end{definition}

It can be shown that for every choice of $a, Q$, there exists a unique Gaussian measure. In \citet{DaPrato2006InfiniteDimensionalAnalysis}, Chapter 1, the interesting argument for constructing this measure is by the theory of infinite product measure and the isomorphism between the sequence space $l^2$ and the Hilbert space $\mathcal{H}$, where the projection map defined using the indices of the spectral decomposition is mapped with Gaussian measures in $\mathbb{R}$ (also known as 1D Gaussian distribution). 

Still, for one Gaussian measure to be absolutely continuous with respect to another, strong conditions have to be imposed on their means and covariance operators: the Feldman--H\'ajek theorem and, for a shifted mean, the Cameron--Martin formula.

\begin{definition}[Cameron--Martin Space]
    Let $\mathcal{H}$ be a Hilbert space and $\mu = N_{0,Q}$ a centered Gaussian measure, then the Cameron--Martin space induced by $N_{0,Q}$ is:
    \begin{align*}
        H_\mu := \text{Ran}(Q^{1/2}), 
    \end{align*}
    with inner product $\langle h_1,h_2\rangle := \langle Q^{-1/2}h_1, Q^{-1/2}h_2\rangle_{\mathcal{H}}$. 
\end{definition}

The general dichotomy for Gaussian measures is the Feldman--H\'ajek theorem:

\begin{theorem}[Feldman--H\'ajek]
    Let $\mathcal{H}$ be a separable Hilbert space and $N_{a,P}, N_{b,Q}$ be two Gaussian measures, then either $N_{a,P} \sim N_{b,Q}$ or $N_{a,P}\perp N_{b,Q}$, and $N_{a,P} \sim N_{b,Q}$ iff:
    \begin{enumerate}
        \item $P^{1/2}(\mathcal{H}) = Q^{1/2}(\mathcal{H})=:\mathcal{H}_0$.
        \item $a-b\in \mathcal{H}_0$.
        \item $(P^{-1/2}Q^{1/2})(P^{-1/2}Q^{1/2})^*-I$ is Hilbert-Schmidt on $\mathcal{H}_0$
    \end{enumerate}
\end{theorem}

\begin{theorem}[Cameron--Martin Formula]
    Let $\mu = N_Q, \nu = N_{a,Q}$ be two Gaussian measures, then: 
    \begin{enumerate}
        \item If $a\notin Q^{1/2}(\mathcal{H})$, then $\mu \perp \nu$.
        \item If $a\in Q^{1/2}(\mathcal{H})$, then $\mu \sim \nu$.
        \item If $\mu \sim \nu$, then with $h:=Q^{-1/2}a$,
        \begin{align*}
            \frac{d\nu}{d\mu}(x) = \exp \left\{ \widehat h(x) -\frac{1}{2}\|h\|_{\mathcal H}^2 \right\},
        \end{align*}
        where $\widehat h$ denotes the Cameron--Martin linear functional associated with $h$.
    \end{enumerate}
\end{theorem}

These are the conditions enforced on the FFM model \citep{kerrigan2023functionalflowmatching}, which relies on the absolute continuity between the marginal and the conditional measure at every time $t$ in the probability flow.

For a general Banach space, which is a complete normed space but need not have an inner product structure or spectral decomposition, the analysis becomes more abstract. For instance, without a coordinate map, the preceding product-measure construction of Gaussian measures no longer applies directly. We refer readers to \citet{LedouxTalagrand1991} for this broader probability theory, and instead explore Hilbert-space embeddings of probability measures in Section \ref{sec:kern_prob}.

\subsection{Reproducing Kernel Hilbert Space}
A Reproducing Kernel Hilbert Space is a Hilbert space of functions with a particular structure that makes each point evaluation well-behaved.

\begin{definition}[Reproducing Kernel Hilbert Space]
Let $\mathcal{X}$ be a non-empty set. A Hilbert space, $\mathcal{H}$ of functions $f: \mathcal{X} \rightarrow \mathbb{R}$ is called a Reproducing Kernel Hilbert Space if for every $x \in \mathcal{X}$, the point evaluation functional $\delta_x: \mathcal{H} \rightarrow \mathbb{R}$ defined by $\delta_x(f) = f(x)$ is continuous and bounded.

By the Riesz representation theorem, there exists a unique element $\kappa_x \in \mathcal{H}$ such that
\begin{align*}
    f(x) = \langle f, \kappa_x \rangle_{\mathcal{H}} \;\; \forall \;\; f \in \mathcal{H}.
\end{align*}
This is the reproducing property that is the defining characteristic of an RKHS.

Since $\kappa_x$ is both a function defined on $\mathcal{X}$ and in $\mathcal{H}$, we have that

\begin{align*}
    \kappa_x(y) = \delta_y (\kappa_x) = \langle \kappa_x, \kappa_y \rangle_{\mathcal{H}},
\end{align*}
where $\kappa_y \in \mathcal{H}$ is the element in $\mathcal{H}$ associated with $\delta_y$. This allows us to define the reproducing kernel, $\kappa(x, y)$ as:
\begin{align*}
    \kappa(x, y) = \langle \kappa_x, \kappa_y \rangle_{\mathcal{H}}.
\end{align*}

From this, it is easy to see that $\kappa: \mathcal{X} \times \mathcal{X} \rightarrow \mathbb{R}$ is both symmetric and positive definite.
\end{definition}

Arguably the core reason to use kernels is that they implicitly define feature maps in high (and potentially infinite) dimensional spaces, which is known colloquially as the ``kernel trick''. We next formalize this property using the properties of the RKHS:

\begin{definition}[Canonical Feature Maps]
    Let $\kappa$ be a kernel with RKHS $\mathcal{H}_\kappa$. The canonical feature map $\phi: \mathcal{X} \rightarrow \mathcal{H}_\kappa$ is defined by
\begin{align*}
    \phi(x) = \kappa(\cdot, x),
\end{align*}
which satisfies the fundamental identity $\kappa(x, y) = \langle \phi(x), \phi(y) \rangle$, which is known as the ``kernel trick''. This allows us to compute inner products in the feature space without explicit construction. 
\end{definition}

In the context of our work, we can understand the cost function in Proposition~\ref{prop:equiv_hsd_rd} by making use of the kernel trick.

For some of our kernel choices, it is useful to know when the RKHS is rich enough to approximate arbitrary continuous functions and when the induced mean embedding distinguishes probability distributions. Kernels with the former property are referred to as ``universal'', defined as follows:

\begin{definition}[Universal Kernel]
\label{def:universal_kernel}
    Let $\mathcal{X}$ be a compact metric space. A continuous kernel on $\mathcal{X}$ is called universal if its associated RKHS, $\mathcal{H}_\kappa$ is dense in $C(\mathcal{X})$ with respect to the supremum norm, $|| \cdot ||_{\infty}$. That is to say that for any $f \in C(\mathcal{X})$ and for any $\epsilon > 0$, there exists $g \in \mathcal{H}_\kappa$ such that

\begin{align*}
    \sup_{x \in \mathcal{X}} \left| f(x) - g(x) \right| < \epsilon.
\end{align*}

\end{definition}

Universality ensures that a kernel can approximate continuous functions, while characteristicness ensures that its mean embedding can distinguish probability distributions. A kernel with the latter property is called ``characteristic'':

\begin{definition}[Characteristic Kernel]\label{def:char}
    Let $\mathcal{X}$ be a compact metric space. A kernel $\kappa$ on $\mathcal{X}$ is characteristic if the mean embedding map
    \[
        \mu \mapsto \int_{\mathcal{X}} \kappa(\cdot,x)\,d\mu(x)
    \]
    is injective on the space of probability measures $\mathcal P(\mathcal{X})$. Equivalently, $\kappa$ is characteristic if

    \begin{align*}
        \int\int \kappa(x, y) d\mu(x) d\mu(y) - 2\int\int \kappa(x, y) d\mu(x) d\nu(y) + \int\int \kappa(x, y) d\nu(x) d\nu(y) = 0
    \end{align*}
implies that $\mu = \nu$. 
\end{definition}

Universal kernels on compact spaces are characteristic, but the converse does not hold in general. These properties are useful for interpreting kernel discrepancies: characteristicness makes the kernel mean embedding injective, while universality gives a stronger approximation property on compact sets. The main theorems in Section \ref{sec:theory} only require the specific assumptions stated there, namely bounded positive-definite kernels, compactness for the error decomposition, and Sobolev regularity for discretization invariance.

The radial basis function kernel is used in our Sobolev-space experiments. On a Sobolev space $H^k(\Omega)$ it is defined as

\begin{align*}
    \kappa_{\mathrm{RBF}}^{(k)}(f, g) = \exp\left\{-\frac{\| f - g \|_{H^k}^2}{2\sigma^2}\right\},
\end{align*}
where $\sigma > 0$ is the bandwidth parameter. This Gaussian kernel is symmetric, positive definite, and characteristic on separable Hilbert spaces \citep{ziegel2022characteristickernelshilbertspaces}. It is also bounded, $\kappa_{\mathrm{RBF}}^{(k)}(f,f)=1$, satisfying the boundedness condition used in Theorems~\ref{thm:regularity} and \ref{thm:approx_error_correct}.

\subsection{Kernel Embedding of Probability Measure}\label{sec:kern_prob}

We next turn our attention to the concept of embedding probability measures into a RKHS, following the framework developed by \citet{Sriperumbudur2010Hilbert}. This is central to our approach because it gives a tractable way to define kernel-induced transport costs on function spaces. As previously discussed, infinite-dimensional spaces lack a Lebesgue-like reference measure, making density-based approaches unavailable. The key insight is that we can embed probability measures as elements into a RKHS, where we compute distances between distributions without requiring explicit density estimation.

For this embedding to be well defined, we require an integrability condition on the kernel. The following proposition shows that bounded kernels ensure the embedding is valid for all probability measures~\citep{Sriperumbudur2010Hilbert}:

\begin{proposition}
    Let $\kappa$ be a measurable kernel on $\mathcal{X}$. Then, $\int_{\mathcal{X}} \sqrt{\kappa(x, x)} \; dP(x) \; < \; \infty \;\; \forall \;\; P \in \mathcal{P}(\mathcal{X})$ if and only if $\kappa$ is bounded. 
\end{proposition}

With the mean embedding in hand, we can define a natural notion of distance between probability measures. This construct relies on the concept of characteristic kernel (Definition \ref{def:char}). If the kernel is characteristic, then the metrics $\gamma_\kappa$ can be defined  such that:
\begin{align*}
    \gamma_\kappa(\mathbb{P},\mathbb{Q}) = 0 \Longleftrightarrow \mathbb{P} = \mathbb{Q}, \mathbb{P,Q}\in \mathcal{P}
\end{align*}

The bounded condition and the characteristic condition gives the following result, which is Theorem 1 in \citet{Sriperumbudur2010Hilbert}: 
\begin{theorem} 
Let $\mathcal{P}_\kappa := \{\mathbb{P}\in \mathcal{P}: \int_M \sqrt{\kappa(x,x)}d\mathbb{P}(x) <\infty\}$, where $\kappa$ is measurable on $M$, then for any $\mathbb{P,Q}\in \mathcal{P}_\kappa$: 
\begin{align*}
    \gamma_\kappa(\mathbb{P}, \mathbb{Q}) = \left|\left| \int_M \kappa(\cdot,x)d\mathbb{P}(x) - \int_M \kappa(\cdot, x) d\mathbb{Q}(x) \right|\right|_{\mathcal{H}} := ||\mathbb{P}\kappa - \mathbb{Q}\kappa||_{\mathcal{H}}
\end{align*}
\end{theorem}

Below we show that the bounded RBF kernel defined on a Sobolev space $H^k$ is characteristic, to provide additional support for the kernel choices in the main text. First we state a lemma proved in \citet{Sriperumbudur2010Hilbert}: 
\begin{lemma}\label{lem:inte_pos}
    Let $\kappa$ be an integrally strictly positive definite kernel on a topological space $M$. Then $\kappa$ is characteristic to $\mathcal{P}$. 
\end{lemma}

\begin{definition}
    Let $M(H)$ be the finite signed Borel measures on $H$. A bounded measurable kernel $\kappa$ on $H$ is \textit{integrally strictly positive definite (ISPD)} if for every $\mu \in M(H) \setminus \{0\}$, 
    \begin{align*}
        \int\int_{H\times H} \kappa(x,y)d\mu(x)d\mu(y) > 0,
    \end{align*}
    or equivalently, if $m_\kappa(\mu):=\int \kappa(\cdot,x)d\mu(x)$ is the \textit{kernel mean element} (KME), then:
    \begin{align*}
        \int\int \kappa(x,y)d\mu(x)d\mu(y) = || m_\kappa(\mu)||^2_{\mathcal{H}_\kappa}.
    \end{align*}
\end{definition}
One can interpret ISPD as the statement that a zero KME implies a zero measure: $m_\kappa(\mu) = 0 \Longrightarrow \mu = 0$.

We then reference Theorem 3.1 from \citet{ziegel2022characteristickernelshilbertspaces}, which shows that RBF / Gaussian kernel is ISPD on separable Hilbert space:
\begin{theorem}
    Let $H$ be a separable Hilbert space, then the Gaussian kernel on $H$ is ISPD with respect to $M(H)$.
\end{theorem}

This theorem and Lemma \ref{lem:inte_pos} therefore imply that the RBF kernel is characteristic on our Sobolev spaces of interest.

\subsection{Path Space and Signature Kernel}

Denote by \(\mathcal{C}_p\) the space of continuous paths \(x:[a,b]\to V\) of finite \(p\)-variation (see Definition~1.1.4 in \citet{cass2024lecturenotesroughpaths}), with \(p \in [1,2)\). \(V\) is a finite-dimensional Banach space. For time series datasets, we usually have \(x: [0,T] \rightarrow \mathbb{R}^d\). For discrete observations, one may work with a piecewise-linear interpolation, which is again of finite \(p\)-variation.
Let \(\mathcal{T}\) be the tensor algebra (written $T((\mathbb{R}^d))$ in Section \ref{sec:signature_background}, where $V=\mathbb{R}^d$)
\begin{align*}
\mathcal{T} \;:=\; \prod_{m=0}^{\infty}V^{\otimes m} .
\end{align*}

\begin{definition}[The signature transform]
    The \emph{signature transform} \(S:\mathcal{C}_p\to\mathcal{T}\) maps a path \(x\) to the sequence of tensors
    \begin{align*}
    S(x) \;=\; \bigl(1, S^{(1)}(x), S^{(2)}(x), \ldots \bigr),
    \end{align*}
    where the \(m\)-th tensor is given by the iterated integral
    \begin{align*}
    S^{(m)}(x)\;:=\;\int_{a \le t_1 < \cdots < t_m \le b} dx_{t_1}\otimes \cdots \otimes dx_{t_m}
    \;\in\; V^{\otimes m},
    \qquad m\in\mathbb{N}.
    \end{align*}
    Here \(\otimes\) denotes the standard tensor product.
\end{definition}
 One can view \(S(x)\) as an infinite-dimensional feature representation of the path \(x\), with higher signature levels \(S^{(m)}(x)\) (the level-$m$ term of $S(x)$; we use this notation throughout) encoding increasingly high-order interactions between the \(d\) channels.

\begin{definition}[Time augmentation]
\label{def:time_augmentation}
The signature is invariant under continuous and non-decreasing time reparametrizations \citep[Lemma 1.2.1]{cass2024lecturenotesroughpaths}. In many machine learning settings, this invariance is not desirable, as we want to capture time series in a way that also depends on their time parametrisation. Following prior work \citep{kidger2020neural, barancikova2025sigdiffusions}, we therefore use the standard \emph{time augmentation} trick. For a path \(x:[a,b]\to V\), we define the augmented path as
\[
\bar x:[a,b]\to \mathbb{R}\times V,
\qquad
\bar x(t) := (t, x(t)).
\]
\end{definition}
Throughout, we apply the signature transform to \(\bar x\) rather than \(x\). For notational simplicity, we suppress the bar in the definitions below. Note the signature is also invariant under constant translations of \(x\) (see Section~1.4 in \citet{cass2024lecturenotesroughpaths}). Throughout, we either fix \(x(a)=0\) or treat paths modulo translation.

\begin{remark}[Tree-like equivalence without time augmentation]
Without the additional time coordinate, the signature is injective only up to \emph{tree-like equivalence} \citep[Section~1.4]{cass2024lecturenotesroughpaths}. Accordingly, statements such as universality and characteristicness of the signature kernel on a set \(K\) should be understood as holding on the corresponding quotient space of equivalence classes when time augmentation is not used.
\end{remark}

When working directly with signatures, one typically truncates \(S(x)\) at level \(L\), defining the truncated signature
\(S_{\le L}(x)=(1,S^{(1)}(x),\ldots,S^{(L)}(x))\).
However, if one is interested in comparing paths \(x,y\in\mathcal{C}_p\) and defining the corresponding kernel-induced discrepancies between distributions, the \emph{signature kernel} provides a way to do so without truncation by working with inner products on the full tensor algebra. Concretely, equip \(V\) with an inner product \(\langle\cdot,\cdot\rangle_V\). This induces canonical inner products \(\langle\cdot,\cdot\rangle_m\) on each tensor space \(V^{\otimes m}\), and hence an inner product \(\langle\cdot,\cdot\rangle_{\mathcal{T}}\) on \(\mathcal{T}\) by linearity across levels.

\begin{definition}[Signature kernel]
For two paths \(x,y\in\mathcal{C}_p\), the signature kernel \(\kappa_{\mathrm{sig}}:\mathcal{C}_p\times\mathcal{C}_p\to\mathbb{R}\) is defined as the inner product of their signatures,
\begin{equation*}
    \kappa_{\mathrm{sig}}(x,y) \;:=\; \langle S(x), S(y) \rangle_{\mathcal{T}}.
\end{equation*}
It is symmetric and positive semidefinite, and thus defines an RKHS on \(\mathcal{C}_p\).
\end{definition}

Crucially, \(\kappa_{\mathrm{sig}}\) can be computed without explicitly forming the signature tensors, via the signature kernel trick introduced in \citet[Theorem~2.5]{salvi2021signature}. In particular, one can evaluate
\begin{equation*}
\kappa_{\mathrm{sig}}(x,y) \;=\; K(b,b),
\end{equation*}
where \(K:[a,b]^2\to\mathbb{R}\) is the signature kernel between two path prefixes, \(K(s,t):=\langle S(x_{[a,s]}), S(y_{[a,t]})\rangle\), which satisfies the integral equation
\begin{equation*}
K(s,t)
\;=\;
1 + \int_{a}^{s}\!\int_{a}^{t} K(u,v)\,\langle dx_u, dy_v\rangle_V.
\end{equation*}
When \(x\) and \(y\) are continuously differentiable, this reduces to a linear hyperbolic Goursat-type PDE with boundary conditions
\(K(a,\cdot)=K(\cdot,a)=1\), for which standard finite-difference solvers apply. See \citet[Section~3.1]{salvi2021signature} for one possible finite-difference scheme. We use the \textsc{pySigLib} library from \citet{shmelev2025pysiglibfastsignaturebased} for efficient signature kernel computations.

\begin{proposition}[Universality of the signature kernel \citep{cass2024lecturenotesroughpaths}]
Let \(\mathcal{C}\) be a path space equipped with a topology such that the signature transform
\(S:\mathcal{C}\to \mathcal{T}\) is continuous into a tensor Hilbert space \(\mathcal{T}\) for which the
signature kernel
\[
\kappa_{\mathrm{sig}}(x,y) := \langle S(x), S(y)\rangle_{\mathcal{T}}
\]
is well-defined.
Then on compact \(K\subset \mathcal{C}\), the restriction \(\kappa_{\mathrm{sig}}|_{K\times K}\) is universal in the sense of Definition \ref{def:universal_kernel}, and hence also characteristic to \(\mathcal{P}(K)\).
\end{proposition}

A precise statement of universality and its proof are given in \citet[Theorem~2.2.3]{cass2024lecturenotesroughpaths}. Characteristicness on
compact \(K\) follows from the relation between universality and characteristicness established in
\citet[Theorem~2.2.7]{cass2024lecturenotesroughpaths}.

For a more detailed overview of signatures and signature kernels, we refer the reader to \citet{cass2024lecturenotesroughpaths}. The compact-set universality and characteristicness statements used here can be found specifically in Chapter 2.2. 

\subsection{Optimal Transport and Sinkhorn Divergence}

This section introduces the static optimal transport problem and its entropy-regularized version, which connects to the theory of statistical divergences.

Let $X, Y$ be metric spaces with Borel $\sigma$-algebras $\mathcal{B}(X),\mathcal{B}(Y)$, let $\mu\in \mathcal{P}(X)$ and $\nu \in\mathcal{P}(Y)$ be probability measures, and let $c: X\times Y\rightarrow \mathbb{R}$ be a measurable cost function.

\begin{definition}[Coupling]
Let $\mathrm{pr}_X(x,y)=x$ and $\mathrm{pr}_Y(x,y)=y$ be the coordinate projections.
Then the set of couplings (transport plans) is: 
\begin{align*}
\Pi(\mu,\nu)
:=\Big\{\pi\in \mathcal{P}(X\times Y)\ :\ (\mathrm{pr}_X)_{\#}\pi=\mu,\ (\mathrm{pr}_Y)_{\#}\pi=\nu\Big\}.
\end{align*}
\end{definition}

\begin{definition}[Kantorovich Problem]
The static optimal transport cost associated with $c$ is
\begin{align*}
    \mathsf{OT}_c(\mu,\nu)
:= \inf_{\pi\in\Pi(\mu,\nu)} \int_{X\times Y} c(x,y)\, \pi(d x, d y).
\end{align*}
Any minimizer $\pi^\star\in\Pi(\mu,\nu)$ (if it exists) is called an \emph{optimal transport plan}.
\end{definition}

\begin{definition}[Wasserstein-$p$ metrics]
If $X=Y$ is a metric space $(X,d)$ and $c(x,y)=d(x,y)^p$ with $p\ge 1$, define
\begin{align*}
W_p(\mu,\nu)
:=\left(\inf_{\pi\in\Pi(\mu,\nu)} \int_{X\times X} d(x,y)^p\, d\pi(x,y)\right)^{1/p},
\end{align*}
for $\mu,\nu$ in $\mathcal{P}_p(X):=\left\{\rho\in\mathcal{P}(X):\int d(x,x_0)^p\,d\rho(x)<\infty\right\}$.
\end{definition}

In the main text, $\mathcal{W}(\mu,\nu)$ denotes $\mathsf{OT}_c(\mu,\nu)$ with the quadratic cost $c=d^2$, that is, $W_2^2(\mu,\nu)$. Next we present the Kantorovich dual form, which clarifies what is required of the ambient space for the infimum in the Kantorovich problem to be attained. First, we define Polish spaces:

\begin{definition}[Polish space]
A topological space $(X,\tau)$ is called \emph{Polish} if there exists a metric
$d:X\times X\to[0,\infty)$ such that:
\begin{enumerate}
  \item $d$ generates the topology $\tau$, i.e.\ $\tau=\tau_d$, where
  $\tau_d$ is the metric topology induced by $d$;
  \item $(X,d)$ is \emph{complete}, i.e.\ every $d$-Cauchy sequence in $X$
  converges (with respect to $d$) to a point in $X$;
  \item $(X,d)$ is \emph{separable}, i.e.\ there exists a countable dense subset
  $D\subset X$ such that $\overline{D}=X$.
\end{enumerate}
Equivalently, a metric space $(X,d)$ is Polish if it is complete and separable.
\end{definition}

\begin{theorem}[Kantorovich duality on Polish spaces]
Let $(X,d_X)$ and $(Y,d_Y)$ be Polish spaces, and let $\mu\in\mathcal{P}(X)$,
$\nu\in\mathcal{P}(Y)$. Assume $c:X\times Y\to(-\infty,+\infty]$ is lower semicontinuous
and bounded from below, and that there exists at least one coupling
$\pi_0\in\Pi(\mu,\nu)$ with $\int c\,d\pi_0 < \infty$.
Then
\begin{align*}
\inf_{\pi\in\Pi(\mu,\nu)} \int_{X\times Y} c(x,y)\, d\pi(x,y)
\;=\;
\sup\left\{\int_X \varphi\, d\mu + \int_Y \psi\, d\nu \ :\ 
\varphi(x)+\psi(y)\le c(x,y)\ \forall(x,y)\right\},
\end{align*}
where the supremum is taken over all Borel functions
$\varphi:X\to\mathbb{R}\cup\{-\infty\}$ and $\psi:Y\to\mathbb{R}\cup\{-\infty\}$
such that $\varphi\in L^1(\mu)$ and $\psi\in L^1(\nu)$.
Moreover, the infimum is attained by some optimal plan $\pi^\star\in\Pi(\mu,\nu)$.
\end{theorem}

This duality theorem can be found in \citet{Villani2008OptimalTO}, Chapter 5, and the attainment of the infimum on Polish spaces in \citet{Villani2008OptimalTO}, Theorem 4.1. For Theorem \ref{thm:approx_error_correct}, attainment of the infimum holds on both function spaces used in our experiments: a separable Hilbert space with its norm topology is Polish, and so is path space with its product topology \citep[Lemma 1.4.13]{cass2024lecturenotesroughpaths}.

The theory of entropy-regularized static OT can be found in \citet{peyre2020computationaloptimaltransport}, Chapter 4, and its connection with statistical divergences in Chapter 8. We state only the definitions needed for the Hilbert Sinkhorn Divergence of the main text.

The entropy-regularized problem adds to the transport cost a multiple of the Kullback--Leibler (KL) divergence of the plan from the product of its marginals.

\begin{definition}[Kullback--Leibler (KL) divergence]
Let $(X,\mathcal{F})$ be a measurable space and let $P,Q$ be probability measures on it.
The \emph{KL divergence} of $P$ from $Q$ is defined by
\begin{align*}
\mathrm{KL}(P\|Q)
:=
\begin{cases}
\displaystyle \int_X \log\!\left(\frac{dP}{dQ}\right)\, dP,
& \text{if } P \ll Q,\\
+\infty, & \text{otherwise.}
\end{cases}
\end{align*}
Equivalently, if $P\ll Q$ then
\begin{align*}
\mathrm{KL}(P\|Q)
= \int_X \frac{dP}{dQ}\,\log\!\left(\frac{dP}{dQ}\right)\, dQ,
\end{align*}
where $\frac{dP}{dQ}$ denotes the Radon--Nikodym derivative (defined $Q$-a.e.).
\end{definition}

\begin{definition}[Entropic OT Functional]
    \begin{align*}
\mathsf{OT}_{c,\varepsilon}(\mu,\nu)
:=\inf_{\pi\in\Pi(\mu,\nu)}
\left\{
\int_{X\times Y} c(x,y)\,d\pi(x,y)
+\varepsilon\,\mathrm{KL}(\pi\,\|\,\mu\otimes\nu)
\right\}
\end{align*}
\end{definition}

The Sinkhorn divergence is the debiased version of the entropic OT functional:

\begin{definition}[Sinkhorn-Divergence]
    \begin{align*}
\mathsf{S}_{c,\varepsilon}(\mu,\nu)
:=\mathsf{OT}_{c,\varepsilon}(\mu,\nu)
-\frac12\,\mathsf{OT}_{c,\varepsilon}(\mu,\mu)
-\frac12\,\mathsf{OT}_{c,\varepsilon}(\nu,\nu).
\end{align*}
\end{definition}

This is the entropic OT functional used in the proof of Theorem \ref{thm:approx_error_correct}; it is well defined in this generality because our ambient spaces are Polish. The Hilbert Sinkhorn divergence in the main text and in \citet{Li_2021_CVPR} is obtained by applying the same self-cost debiasing after restricting the ambient space to an RKHS with its norm-induced topology.

\section{Technical Proofs}\label{app:b}

\subsection{Inherited Results from Li et al.\ (2021)}

For transparency, we separate prior results from our original proofs. Proposition \ref{prop:equiv_hsd_rd} (the equivalence behind Equation \eqref{eq:hsd_x}) and the two propositions below are all inherited from \citet{Li_2021_CVPR}. We restate them here because they are standard ingredients used by our function-space analysis, but we do not claim them as new contributions. Their proofs are therefore referenced directly to the original source.

\begin{proposition}[\citet{Li_2021_CVPR}]\label{prop:equiv_hsd_rd}
The entropic OT functional underlying Definition \ref{def:hsd} can be equivalently formulated as in Equation \eqref{eq:hsd_x}, in the sense that if $\pi^*$ minimizes Equation \eqref{eq:hsd_x}, then its pushforward $(\phi\otimes\phi)_\#\pi^*$ minimizes Equation \eqref{eq:hsd_h}; applying the same self-cost subtraction on both sides yields $S_{\kappa,\epsilon}(\mu,\nu)=S_{\epsilon}(\phi_\#\mu, \phi_\#\nu)$.
\end{proposition}
\paragraph{Proof:} The proof is given as Theorem 1 in the supplementary material of \citet{Li_2021_CVPR} \footnote{\label{linksupp}\url{https://openaccess.thecvf.com/content/CVPR2021/supplemental/Li_Hilbert_Sinkhorn_Divergence_CVPR_2021_supplemental.pdf}}.

\begin{proposition}\label{prop:consistency}
Let $\mu_n, \nu_n$ be the empirical estimators in Equation \eqref{eq:emp_est} and $\epsilon,\eta>0$. Then there exists $N>0$ such that
\begin{align*}
    \forall n \geq N, \qquad \mathbb{P}\left(|S_{\epsilon}(\mu_n,\nu_n)-S_\epsilon(\mu,\nu)| \leq \epsilon \eta\right) = 1.
\end{align*}
\end{proposition}

\paragraph{Proof: } The proof is given as Theorem 2 in the supplementary material of \citet{Li_2021_CVPR}, same link as before \textsuperscript{\ref{linksupp}}.

\begin{proposition}\label{prop:complexity}
Let $||f||_{\mathcal{H}} \leq M$ and consider $B_\eta = \{f\in \mathcal{H}: ||f||_{\mathcal{H}} <\eta\}$. Let $m$ be the number of basis spanning functions $f$ in $B_\eta$. Then for $M,\eta,\delta,\epsilon > 0$,
\begin{align*}
    \mathbb{P}\Big(|S_{\epsilon}(\mu_n,\nu_n)-S_\epsilon(\mu, \nu)| \leq \epsilon \eta\Big)\geq 1-\delta,
\end{align*}
whenever
\begin{align*}
    n\geq \frac{1}{\eta^2}\Big(2M^2(\log (2/\delta)+m\log(24M/\eta))\Big).
\end{align*}
\end{proposition}

\paragraph{Proof: } The proof is given as Proposition 3 in the supplementary material of \citet{Li_2021_CVPR}, same link as before \textsuperscript{\ref{linksupp}}.

\subsection{Absolute Continuity Behind the CFM/FM Equivalence}\label{app:cameron_martin}

The equivalence between the conditional and unconditional flow-matching losses in Section \ref{sec:ffm_background} rests on an absolute-continuity assumption that is worth making explicit in infinite dimensions. Two FFM conditionals $\mathcal N(tf,\sigma_t^2C_0)$ and $\mathcal N(t'f',\sigma_{t'}^2C_0)$ are typically mutually singular (Feldman--H\'ajek), so the relevant structure is not conditional-versus-conditional but \emph{conditional-versus-mixture at each fixed $t$}: one needs $\mu_t^f\ll\mu_t$ for $\nu$-a.e.\ $f$. A sufficient condition, used implicitly by FFM \citep{kerrigan2023functionalflowmatching}, is that the data measure $\nu$ is supported in the Cameron--Martin space $H(C_0)=\mathrm{Ran}(C_0^{1/2})$ of the prior, together with $\sigma_t$ bounded below ($\sigma_{\min}>0$, which Algorithm \ref{alg:mbotcfm} enforces). Then, for fixed $t$, all conditionals $\{\mathcal N(tf,\sigma_t^2C_0):f\in\mathrm{supp}\,\nu\}$ are mutually equivalent by the Cameron--Martin formula, hence absolutely continuous with respect to the mixture $\mu_t$, and the exchange of marginal and conditional velocities is justified at each $t$. When the data are rougher than the prior ($\mathrm{supp}\,\nu\not\subset H(C_0)$) the equivalence degrades, so the prior roughness should be matched to the data; this is one reason the white-noise base measure is selected on the Navier--Stokes family (Appendix \ref{app:selection}).

\paragraph{How kFFM inherits this structure.} At the population level, write $G_t(f_0,f_1):=tf_1+\sigma_t f_0$. The kFFM path marginal is $\mu_t^{\pi}=(G_t)_\#\pi$, the FFM marginal is $\mu_t=(G_t)_\#(\mu_0\otimes\nu)$, and the kFFM conditional given $f_1$ is $\mu_t^{f_1,\pi}=G_t(\cdot,f_1)_\#\,\pi(\cdot\mid f_1)$, while the FFM conditional is $\mathcal N(tf_1,\sigma_t^2C_0)=G_t(\cdot,f_1)_\#\,\mu_0$. For the entropic problem with a bounded cost (Theorem \ref{thm:regularity}) the optimal plan has the Gibbs form $d\pi/d(\mu_0\otimes\nu)=\exp\{(\varphi\oplus\psi-c_\kappa)/\varepsilon\}$ with bounded potentials, so its density is bounded above and away from zero. Hence $\pi\sim\mu_0\otimes\nu$ and $\pi(\cdot\mid f_1)\sim\mu_0$ for $\nu$-a.e.\ $f_1$; since equivalence of measures is preserved under pushforward, $\mu_t^{f_1,\pi}\sim\mathcal N(tf_1,\sigma_t^2C_0)$ and $\mu_t^{\pi}\sim\mu_t$, and every absolute-continuity relation that FFM relies on transfers verbatim to kFFM. Two remarks. First, this is exactly where entropic regularization matters: an unregularized OT plan can be singular with respect to the product measure, and the argument would fail. Second, the minibatch plans used in training are discrete surrogates of this population object (Section \ref{sec:coupling}); the statement concerns the coupling they approximate.

\subsection{Proof of Well-posedness and Uniform Boundedness (Theorem \ref{thm:regularity})}

\regularitythm*

\paragraph{Proof: } The argument is elementary once one uses the boundedness of the RKHS cost induced by a bounded kernel. We start by writing the Hilbert Sinkhorn divergence in the general form
\[
S_{\kappa,\varepsilon}(\mu,\nu)
:=
\mathsf{OT}_{\kappa,\varepsilon}(\mu,\nu)
-\tfrac12 \mathsf{OT}_{\kappa,\varepsilon}(\mu,\mu)
-\tfrac12 \mathsf{OT}_{\kappa,\varepsilon}(\nu,\nu).
\]

By the reproducing property, the canonical feature map $\phi$ of Definition \ref{def:hsd} satisfies
\[
\|\phi(f)\|_{\mathcal H_\kappa}^2
=
\langle \kappa(\cdot,f),\kappa(\cdot,f)\rangle_{\mathcal H_\kappa}
=
\kappa(f,f)
\le B_\kappa
\qquad\text{for all }f\in\mathcal X.
\]
Therefore, for any $f,g\in\mathcal X$,
\begin{align*}
c_\kappa(f,g)
&=\|\phi(f)-\phi(g)\|_{\mathcal H_\kappa}^2
\le 2\|\phi(f)\|_{\mathcal H_\kappa}^2 + 2\|\phi(g)\|_{\mathcal H_\kappa}^2
\le 4B_\kappa.
\end{align*}

Fix $\varepsilon>0$. Since $c_\kappa\ge 0$ and $\mathrm{KL}\ge 0$, we have $\mathsf{OT}_{\kappa,\varepsilon}(\mu,\nu)\ge 0$.
To upper bound $\mathsf{OT}_{\kappa,\varepsilon}(\mu,\nu)$, use the feasible coupling $\pi_0:=\mu\otimes\nu\in\Pi(\mu,\nu)$.
Then $\mathrm{KL}(\pi_0\|\mu\otimes\nu)=0$ and hence
\[
\mathsf{OT}_{\kappa,\varepsilon}(\mu,\nu)
\le
\int_{\mathcal X\times\mathcal X} c_\kappa(f,g)\,d(\mu\otimes\nu)(f,g)
\le
4B_\kappa.
\]
Since $\mu\otimes \nu$ is a probability measure, it follows that $0\le \mathsf{OT}_{\kappa,\varepsilon}(\mu,\nu)\le 4B_\kappa<\infty$.
The same argument applies to $\mathsf{OT}_{\kappa,\varepsilon}(\mu,\mu)$ and $\mathsf{OT}_{\kappa,\varepsilon}(\nu,\nu)$. Since each of the three terms in the definition of $S_{\kappa,\varepsilon}(\mu,\nu)$ lies in $[0,4B_\kappa]$,
the expression is finite. Moreover,
\[
\begin{aligned}
|S_{\kappa,\varepsilon}(\mu,\nu)|
&\le
\mathsf{OT}_{\kappa,\varepsilon}(\mu,\nu)
+\tfrac12 \mathsf{OT}_{\kappa,\varepsilon}(\mu,\mu)
+\tfrac12 \mathsf{OT}_{\kappa,\varepsilon}(\nu,\nu)
\\
&\le
4B_\kappa + \tfrac12(4B_\kappa)+\tfrac12(4B_\kappa)
=
8B_\kappa. \qquad \qedsymbol
\end{aligned}
\]

\paragraph{Interpretation.} Theorem \ref{thm:regularity} is the well-posedness step needed in the function-space setting: it guarantees that the kernel cost is uniformly bounded by $4B_\kappa$ and that the HSD objective is finite for all probability measures on the ambient Banach space whenever the kernel is bounded, with no moment assumption on the measures. The uniform bound is what conditions the entropic solver (Section \ref{sec:theory}). This is precisely the hypothesis that is left implicit in the inherited equivalence proposition from \citet{Li_2021_CVPR}.

\subsection{Proof of the Error Decomposition (Theorem \ref{thm:approx_error_correct})}

\approximationthm*

\paragraph{Proof.}
Let $c(x,y)=d(x,y)^2$ and let $c_\kappa$ be the RKHS-induced cost in Theorem \ref{thm:approx_error_correct}. By definition of $\Delta_\kappa$, for every coupling $\pi\in\Pi(\mu,\nu)$,
\[
\left|\int c_\kappa\,d\pi-\int c\,d\pi\right|\le \Delta_\kappa.
\]
Therefore the unregularized OT values for the two costs differ by at most $\Delta_\kappa$.

For the lower bound, non-negativity of the KL term gives
\[
\mathsf{OT}_{\kappa,\varepsilon}(\mu,\nu)
\ge \inf_{\pi\in\Pi(\mu,\nu)}\int c_\kappa\,d\pi
\ge \mathcal W(\mu,\nu)-\Delta_\kappa.
\]

For the upper bound, let $m:=\mathcal N(\mathcal X,\delta;d)$ and choose centers $x_1,\ldots,x_m$ whose closed $\delta$-balls cover $\mathcal X$. A measurable quantization map $Q:\mathcal X\to\{x_1,\ldots,x_m\}$ with $d(x,Q(x))\le\delta$ is obtained by assigning each point to the first center whose ball contains it, giving a finite Borel partition $C_i=Q^{-1}(\{x_i\})$. Since $\mathcal X$ is compact and $c=d^2$ is continuous, an optimal coupling $\pi^\star$ for $\mathcal W(\mu,\nu)$ exists.

Set $p_{ij}:=\pi^\star(C_i\times C_j)$, $a_i:=\mu(C_i)$, and $b_j:=\nu(C_j)$. For cells with positive mass, let $\mu_i$ and $\nu_j$ be the conditional restrictions of $\mu$ and $\nu$ to $C_i$ and $C_j$, respectively; zero-mass cells may be assigned arbitrary probability measures on $\mathcal X$, since their weights vanish. Define
\[
\widehat\pi:=\sum_{i,j} p_{ij}\,\mu_i\otimes\nu_j.
\]
This construction preserves the marginals, so $\widehat\pi\in\Pi(\mu,\nu)$. On $C_i\times C_j$ with $a_i b_j>0$, the density of $\widehat\pi$ with respect to $\mu\otimes\nu$ is $p_{ij}/(a_i b_j)$, and cells with $a_i b_j=0$ have $p_{ij}=0$. Hence
\[
\mathrm{KL}(\widehat\pi\|\mu\otimes\nu)
= \mathrm{KL}(p\|a\otimes b)
\le 2\log m,
\]
where the final inequality follows from $\mathrm{KL}(p\|a\otimes b)=H(a)+H(b)-H(p)\le H(a)+H(b)\le 2\log m$.
Moreover, if $x\in C_i$ and $y\in C_j$, then
\[
|d(x,y)^2-d(x_i,x_j)^2|
\le (d(x,y)+d(x_i,x_j))\,|d(x,y)-d(x_i,x_j)|
\le 4M\delta.
\]
Applying this once under $\widehat\pi$ and once under $\pi^\star$ gives
\[
\int c\,d\widehat\pi\le \mathcal W(\mu,\nu)+8M\delta.
\]
Using again $\int c_\kappa\,d\widehat\pi\le \int c\,d\widehat\pi+\Delta_\kappa$, feasibility of $\widehat\pi$ for the entropic problem yields
\[
\mathsf{OT}_{\kappa,\varepsilon}(\mu,\nu)
\le
\mathcal W(\mu,\nu)+\Delta_\kappa+8M\delta+2\varepsilon\log m.
\]
Combining the lower and upper bounds and substituting $m=\mathcal N(\mathcal X,\delta;d)$ proves the result. \qedsymbol

\paragraph{Euclidean covering-rate special case.}
If $\mathcal X\subset\mathbb R^d$ has bounded radius $M$, then its covering number satisfies $\log \mathcal N(\mathcal X,\delta)\le C d\log(M/\delta)$ for a universal constant $C>0$. Substituting this into Theorem \ref{thm:approx_error_correct} gives
\[
\left|\mathsf{OT}_{\kappa,\varepsilon}(\mu,\nu)-\mathcal W(\mu,\nu)\right|
\le
\Delta_\kappa + C\bigl(\varepsilon d\log(M/\delta)+M\delta\bigr),
\]
after absorbing constants. When $c_\kappa=c$, this recovers the usual entropic-regularization rate up to constants.

\subsection{Proof of Discretization Invariance (Theorem \ref{thm:discretization})}\label{app:b_discretization}

\discretizationthm*

\paragraph{Notation.} Throughout this proof, $\Omega\subset\mathbb R^d$ is the bounded periodic domain, $\{e_n\}_{n\in\mathbb Z^d}$ is the Fourier basis on $\Omega$, and the Sobolev norm is
\[
\|f\|^2_{H^s} = \sum_{n\in\mathbb Z^d}(1+|n|^2)^s\,|\hat f_n|^2,
\qquad s\ge 0.
\]
$P_N$ denotes the orthogonal projection onto $V_N:=\mathrm{span}\{e_n:|n|^2\le N^{2/d}\}$, so $\dim V_N\asymp N$ and the highest retained mode has frequency $\asymp N^{1/d}$. We write $\mathsf{KOT}^{(k)}_{\varepsilon}:=\mathsf{OT}_{\kappa_{\mathrm{RBF}}^{(k)},\varepsilon}$ and
\[
c_{\mathrm{RBF}}^{(k)}(f,g):=\|\phi(f)-\phi(g)\|^2_{\mathcal H_{\kappa_{\mathrm{RBF}}^{(k)}}}
=2\bigl(1-\kappa_{\mathrm{RBF}}^{(k)}(f,g)\bigr)
\]
for the RKHS-induced cost associated with the RBF kernel on $H^k$. Since $B_R^{H^\alpha}$ is compact in $H^k$ by the compact Sobolev embedding on bounded periodic domains \citep{Evans2010} and $c_{\mathrm{RBF}}^{(k)}$ is bounded and continuous, the entropic kernel-OT functionals below attain minimizers.

\paragraph{Step 1: Sobolev approximation on the truncation tail.} For any $f\in H^\alpha(\Omega)$ with $\|f\|_{H^\alpha}\le R$ and any $0\le k<\alpha$,
\begin{align*}
\|f-P_N f\|^2_{H^k}
&=\sum_{|n|^2 > N^{2/d}}(1+|n|^2)^k\,|\hat f_n|^2 \\
&=\sum_{|n|^2 > N^{2/d}}(1+|n|^2)^{-(\alpha-k)}\,(1+|n|^2)^\alpha\,|\hat f_n|^2 \\
&\le (1+N^{2/d})^{-(\alpha-k)}\,\|f\|^2_{H^\alpha}
\;\le\; R^2\,N^{-2(\alpha-k)/d}.
\end{align*}
Hence
\begin{equation}\label{eq:sobolev_tail}
\|f-P_N f\|_{H^k}\;\le\; R\,N^{-(\alpha-k)/d}\qquad\text{for all }f\in B_R^{H^\alpha}.
\end{equation}

\paragraph{Step 2: Stability of the squared $H^k$ distance under simultaneous projection.} For $f,g\in B_R^{H^\alpha}$, since $P_N$ is an orthogonal projection on $H^k$,
\begin{align*}
\|f-g\|^2_{H^k}-\|P_N(f-g)\|^2_{H^k}
&=\|(I-P_N)(f-g)\|^2_{H^k} \\
&\le \|f-g\|^2_{H^\alpha}\cdot N^{-2(\alpha-k)/d} \\
&\le 4R^2\,N^{-2(\alpha-k)/d},
\end{align*}
where we used $\|f-g\|_{H^\alpha}\le 2R$ and the Sobolev tail estimate from Step~1 applied to $f-g$.

\paragraph{Step 3: Cost stability for the RBF kernel.} Let $h(s):=e^{-s/(2\sigma^2)}$ for $s\ge0$, so $\kappa_{\mathrm{RBF}}^{(k)}(f,g)=h(\|f-g\|^2_{H^k})$ and $|h'(s)|\le\tfrac{1}{2\sigma^2}$. Then
\begin{align*}
\bigl|c_{\mathrm{RBF}}^{(k)}(f,g)-c_{\mathrm{RBF}}^{(k)}(P_N f,P_N g)\bigr|
&=2\bigl|h(\|f-g\|^2_{H^k})-h(\|P_N(f-g)\|^2_{H^k})\bigr| \\
&\le \frac{1}{\sigma^2}\bigl|\|f-g\|^2_{H^k}-\|P_N(f-g)\|^2_{H^k}\bigr|
\;\le\; \frac{4R^2}{\sigma^2}\,N^{-2(\alpha-k)/d}.
\end{align*}
The bound is uniform over $(f,g)\in B_R^{H^\alpha}\times B_R^{H^\alpha}$.

\paragraph{Step 4: Upper bound.} Let $\pi^\star\in\Pi(\mu,\nu)$ be a feasible coupling for the continuum problem and define $\widetilde\pi_N:=(P_N\otimes P_N)_\#\pi^\star\in\Pi(\mu_N,\nu_N)$. By Step~3,
\[
\int_{V_N\times V_N} c_{\mathrm{RBF}}^{(k)}\,d\widetilde\pi_N
\;=\;
\int_{\mathcal X\times\mathcal X} c_{\mathrm{RBF}}^{(k)}(P_N f,P_N g)\,d\pi^\star(f,g)
\;\le\;
\int c_{\mathrm{RBF}}^{(k)}\,d\pi^\star + \frac{4R^2}{\sigma^2}\,N^{-2(\alpha-k)/d}.
\]
By the data-processing inequality for KL divergence applied to the measurable map $P_N\otimes P_N$,
\[
\mathrm{KL}(\widetilde\pi_N\|\mu_N\otimes\nu_N)
\;\le\;
\mathrm{KL}(\pi^\star\|\mu\otimes\nu).
\]
Taking $\pi^\star$ to be an optimizer of $\mathsf{KOT}^{(k)}_{\varepsilon}(\mu,\nu)$ and using $\widetilde\pi_N$ as a feasible point for $\mathsf{KOT}^{(k)}_{\varepsilon}(\mu_N,\nu_N)$ gives
\begin{equation}\label{eq:upper_bound_disc}
\mathsf{KOT}^{(k)}_{\varepsilon}(\mu_N,\nu_N)\;\le\;\mathsf{KOT}^{(k)}_{\varepsilon}(\mu,\nu)+\frac{4R^2}{\sigma^2}\,N^{-2(\alpha-k)/d}.
\end{equation}

\paragraph{Step 5: Lower bound via gluing.} Let $\pi_N^\star\in\Pi(\mu_N,\nu_N)$ be an optimizer for $\mathsf{KOT}^{(k)}_{\varepsilon}(\mu_N,\nu_N)$. Disintegrate $\mu$ along $P_N$: since compact metric spaces are standard Borel, the standard regular conditional probability theorem applies \citep{Dudley_2002}. Thus there exist Borel kernels $\mu(\cdot\,|\,u)$, $u\in V_N$, such that
\[
\mu(A) = \int_{V_N} \mu(A\,|\,u)\,d\mu_N(u),\qquad A\subset\mathcal X\ \text{Borel},
\]
with $\mu(\cdot\,|\,u)$ supported on the fibre $P_N^{-1}(u)$ for $\mu_N$-a.e. $u$, and similarly for $\nu(\cdot\,|\,v)$. Define the lifted coupling
\[
\widehat\pi(A\times B):=\int_{V_N\times V_N}\mu(A\,|\,u)\,\nu(B\,|\,v)\,d\pi_N^\star(u,v),
\qquad A,B\subset\mathcal X\ \text{Borel}.
\]
Direct verification gives $\widehat\pi\in\Pi(\mu,\nu)$ and $(P_N\otimes P_N)_\#\widehat\pi=\pi_N^\star$. Since $\mu_N\otimes\nu_N$ is feasible with finite KL, the optimizer $\pi_N^\star$ has finite KL and hence $\pi_N^\star\ll\mu_N\otimes\nu_N$. On each fibre product, the Radon--Nikodym derivative of $\widehat\pi$ with respect to $\mu\otimes\nu$ is exactly the derivative of $\pi_N^\star$ with respect to $\mu_N\otimes\nu_N$ evaluated at $(u,v)=(P_N f,P_N g)$. Therefore, by the chain rule for KL divergence,
\[
\mathrm{KL}(\widehat\pi\|\mu\otimes\nu) = \mathrm{KL}(\pi_N^\star\|\mu_N\otimes\nu_N).
\]
For the cost integral, Step~3 applied in reverse gives
\[
\int c_{\mathrm{RBF}}^{(k)}\,d\widehat\pi
\;\le\;
\int_{V_N\times V_N} c_{\mathrm{RBF}}^{(k)}\,d\pi_N^\star + \frac{4R^2}{\sigma^2}\,N^{-2(\alpha-k)/d},
\]
using that $(f,g)\sim\widehat\pi$ projects to $(u,v)\sim\pi_N^\star$ and that the cost difference is bounded uniformly. Hence
\begin{equation}\label{eq:lower_bound_disc}
\mathsf{KOT}^{(k)}_{\varepsilon}(\mu,\nu)\;\le\;\mathsf{KOT}^{(k)}_{\varepsilon}(\mu_N,\nu_N)+\frac{4R^2}{\sigma^2}\,N^{-2(\alpha-k)/d}.
\end{equation}

\paragraph{Step 6: Conclusion.} Combining \eqref{eq:upper_bound_disc} and \eqref{eq:lower_bound_disc} yields
\[
\bigl|\mathsf{KOT}^{(k)}_{\varepsilon}(\mu,\nu)-\mathsf{KOT}^{(k)}_{\varepsilon}(\mu_N,\nu_N)\bigr|\;\le\;\frac{4R^2}{\sigma^2}\,N^{-2(\alpha-k)/d}. \qquad\qedsymbol
\]

\subsection{Composed End-to-End Bound}\label{app:b_composed}

\begin{corollary}[Composed end-to-end bound]\label{cor:composed} Let $\mathcal X = B_R^{H^\alpha}$ be as in Theorem \ref{thm:discretization}, equipped with $d_k(f,g):=\|f-g\|_{H^k}$, and let
\[
\mathcal W_k(\mu,\nu):=\inf_{\pi\in\Pi(\mu,\nu)}\int_{\mathcal X\times\mathcal X}\|f-g\|_{H^k}^2\,d\pi(f,g)
\]
denote quadratic-cost OT in the $H^k$ metric. Write $M_k:=\operatorname{diam}(\mathcal X;d_k)$ and let $\mathcal N_k(\delta):=\mathcal N(\mathcal X,\delta;d_k)$ be the corresponding covering number. For the Sobolev-RBF kernel in Theorem \ref{thm:discretization}, define the mismatch $\Delta_{\mathrm{RBF}}:=\Delta_{\kappa_{\mathrm{RBF}}^{(k)}}$, i.e.,
\[
\Delta_{\mathrm{RBF}}:=\sup_{f,g\in\mathcal X}\left|2\left(1-e^{-\|f-g\|_{H^k}^2/(2\sigma^2)}\right)-\|f-g\|_{H^k}^2\right|\;\le\; M_k^2+2,
\]
where the bound follows because the RBF-induced cost lies in $[0,2]$ and $\|f-g\|_{H^k}^2\in[0,M_k^2]$ on $\mathcal X$.
Then for every $\varepsilon>0$, $\delta\in(0,M_k]$ and every projection rank $N\in\mathbb N$,
\[
\begin{aligned}
\bigl|\mathsf{KOT}^{(k)}_{\varepsilon}(\mu_N,\nu_N) - \mathcal W_k(\mu,\nu)\bigr|
\;\le{}&
\underbrace{\Delta_{\mathrm{RBF}}}_{\text{kernel-cost mismatch}}
\;+\; \underbrace{2\varepsilon\,\log \mathcal N_k(\delta)}_{\text{entropic regularization}}
\;\\
&+\; \underbrace{8M_k\delta}_{\text{covering}}
\;+\; \underbrace{\frac{4R^2}{\sigma^2}\,N^{-2(\alpha-k)/d}}_{\text{discretization}}.
\end{aligned}
\]
\end{corollary}

\paragraph{Proof.} By the triangle inequality, $|\mathsf{KOT}^{(k)}_{\varepsilon}(\mu_N,\nu_N) - \mathcal W_k(\mu,\nu)|\le |\mathsf{KOT}^{(k)}_{\varepsilon}(\mu_N,\nu_N)-\mathsf{KOT}^{(k)}_{\varepsilon}(\mu,\nu)|+|\mathsf{KOT}^{(k)}_{\varepsilon}(\mu,\nu)-\mathcal W_k(\mu,\nu)|$. The first term is bounded by Theorem \ref{thm:discretization}; the second by Theorem \ref{thm:approx_error_correct} applied on the compact metric space $(\mathcal X,d_k)$ with the Sobolev-RBF kernel, whose mismatch is $\Delta_{\mathrm{RBF}}$. \qedsymbol

\paragraph{Remark.} Of the four terms, only the discretization term vanishes with projection rank $N$ (rate $N^{-2(\alpha-k)/d}$, driven by the regularity gap $\alpha-k$), the formal counterpart of the empirical resolution-invariance in Section \ref{sec:exp}. The corollary relies on compact Sobolev embedding \citep{Evans2010}, and our PDE benchmarks are used in the corresponding smooth-solution regime; such Sobolev and analytic regularity is standard in the PDE settings considered here, including Navier--Stokes in appropriate regimes \citep{CamliyurtKukavicaVicol2020_AnalyticityBoundary}. The kernel-cost mismatch $\Delta_{\mathrm{RBF}}$, by contrast, depends only on kernel design: zero for the linear kernel, finite for RBF (controlled by data scale and bandwidth $\sigma$).

\paragraph{Remark (path space).} For data on path space $C_p([0,T],\mathbb R^d)$ with $p\in[1,2)$, the analogue of spectral truncation is signature truncation at level $L$, $S_{\le L}(x):=(1,S^{(1)}(x),\ldots,S^{(L)}(x))$. Known factorial-decay bounds for signatures of finite-$p$-variation paths \citep[Proposition 1.4.6]{cass2024lecturenotesroughpaths} imply that, on compact sets with uniformly bounded $p$-variation control, the tail $\sum_{\ell>L}\|S^{(\ell)}(x)\|^2$ decays faster than any polynomial in $L$. By Cauchy--Schwarz, the full and truncated signature kernels are then uniformly close on such compact sets. Consequently, the corresponding entropic kernel-OT objectives differ by at most the uniform signature-kernel cost error, since the KL term is unchanged. This gives the same qualitative discretization-invariance message for path space, but with a factorial tail in the signature truncation level rather than a Sobolev Fourier-tail rate.

\section{Additional Experimental Details}\label{app:c}

This appendix provides implementation details and additional experimental results that complement the main text.

\subsection{Model Parameterization}\label{app:model}

We implement Functional Flow Matching with the official implementation provided by its authors\footnote{\url{https://github.com/GavinKerrigan/functional_flow_matching}} without modifying any provided hyperparameter of the neural architecture, including $\sigma_{\min}=10^{-4}$. The baseline models use the following hyperparameters, also following \citet{kerrigan2023functionalflowmatching}:
\begin{enumerate}
    \item DDPM: the noise schedule has $\beta_0 = 10^{-4}$, $\beta_T=0.02$, and $T=1000$.
    \item DDO (NCSN): the time interval is set to $T=10$, and the noise schedule interpolates geometrically between $\sigma_{10} =10^{-3}$ and $\sigma_1 = 1$ on the 1D datasets and between $\sigma_{10}=10^{-2}$ and $\sigma_1=100$ on the 2D datasets.
    \item GANO: the generator is trained every 5 epochs, and the gradient penalty is $\lambda=0.1$ on the 1D and $\lambda=10$ on the 2D datasets.
    \end{enumerate}

\paragraph{Training.} The baseline models are trained with Adam, with learning rate $10^{-3}$ on 1D and $5\times10^{-4}$ on 2D data, except that the rate is $10^{-4}$ for GANO. This configuration also follows \citet{kerrigan2023functionalflowmatching}; no hyperparameter search was performed on the optimizer.

\subsection{Baseline Description}\label{app:baselines}

We describe the baseline generative models used for the empirical comparisons. All operate in function space, and the roster covers the diffusion, flow-matching, adversarial, and score-based paradigms; every baseline uses the same neural-operator backbone and data pipeline as kFFM.

\begin{enumerate}
    \item \textbf{DDPM \citep{Kerrigan2023DiffusionInfiniteDimensions}}: The DDPM used as a baseline is the functional DDPM model proposed by \citet{Kerrigan2023DiffusionInfiniteDimensions}, which generalizes the diffusion model in $\mathbb{R}^d$ using Gaussian measure theory in Hilbert space. Its formulation is analogous to the base DDPM but performs diffusion directly on Gaussian measures. 

    \item \textbf{FFM \citep{kerrigan2023functionalflowmatching}}: Functional Flow Matching is the base model that we extend. It generalizes flow matching in $\mathbb{R}^d$ to infinite dimensions by formulating a path of Gaussian measures in a Hilbert space, in analogy with \citet{Kerrigan2023DiffusionInfiniteDimensions}. In the terminology of \citet{tong2024improvinggeneralizingflowbasedgenerative}, it couples prior and data samples independently, which leaves room for OT-based coupling.
    \item \textbf{CFM-OT($L^2$) \citep{tong2024improvinggeneralizingflowbasedgenerative}}: This is the direct finite-dimensional OT baseline used in the main text. It uses the same minibatch Sinkhorn solver and FNO backbone as our method, but applies OT to discretized vectors with the raw $L^2$ cost and the white-noise (i.i.d.\ Gaussian) base measure, rather than a function-space kernel cost.
    \item \textbf{GANO \citep{Rahman2022GANO}}: The Generative Adversarial Neural Operator (GANO) is a generative adversarial network in function space whose generator maps Gaussian random fields to samples.
    \item \textbf{DDO/NCSN \citep{Lim2023aScoreBasedFunctionSpace}}: The Denoising Diffusion Operator (DDO) is the analogue of score-based generative models \citep{song2021scorebasedgenerativemodelingstochastic} in infinite dimensions. We follow \citet{kerrigan2023functionalflowmatching} and use the NCSN noise scale; the loss and noise schedule, including preconditioning, follow the code released with \citet{Lim2023aScoreBasedFunctionSpace}.
\end{enumerate}

\subsection{Dataset Description}\label{app:datasets}

Of all the datasets we experiment on, AEMET, Gene expression, Economics are from \citet{kerrigan2023functionalflowmatching}, Navier--Stokes, KdV, stochastic Navier--Stokes, stochastic KdV are from \citet{salvi2022neuralstochasticpdesresolutioninvariant}, and Heston is from \citet{issa2023nonadversarialtrainingneuralsdes}. The descriptions below follow those sources; the architectures and training budgets of our runs are given in Appendix \ref{app:training}.

\paragraph{AEMET Dataset: } This dataset comprises functional observations representing the mean annual profile of average daily temperature (°C) over 1980–2009, collected from 73 weather stations in Spain. Each curve is sampled on a uniform grid of length 365.

\paragraph{Gene Expression: } The full dataset contains 10,928 time series sampled at 20 evenly spaced time points, measuring gene-expression amplitudes for four genes. To produce a periodic spiking appearance while preserving the underlying structure, we concatenate the gene-specific sequences in time. Before training, the data are log-transformed and mean-centered. We then focus on 156 high-variability trajectories, selected by requiring the time-averaged standard deviation of each centered series to exceed 0.3.

\paragraph{Economic Dataset: } We use the three economic datasets (population, GDP, labor) of \citet{kerrigan2023functionalflowmatching}; a separate model is trained on each, and the reported Economy metrics are averaged over the component series within each seed (the matched-seed comparisons of Tables \ref{tab:kernel_vs_l2_main} and \ref{tab:paired_tests} use the population and GDP series). As an example, the population dataset consists of time series tracking population changes for 169 countries worldwide from 1950 to 2018, sampled at 69 discrete time points. For visualization, each country’s series is normalized by its own mean, so the curves show population relative to that country’s average over the 69-year period. Overall, the trajectories display approximately linear growth with a common change point shared across countries.

\paragraph{Heston Dataset: } The Heston model \citep{Heston1993ClosedFormSV} is a stochastic volatility model with diffusion-driven, non-smooth sample paths. It models the asset price $S_t$ and its variance $\nu_t$ as (in the standard Heston notation, which is local to this paragraph):
\begin{align*}
    & dS_t = \mu_t S_t dt + \sqrt{\nu_t} S_t dW_t^S \\
    & d\nu_t = \kappa(\theta-\nu_t)dt + \xi \sqrt{\nu_t} dW_t^\nu,
\end{align*}
where $W_t^S, W_t^\nu$ are Wiener processes with correlations.

The Heston dataset is generated using the simulation code provided in the official repository of \texttt{sigker-nsde} \citep{issa2023nonadversarialtrainingneuralsdes} \footnote{\url{https://github.com/issaz/sigker-nsdes}}. The script generates $5000$ univariate time series of length $100$ on $[0,1]$; we model the normalized log-variance paths. A long-horizon variant with $1000$ time steps (Heston-Long) is used only in the sensitivity study (the ``long volatility paths'' of Appendices \ref{app:practices} and \ref{app:sensitivity}) and in the convergence diagnostic of Figure \ref{fig:convergence_main}.

\paragraph{KdV and Stochastic KdV Equation: } The stochastic Korteweg--de Vries equation is a higher-order SPDE given by:
\begin{align*}
    \partial_t u + \gamma \partial^3_x u = 6u\partial_x u + \xi, \quad u(t,0)=u(t,1), \quad u(0,x) = u_0(x), \quad (t,x)\in [0,T]\times [0,1],
\end{align*}

and the base KdV equation simply removes the $Q$-Wiener process $\xi$. The KdV dataset is a single trajectory on a grid of $512$ spatial points with $201$ time steps, and each time slice is one sample; the stochastic KdV dataset consists of $1200$ trajectories on a grid of $128$ spatial points with $101$ time steps each, again treated as snapshots. The simulation code for these equations is taken from the \texttt{torch-spde} repository \citep{salvi2022neuralstochasticpdesresolutioninvariant} \footnote{\url{https://github.com/crispitagorico/torchspde}}.

\paragraph{Navier--Stokes and Stochastic Navier--Stokes Equation: } The 2D stochastic Navier--Stokes equation for an incompressible flow has the form:
\begin{align*}
    \partial_t w - \nu \Delta w = -u\cdot \nabla w+f+\sigma \xi, \quad w(0,x) = w_0(x), \quad (t,x)\in [0,T]\times [0,1]^2,
\end{align*}

where the deterministic force is a function of space only, and $\xi$ is a $Q$-Wiener process colored in space and rescaled by $\sigma = 0.05$. The initial condition is a Gaussian Random Field $w_0\sim \mathcal{N}(0, 3^{3/2}(-\Delta + 49I)^{-3})$ with periodic boundary conditions, and the viscosity is $\nu = 10^{-3}$. The base Navier--Stokes equation corresponds to the case where $\sigma = 0$ (no stochastic forcing). We again use the simulation code from \texttt{torch-spde} for these 2D datasets, on a uniform $64\times 64$ grid. The turbulent benchmark of Section \ref{sec:turbulent} uses the $\nu=10^{-5}$ data of \citet{li2021fourierneuraloperatorparametric} on the same grid ($1200$ trajectories; we keep the snapshots at $t\ge10$ and train on $9600$ of them), and a $128\times128$ forced Navier--Stokes dataset is used only for the cost measurements of Appendix \ref{app:compute}.

\subsection{Training Configuration Description}\label{app:training}

The following tables give the hyperparameters of our runs. In all cases, the base FFM model follows the configuration in the original codebase of \citet{kerrigan2023functionalflowmatching}. We keep the neural architectures fixed and tune only the OT-related hyperparameters; for the 2D PDE datasets we additionally study the Gaussian-prior length scale in Appendix \ref{app:sensitivity}.
\begin{enumerate}
    \item All FFM and kFFM runs use the Adam optimizer with learning rate $10^{-3}$ and a StepLR scheduler (step size 50, $\gamma=0.1$).
    \item Table \ref{tab:1d_hyperparams} gives the FNO configurations on the 1D sequence datasets and Table \ref{tab:1d_pde_hyperparams} those on the 1D PDE datasets. We perform no hyperparameter search for the FNO backbone; only the OT-related hyperparameters are varied, so the configurations are dataset-specific. ``Batch Size (Sig)'' is the batch size used with the signature kernel, whose memory grows with path length.
    \item The hyperparameters of the FNO on the 2D PDE datasets are in Table \ref{tab:2d_hyperparams}.
    \item Shared training settings are given in Table \ref{tab:common_settings}.
\end{enumerate}

\begin{table}[h]
\centering
\scriptsize
\setlength{\tabcolsep}{3pt}
\caption{Training hyperparameters for 1D sequence datasets (FFM with FNO backbone).}
\label{tab:1d_hyperparams}
\begin{adjustbox}{max width=\textwidth}
\begin{tabular}{@{}lcccccccc@{}}
\toprule
\textbf{Dataset} & \textbf{Epochs} & \textbf{Batch Size} & \textbf{Batch Size (Sig)} & \textbf{FNO Modes} & \textbf{Width} & \textbf{MLP Width} & \textbf{GP $\ell$} & \textbf{GP $\sigma^2$} \\
\midrule
AEMET & 300 & 512 & 64 & 64 & 256 & 128 & 0.01 & 0.1 \\
Gene Expr. & 300 & 512 & 128 & 16 & 256 & 128 & 0.01 & 0.1 \\
Economy & 300 & 512 & 128 & 16 & 128 & 64 & 0.01 & 0.1 \\
Heston & 300 & 512 & 128 & 32 & 256 & 128 & 0.01 & 0.1 \\
\bottomrule
\end{tabular}
\end{adjustbox}
\end{table}

\begin{table}[h]
\centering
\scriptsize
\setlength{\tabcolsep}{3pt}
\caption{Training hyperparameters for 1D PDE datasets (FFM with FNO backbone).}
\label{tab:1d_pde_hyperparams}
\begin{adjustbox}{max width=\textwidth}
\begin{tabular}{@{}lcccccccc@{}}
\toprule
\textbf{Dataset} & \textbf{Epochs} & \textbf{Batch Size} & \textbf{Batch Size (Sig)} & \textbf{FNO Modes} & \textbf{Width} & \textbf{MLP Width} & \textbf{GP $\ell$} & \textbf{GP $\sigma^2$} \\
\midrule
KdV & 300 & 16 & 16 & 64 & 256 & 128 & 0.01 & 0.1 \\
Stochastic KdV & 100 & 512 & 128 & 32 & 256 & 128 & 0.01 & 0.1 \\
\bottomrule
\end{tabular}
\end{adjustbox}
\end{table}

\begin{table}[h]
\centering
\scriptsize
\setlength{\tabcolsep}{3pt}
\caption{Training hyperparameters for 2D spatial PDE datasets (FFM with 2D-FNO backbone).}
\label{tab:2d_hyperparams}
\begin{adjustbox}{max width=\textwidth}
\begin{tabular}{@{}lcccccccc@{}}
\toprule
\textbf{Dataset} & \textbf{Epochs} & \textbf{Batch Size} & \textbf{Batch Size (Sig)} & \textbf{FNO Modes} & \textbf{Hidden Ch.} & \textbf{Proj. Ch.} & \textbf{GP $\ell$} & \textbf{GP $\sigma^2$} \\
\midrule
Navier--Stokes & 300 & 512 & 512 & 16 & 32 & 64 & 0.001 & 1.0 \\
Stoch.\ NS & 100 & 512 & 512 & 16 & 32 & 64 & 0.001 & 1.0 \\
\bottomrule
\end{tabular}
\end{adjustbox}
\end{table}

\begin{table}[h]
\centering
\scriptsize
\setlength{\tabcolsep}{4pt}
\caption{Common training settings for all datasets.}
\label{tab:common_settings}
\begin{adjustbox}{max width=\textwidth}
\begin{tabular}{@{}ll@{}}
\toprule
\textbf{Parameter} & \textbf{Value} \\
\midrule
Optimizer & Adam \\
Learning Rate & $10^{-3}$ \\
LR Scheduler & StepLR (step=50, $\gamma$=0.1) \\
FFM $\sigma_{\min}$ & $10^{-4}$ \\
Number of Seeds & 10 (seeds $2^0,\dots,2^9$); 20 (seeds $2^0,\dots,2^{19}$) where $n=20$ in Table \ref{tab:paired_tests}; 5 on the turbulent benchmark \\
ODE solver tolerances (atol, rtol) & $10^{-5}$ \\
\bottomrule
\end{tabular}
\end{adjustbox}
\end{table}

\subsection{Configuration Selection Protocol}\label{app:selection}

kFFM as evaluated in this paper has two per-dataset configuration choices: the OT cost entering the entropic coupling (signature kernel, Euclidean--RBF kernel, Sobolev-RBF kernel, or the raw $L^2$ cost) and the Gaussian base measure (a smooth GP prior or white noise). The ``kFFM (selected)'' column of Table~\ref{tab:kernel_vs_l2_main} is determined by a validation-based rule rather than by test-set performance: for each dataset, every candidate configuration is trained identically, the configuration with the best sliced Wasserstein distance --- a criterion independent of the kernels being compared --- on a held-out validation split is selected, and only the selected configuration is then evaluated on the test set. In practice this rule is nearly deterministic and coincides with a simple prior-knowledge heuristic that practitioners can apply without any search: use the signature kernel with a GP base for path-like or rough sequence data, and a bounded RBF cost (Sobolev-RBF with a GP base on the 1D KdV family, Euclidean-RBF on the Navier--Stokes family) for PDE fields, with the white-noise base selected on the Navier--Stokes family, whose fields are rougher than a smooth GP prior. On our benchmark the validated choice coincides with the best-performing configuration on all datasets except short-horizon Heston, where the top configurations are statistically indistinguishable. Pointwise diagnostics by kernel variant are reported in Appendix \ref{app:detailed}; the cost and base measure of each selected cell are tagged in Table \ref{tab:kernel_vs_l2_main}, and the Sinkhorn regularization and kernel bandwidth follow the defaults of Appendix \ref{app:practices}, to which the results are insensitive within the ranges reported in the sensitivity study below.

\subsection{Detailed Main-Paper Results}\label{app:detailed}

This appendix includes a self-contained summary of the full baseline roster used in the paper, together with structural diagnostics, fuller pointwise comparisons, kernel ablations, and the super-resolution visualizations referenced in Section \ref{sec:exp}. The additional pointwise tables follow the evaluation convention of \citet{kerrigan2023functionalflowmatching}: discretization-level mean/variance MSE summaries together with autocorrelation or log-spectrum diagnostics, and higher moments where reported. These tables focus on models for which such pointwise statistics are currently available, while Table \ref{tab:kernel_vs_l2_main} in the main text summarizes the complete baseline family comparison including CFM-OT($L^2$).

\newcommand{\tstat}[3]{$#1{\scriptscriptstyle\pm}#2\!\times\!10^{#3}$}
\newcommand{\tbest}[3]{\begingroup\boldmath\bfseries$#1{\scriptscriptstyle\pm}#2\!\times\!10^{#3}$\endgroup}

\begin{table}[t]
\scriptsize
\setlength{\tabcolsep}{2pt}
\renewcommand{\arraystretch}{0.97}
\centering
\caption{Interpretable structural diagnostics complementing the full-baseline MMD comparison in Table \ref{tab:kernel_vs_l2_main}. The first block reports autocorrelation error on sequence datasets; the second block reports log-spectrum error on PDE datasets. Lower is better. As in Tables~\ref{tab:baseline_seq} and~\ref{tab:baseline_pde}, the kFFM column reports the best kernel variant per row (the signature kernel on Gene Expression, Economy, and Heston). All results are averaged over 10 seeds.}
\label{tab:structure_main}
\begin{adjustbox}{max width=\textwidth}
\begin{tabular}{lcccccc}
\toprule
Dataset & DDO/NCSN & GANO & DDPM & FFM & CFM-OT($L^2$) & kFFM \\
\midrule
AEMET (Autocorr.) & \tstat{4.21}{0.94}{-3} & \tstat{4.45}{3.43}{-6} & \tstat{2.84}{2.98}{-7} & \tstat{4.69}{1.53}{-7} & \tbest{1.23}{2.72}{-8} & \tstat{2.60}{1.92}{-7} \\
Gene Expr. (Autocorr.) & \tstat{5.48}{0.24}{-2} & \tstat{1.16}{0.68}{-1} & \tstat{3.59}{3.46}{-4} & \tstat{1.59}{1.98}{-4} & \tstat{2.78}{2.32}{-3} & \tbest{4.18}{4.02}{-5} \\
Economy (Autocorr.) & \tstat{2.73}{0.95}{-1} & \tstat{2.29}{1.91}{-1} & \tstat{2.05}{2.85}{-2} & \tstat{1.72}{1.02}{-4} & \tstat{9.35}{3.75}{-3} & \tbest{8.90}{8.39}{-5} \\
Heston (Autocorr.) & \tstat{4.36}{0.25}{-2} & \tstat{8.36}{8.68}{-3} & \tstat{3.35}{2.80}{-4} & \tstat{2.78}{2.05}{-4} & \tstat{1.73}{1.17}{-4} & \tbest{3.31}{3.38}{-5} \\
\midrule
KdV (Log-spectrum) & \tstat{4.52}{0.00}{1} & \tstat{2.02}{0.30}{1} & \tstat{1.71}{0.13}{1} & \tstat{1.72}{0.03}{1} & \tstat{2.73}{0.05}{1} & \tbest{1.69}{0.02}{1} \\
Navier--Stokes (Log-spectrum) & \tstat{9.93}{0.08}{-1} & \tstat{2.21}{0.15}{0} & \tbest{2.58}{0.63}{-1} & \tstat{7.92}{0.89}{-1} & \tstat{3.44}{0.35}{-1} & \tstat{5.87}{0.46}{-1} \\
Stoch.\ KdV (Log-spectrum) & \tstat{1.84}{0.00}{1} & \tstat{1.45}{0.09}{1} & \tbest{1.57}{0.22}{0} & \tstat{4.27}{0.25}{0} & \tstat{9.25}{0.31}{0} & \tstat{3.97}{0.28}{0} \\
Stoch.\ NS (Log-spectrum) & \tstat{1.06}{0.28}{0} & \tstat{1.03}{0.93}{0} & \tstat{8.60}{2.39}{-1} & \tstat{6.22}{1.05}{-2} & \tbest{3.11}{0.93}{-2} & \tstat{3.22}{0.63}{-2} \\
\bottomrule
\end{tabular}
\end{adjustbox}
\end{table}

\begin{table}[t]
\scriptsize
\setlength{\tabcolsep}{2pt}
\centering
\caption{Sequence-data pointwise diagnostics following the evaluation convention of \citet{kerrigan2023functionalflowmatching}: mean and variance of the discretization-level MSE between target and generated samples, together with average autocorrelation error. We compare baseline FFM (Independent) against the OT-coupled variants: Signature = signature kernel, RBF = Euclidean--RBF kernel, Euclidean = raw $L^2$ cost with the GP base (kFFM-Euc). Results are mean$\pm$std over 10 seeds. Best is \textcolor{ForestGreen}{$\mathbf{green}$}; second best is \textcolor{Orange}{$\mathbf{orange}$}.}
\label{tab:sequence_kernel}
\begin{adjustbox}{max width=\textwidth}
\begin{tabular}{llrrr}
\toprule
Dataset & Kernel & Mean & Variance & Autocorr. \\
\midrule
\multirow{4}{*}{AEMET}
 & Independent & $7.38 \pm 6.37 \times 10^{-2}$ & $1.20 \pm 0.40 \times 10^{-2}$ & $4.69 \pm 1.53 \times 10^{-7}$ \\
 & Signature & \textcolor{Orange}{\textbf{$3.60 \pm 3.66 \times 10^{-2}$}} & $7.84 \pm 5.59 \times 10^{-3}$ & $3.87 \pm 2.45 \times 10^{-7}$ \\
 & RBF & \textcolor{ForestGreen}{\textbf{$1.51 \pm 1.61 \times 10^{-2}$}} & \textcolor{ForestGreen}{\textbf{$4.70 \pm 2.57 \times 10^{-3}$}} & \textcolor{ForestGreen}{\textbf{$2.60 \pm 1.92 \times 10^{-7}$}} \\
 & Euclidean & $4.10 \pm 2.69 \times 10^{-2}$ & \textcolor{Orange}{\textbf{$5.59 \pm 2.65 \times 10^{-3}$}} & \textcolor{Orange}{\textbf{$3.47 \pm 1.99 \times 10^{-7}$}} \\
\midrule
\multirow{4}{*}{Gene Expr.}
 & Independent & $1.39 \pm 0.30 \times 10^{-3}$ & \textcolor{Orange}{\textbf{$5.13 \pm 0.74 \times 10^{-3}$}} & $1.59 \pm 1.98 \times 10^{-4}$ \\
 & Signature & \textcolor{ForestGreen}{\textbf{$5.21 \pm 0.98 \times 10^{-4}$}} & \textcolor{ForestGreen}{\textbf{$2.23 \pm 0.23 \times 10^{-3}$}} & \textcolor{ForestGreen}{\textbf{$4.18 \pm 4.02 \times 10^{-5}$}} \\
 & RBF & $1.25 \pm 0.20 \times 10^{-3}$ & $5.67 \pm 0.74 \times 10^{-3}$ & $4.54 \pm 6.51 \times 10^{-5}$ \\
 & Euclidean & \textcolor{Orange}{\textbf{$1.25 \pm 0.20 \times 10^{-3}$}} & $5.56 \pm 0.61 \times 10^{-3}$ & \textcolor{Orange}{\textbf{$4.45 \pm 4.47 \times 10^{-5}$}} \\
\midrule
\multirow{4}{*}{Economy}
 & Independent & $4.90 \pm 3.26 \times 10^{-5}$ & $1.08 \pm 1.13 \times 10^{-5}$ & $1.72 \pm 1.02 \times 10^{-4}$ \\
 & Signature & \textcolor{ForestGreen}{\textbf{$3.30 \pm 2.22 \times 10^{-5}$}} & \textcolor{ForestGreen}{\textbf{$6.60 \pm 5.62 \times 10^{-6}$}} & \textcolor{ForestGreen}{\textbf{$8.90 \pm 8.39 \times 10^{-5}$}} \\
 & RBF & $3.96 \pm 2.12 \times 10^{-5}$ & $6.82 \pm 3.44 \times 10^{-6}$ & $1.38 \pm 0.89 \times 10^{-4}$ \\
 & Euclidean & \textcolor{Orange}{\textbf{$3.94 \pm 2.07 \times 10^{-5}$}} & \textcolor{Orange}{\textbf{$6.76 \pm 3.48 \times 10^{-6}$}} & \textcolor{Orange}{\textbf{$1.38 \pm 0.89 \times 10^{-4}$}} \\
\midrule
\multirow{4}{*}{Heston}
 & Independent & $6.66 \pm 4.77 \times 10^{-3}$ & $1.07 \pm 0.25 \times 10^{-1}$ & $2.78 \pm 2.05 \times 10^{-4}$ \\
 & Signature & \textcolor{ForestGreen}{\textbf{$3.33 \pm 2.25 \times 10^{-3}$}} & \textcolor{ForestGreen}{\textbf{$8.17 \pm 1.14 \times 10^{-2}$}} & \textcolor{ForestGreen}{\textbf{$3.31 \pm 3.38 \times 10^{-5}$}} \\
 & RBF & $4.48 \pm 4.18 \times 10^{-3}$ & \textcolor{Orange}{\textbf{$8.27 \pm 1.53 \times 10^{-2}$}} & $1.41 \pm 1.03 \times 10^{-4}$ \\
 & Euclidean & \textcolor{Orange}{\textbf{$3.38 \pm 3.81 \times 10^{-3}$}} & $8.40 \pm 0.77 \times 10^{-2}$ & \textcolor{Orange}{\textbf{$9.08 \pm 7.99 \times 10^{-5}$}} \\
\bottomrule
\end{tabular}
\end{adjustbox}
\end{table}

\begin{table}[t]
\scriptsize
\setlength{\tabcolsep}{2pt}
\renewcommand{\arraystretch}{0.95}
\centering
\caption{Baseline comparison for time series datasets among models with available pointwise statistics. The full baseline roster, including CFM-OT($L^2$), is summarized in Table \ref{tab:kernel_vs_l2_main}. Following \citet{kerrigan2023functionalflowmatching}, we report mean/variance MSE and autocorrelation error computed on the discretized samples. The kFFM row reports the best of the kernel variants in Table \ref{tab:sequence_kernel} for each metric (an oracle over variants, used only for these diagnostics), whereas Table \ref{tab:kernel_vs_l2_main} uses the validation-selected configuration. Results are mean$\pm$std over 10 seeds. Best is \textcolor{ForestGreen}{$\mathbf{green}$}; second best is \textcolor{Orange}{$\mathbf{orange}$}.}
\label{tab:baseline_seq}
\begin{adjustbox}{max width=\textwidth}
\begin{tabular}{llrrr}
\toprule
Dataset & Model & Mean & Variance & Autocorr. \\
\midrule
\multirow{5}{*}{AEMET}
 & DDO/NCSN & $3.56 \pm 0.96 \times 10^{-1}$ & $1.17 \pm 0.03 \times 10^{-1}$ & $4.21 \pm 0.94 \times 10^{-3}$ \\
 & DDPM & $1.92 \pm 1.29 \times 10^{-1}$ & $3.22 \pm 2.98 \times 10^{-2}$ & \textcolor{Orange}{\textbf{$2.84 \pm 2.98 \times 10^{-7}$}} \\
 & GANO & $2.47 \pm 0.89 \times 10^{-1}$ & $2.35 \pm 0.01 \times 10^{-1}$ & $4.45 \pm 3.43 \times 10^{-6}$ \\
 & FFM & \textcolor{Orange}{\textbf{$7.38 \pm 6.37 \times 10^{-2}$}} & \textcolor{Orange}{\textbf{$1.20 \pm 0.40 \times 10^{-2}$}} & $4.69 \pm 1.53 \times 10^{-7}$ \\
 & kFFM & \cellcolor[gray]{0.9}\textcolor{ForestGreen}{\textbf{$1.51 \pm 1.61 \times 10^{-2}$}} & \cellcolor[gray]{0.9}\textcolor{ForestGreen}{\textbf{$4.70 \pm 2.57 \times 10^{-3}$}} & \cellcolor[gray]{0.9}\textcolor{ForestGreen}{\textbf{$2.60 \pm 1.92 \times 10^{-7}$}} \\
\midrule
\multirow{5}{*}{Gene Expr.}
 & DDO/NCSN & $6.07 \pm 1.06 \times 10^{-3}$ & $9.69 \pm 0.70 \times 10^{-3}$ & $5.48 \pm 0.24 \times 10^{-2}$ \\
 & DDPM & $4.10 \pm 1.24 \times 10^{-3}$ & $1.08 \pm 0.08 \times 10^{-2}$ & $3.59 \pm 3.46 \times 10^{-4}$ \\
 & GANO & $1.89 \pm 0.52 \times 10^{-2}$ & $3.14 \pm 0.00 \times 10^{-2}$ & $1.16 \pm 0.68 \times 10^{-1}$ \\
 & FFM & \textcolor{Orange}{\textbf{$1.39 \pm 0.30 \times 10^{-3}$}} & \textcolor{Orange}{\textbf{$5.13 \pm 0.74 \times 10^{-3}$}} & \textcolor{Orange}{\textbf{$1.59 \pm 1.98 \times 10^{-4}$}} \\
 & kFFM & \cellcolor[gray]{0.9}\textcolor{ForestGreen}{\textbf{$5.21 \pm 0.98 \times 10^{-4}$}} & \cellcolor[gray]{0.9}\textcolor{ForestGreen}{\textbf{$2.23 \pm 0.23 \times 10^{-3}$}} & \cellcolor[gray]{0.9}\textcolor{ForestGreen}{\textbf{$4.18 \pm 4.02 \times 10^{-5}$}} \\
\midrule
\multirow{5}{*}{Economy}
 & DDO/NCSN & $1.70 \pm 1.34 \times 10^{-3}$ & $9.91 \pm 2.06 \times 10^{-3}$ & $2.73 \pm 0.95 \times 10^{-1}$ \\
 & DDPM & $6.08 \pm 7.92 \times 10^{-4}$ & $2.27 \pm 5.34 \times 10^{-4}$ & $2.05 \pm 2.85 \times 10^{-2}$ \\
 & GANO & $3.79 \pm 0.51 \times 10^{-1}$ & $7.28 \pm 4.08 \times 10^{-4}$ & $2.29 \pm 1.91 \times 10^{-1}$ \\
 & FFM & \textcolor{Orange}{\textbf{$4.90 \pm 3.26 \times 10^{-5}$}} & \textcolor{Orange}{\textbf{$1.08 \pm 1.13 \times 10^{-5}$}} & \textcolor{Orange}{\textbf{$1.72 \pm 1.02 \times 10^{-4}$}} \\
 & kFFM & \cellcolor[gray]{0.9}\textcolor{ForestGreen}{\textbf{$3.30 \pm 2.22 \times 10^{-5}$}} & \cellcolor[gray]{0.9}\textcolor{ForestGreen}{\textbf{$6.60 \pm 5.62 \times 10^{-6}$}} & \cellcolor[gray]{0.9}\textcolor{ForestGreen}{\textbf{$8.90 \pm 8.39 \times 10^{-5}$}} \\
\midrule
\multirow{5}{*}{Heston}
 & DDO/NCSN & $1.60 \pm 0.47 \times 10^{-2}$ & $7.32 \pm 0.15 \times 10^{-1}$ & $4.36 \pm 0.25 \times 10^{-2}$ \\
 & DDPM & $9.33 \pm 9.43 \times 10^{-3}$ & $1.08 \pm 0.25 \times 10^{-1}$ & $3.35 \pm 2.80 \times 10^{-4}$ \\
 & GANO & $1.55 \pm 0.90 \times 10^{-2}$ & $4.41 \pm 1.15 \times 10^{-1}$ & $8.36 \pm 8.68 \times 10^{-3}$ \\
 & FFM & \textcolor{Orange}{\textbf{$6.66 \pm 4.77 \times 10^{-3}$}} & \textcolor{Orange}{\textbf{$1.07 \pm 0.25 \times 10^{-1}$}} & \textcolor{Orange}{\textbf{$2.78 \pm 2.05 \times 10^{-4}$}} \\
 & kFFM & \cellcolor[gray]{0.9}\textcolor{ForestGreen}{\textbf{$3.33 \pm 2.25 \times 10^{-3}$}} & \cellcolor[gray]{0.9}\textcolor{ForestGreen}{\textbf{$8.17 \pm 1.14 \times 10^{-2}$}} & \cellcolor[gray]{0.9}\textcolor{ForestGreen}{\textbf{$3.31 \pm 3.38 \times 10^{-5}$}} \\
\bottomrule
\end{tabular}
\end{adjustbox}
\end{table}

\begin{table}[t]
\scriptsize
\setlength{\tabcolsep}{2pt}
\renewcommand{\arraystretch}{0.95}
\centering
\caption{PDE pointwise diagnostics following the evaluation convention of \citet{kerrigan2023functionalflowmatching}: mean and variance of the discretization-level MSE between target and generated samples, together with average log-spectrum error. Here we compare the baseline FFM (Independent) against its OT-coupled variants: RBF = Euclidean--RBF kernel, Euclidean = raw $L^2$ cost with the GP base (kFFM-Euc). Results are mean$\pm$std over 10 seeds. Best is \textcolor{ForestGreen}{$\mathbf{green}$}; second best is \textcolor{Orange}{$\mathbf{orange}$}.}
\label{tab:pde_kernel}
\begin{adjustbox}{max width=\textwidth}
\begin{tabular}{llrrr}
\toprule
Dataset & Kernel & Mean & Variance & Spectrum (log) \\
\midrule
\multirow{3}{*}{KdV}
 & Independent & $5.20 \pm 3.37 \times 10^{-4}$ & $1.39 \pm 0.49 \times 10^{-4}$ & $1.72 \pm 0.03 \times 10^{1}$ \\
 & RBF & \textcolor{ForestGreen}{\textbf{$1.48 \pm 0.95 \times 10^{-4}$}} & \textcolor{ForestGreen}{\textbf{$5.68 \pm 2.58 \times 10^{-5}$}} & \textcolor{Orange}{\textbf{$1.69 \pm 0.02 \times 10^{1}$}} \\
 & Euclidean & \textcolor{Orange}{\textbf{$2.51 \pm 0.74 \times 10^{-4}$}} & \textcolor{Orange}{\textbf{$6.88 \pm 3.30 \times 10^{-5}$}} & \textcolor{ForestGreen}{\textbf{$1.69 \pm 0.03 \times 10^{1}$}} \\
\midrule
\multirow{3}{*}{Navier--Stokes}
 & Independent & $7.96 \pm 7.82 \times 10^{-2}$ & $3.11 \pm 0.80 \times 10^{-2}$ & $7.92 \pm 0.89 \times 10^{-1}$ \\
 & RBF & \textcolor{ForestGreen}{\textbf{$2.01 \pm 2.01 \times 10^{-3}$}} & \textcolor{ForestGreen}{\textbf{$7.51 \pm 3.33 \times 10^{-4}$}} & \textcolor{ForestGreen}{\textbf{$5.87 \pm 0.46 \times 10^{-1}$}} \\
 & Euclidean & \textcolor{Orange}{\textbf{$2.03 \pm 1.68 \times 10^{-2}$}} & \textcolor{Orange}{\textbf{$1.87 \pm 0.57 \times 10^{-2}$}} & \textcolor{Orange}{\textbf{$6.00 \pm 0.57 \times 10^{-1}$}} \\
\midrule
\multirow{3}{*}{Stoch. KdV}
 & Independent & \textcolor{Orange}{\textbf{$1.69 \pm 1.31 \times 10^{-3}$}} & $3.15 \pm 2.31 \times 10^{-3}$ & $4.27 \pm 0.25 \times 10^{0}$ \\
 & RBF & \textcolor{ForestGreen}{\textbf{$1.17 \pm 0.72 \times 10^{-3}$}} & \textcolor{ForestGreen}{\textbf{$1.62 \pm 0.91 \times 10^{-3}$}} & \textcolor{Orange}{\textbf{$4.05 \pm 0.14 \times 10^{0}$}} \\
 & Euclidean & $2.34 \pm 1.71 \times 10^{-3}$ & \textcolor{Orange}{\textbf{$2.51 \pm 1.50 \times 10^{-3}$}} & \textcolor{ForestGreen}{\textbf{$3.97 \pm 0.28 \times 10^{0}$}} \\
\midrule
\multirow{3}{*}{Stoch. NS}
 & Independent & $3.72 \pm 2.31 \times 10^{-1}$ & $5.76 \pm 0.76 \times 10^{0}$ & $6.22 \pm 1.05 \times 10^{-2}$ \\
 & RBF & \textcolor{ForestGreen}{\textbf{$3.38 \pm 1.15 \times 10^{-2}$}} & \textcolor{ForestGreen}{\textbf{$2.59 \pm 0.76 \times 10^{-1}$}} & \textcolor{ForestGreen}{\textbf{$3.22 \pm 0.63 \times 10^{-2}$}} \\
 & Euclidean & \textcolor{Orange}{\textbf{$1.81 \pm 0.66 \times 10^{-1}$}} & \textcolor{Orange}{\textbf{$4.59 \pm 0.81 \times 10^{0}$}} & \textcolor{Orange}{\textbf{$5.25 \pm 0.62 \times 10^{-2}$}} \\
\bottomrule
\end{tabular}
\end{adjustbox}
\end{table}

\begin{table}[t]
\scriptsize
\setlength{\tabcolsep}{2pt}
\renewcommand{\arraystretch}{0.95}
\centering
\caption{Baseline comparison for PDE datasets among models with available pointwise statistics. The full baseline roster, including CFM-OT($L^2$), is summarized in Table \ref{tab:kernel_vs_l2_main}. Following \citet{kerrigan2023functionalflowmatching}, we compare kFFM against baseline models on discretization-level mean/variance MSE and log-spectrum error; the kFFM row reports the best of the kernel variants in Table \ref{tab:pde_kernel} for each metric (an oracle over variants, used only for these diagnostics), whereas Table \ref{tab:kernel_vs_l2_main} uses the validation-selected configuration. Results are mean$\pm$std over 10 seeds. Best is \textcolor{ForestGreen}{$\mathbf{green}$}; second best is \textcolor{Orange}{$\mathbf{orange}$}.}
\label{tab:baseline_pde}
\begin{adjustbox}{max width=\textwidth}
\begin{tabular}{llrrr}
\toprule
Dataset & Model & Mean & Variance & Spectrum (log) \\
\midrule
\multirow{5}{*}{KdV}
 & DDO/NCSN & $2.46 \pm 0.46 \times 10^{-1}$ & $2.78 \pm 0.43 \times 10^{-2}$ & $4.52 \pm 0.00 \times 10^{1}$ \\
 & DDPM & $3.67 \pm 0.91 \times 10^{-3}$ & $5.81 \pm 0.69 \times 10^{-3}$ & \textcolor{Orange}{\textbf{$1.71 \pm 0.13 \times 10^{1}$}} \\
 & GANO & $1.87 \pm 0.94 \times 10^{-1}$ & $2.91 \pm 0.16 \times 10^{-2}$ & $2.02 \pm 0.30 \times 10^{1}$ \\
 & FFM & \textcolor{Orange}{\textbf{$5.20 \pm 3.37 \times 10^{-4}$}} & \textcolor{Orange}{\textbf{$1.39 \pm 0.49 \times 10^{-4}$}} & $1.72 \pm 0.03 \times 10^{1}$ \\
 & kFFM & \cellcolor[gray]{0.9}\textcolor{ForestGreen}{\textbf{$1.48 \pm 0.95 \times 10^{-4}$}} & \cellcolor[gray]{0.9}\textcolor{ForestGreen}{\textbf{$5.68 \pm 2.58 \times 10^{-5}$}} & \cellcolor[gray]{0.9}\textcolor{ForestGreen}{\textbf{$1.69 \pm 0.02 \times 10^{1}$}} \\
\midrule
\multirow{5}{*}{Navier--Stokes}
 & DDO/NCSN & $4.79 \pm 0.01 \times 10^{-1}$ & $4.39 \pm 0.00 \times 10^{-2}$ & $9.93 \pm 0.08 \times 10^{-1}$ \\
 & DDPM & $5.27 \pm 0.21 \times 10^{-1}$ & \textcolor{Orange}{\textbf{$2.17 \pm 0.52 \times 10^{-2}$}} & \textcolor{ForestGreen}{\textbf{$2.58 \pm 0.63 \times 10^{-1}$}} \\
 & GANO & \textcolor{Orange}{\textbf{$6.16 \pm 4.25 \times 10^{-2}$}} & $2.74 \pm 1.59 \times 10^{-2}$ & $2.21 \pm 0.15 \times 10^{0}$ \\
 & FFM & $7.96 \pm 7.82 \times 10^{-2}$ & $3.11 \pm 0.80 \times 10^{-2}$ & $7.92 \pm 0.89 \times 10^{-1}$ \\
 & kFFM & \cellcolor[gray]{0.9}\textcolor{ForestGreen}{\textbf{$2.01 \pm 2.01 \times 10^{-3}$}} & \cellcolor[gray]{0.9}\textcolor{ForestGreen}{\textbf{$7.51 \pm 3.33 \times 10^{-4}$}} & \cellcolor[gray]{0.9}\textcolor{Orange}{\textbf{$5.87 \pm 0.46 \times 10^{-1}$}} \\
\midrule
\multirow{5}{*}{Stoch. KdV}
 & DDO/NCSN & $3.59 \pm 0.60 \times 10^{-2}$ & $1.55 \pm 0.01 \times 10^{-1}$ & $1.84 \pm 0.00 \times 10^{1}$ \\
 & DDPM & $3.90 \pm 0.78 \times 10^{-3}$ & $7.19 \pm 8.48 \times 10^{-3}$ & \textcolor{ForestGreen}{\textbf{$1.57 \pm 0.22 \times 10^{0}$}} \\
 & GANO & $1.80 \pm 0.31 \times 10^{-3}$ & $1.42 \pm 1.57 \times 10^{-2}$ & $1.45 \pm 0.09 \times 10^{1}$ \\
 & FFM & \textcolor{Orange}{\textbf{$1.69 \pm 1.31 \times 10^{-3}$}} & \textcolor{Orange}{\textbf{$3.15 \pm 2.31 \times 10^{-3}$}} & $4.27 \pm 0.25 \times 10^{0}$ \\
 & kFFM & \cellcolor[gray]{0.9}\textcolor{ForestGreen}{\textbf{$1.17 \pm 0.72 \times 10^{-3}$}} & \cellcolor[gray]{0.9}\textcolor{ForestGreen}{\textbf{$1.62 \pm 0.91 \times 10^{-3}$}} & \cellcolor[gray]{0.9}\textcolor{Orange}{\textbf{$3.97 \pm 0.28 \times 10^{0}$}} \\
\midrule
\multirow{5}{*}{Stoch. NS}
 & DDO/NCSN & \textcolor{Orange}{\textbf{$3.45 \pm 0.01 \times 10^{-1}$}} & $7.59 \pm 0.02 \times 10^{0}$ & $1.06 \pm 0.28 \times 10^{0}$ \\
 & DDPM & $3.52 \pm 0.36 \times 10^{-1}$ & \textcolor{Orange}{\textbf{$4.81 \pm 0.28 \times 10^{0}$}} & $8.60 \pm 2.39 \times 10^{-1}$ \\
 & GANO & $4.38 \pm 0.48 \times 10^{-1}$ & $7.05 \pm 0.65 \times 10^{0}$ & $1.03 \pm 0.93 \times 10^{0}$ \\
 & FFM & $3.72 \pm 2.31 \times 10^{-1}$ & $5.76 \pm 0.76 \times 10^{0}$ & \textcolor{Orange}{\textbf{$6.22 \pm 1.05 \times 10^{-2}$}} \\
 & kFFM & \cellcolor[gray]{0.9}\textcolor{ForestGreen}{\textbf{$3.38 \pm 1.15 \times 10^{-2}$}} & \cellcolor[gray]{0.9}\textcolor{ForestGreen}{\textbf{$2.59 \pm 0.76 \times 10^{-1}$}} & \cellcolor[gray]{0.9}\textcolor{ForestGreen}{\textbf{$3.22 \pm 0.63 \times 10^{-2}$}} \\
\bottomrule
\end{tabular}
\end{adjustbox}
\end{table}

\begin{figure}[t]
\centering
\begin{subfigure}[t]{0.42\linewidth}
\centering
\includegraphics[width=\linewidth]{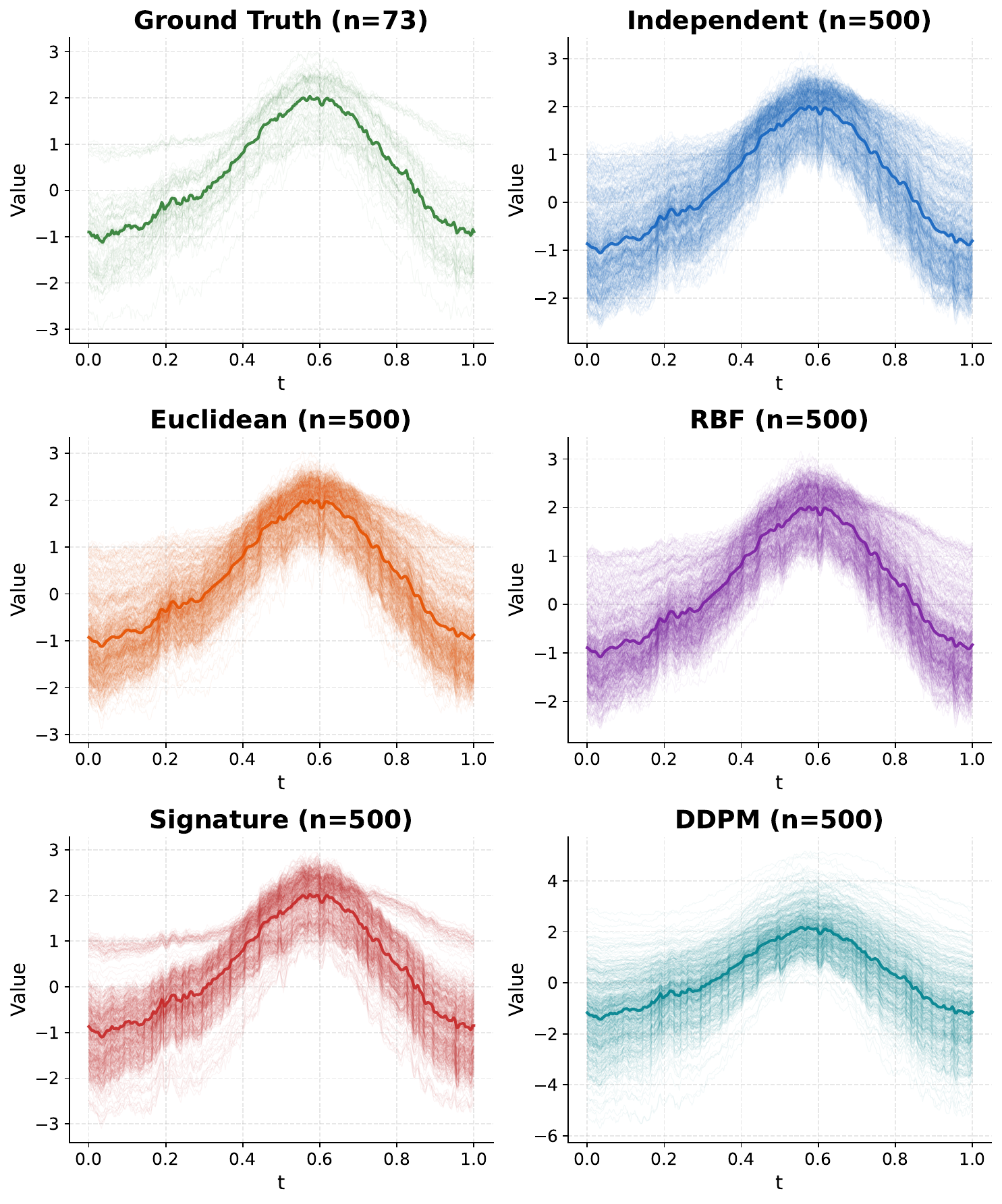}
\caption{Samples.}
\end{subfigure}\hfill
\begin{subfigure}[t]{0.56\linewidth}
\centering
\includegraphics[width=\linewidth]{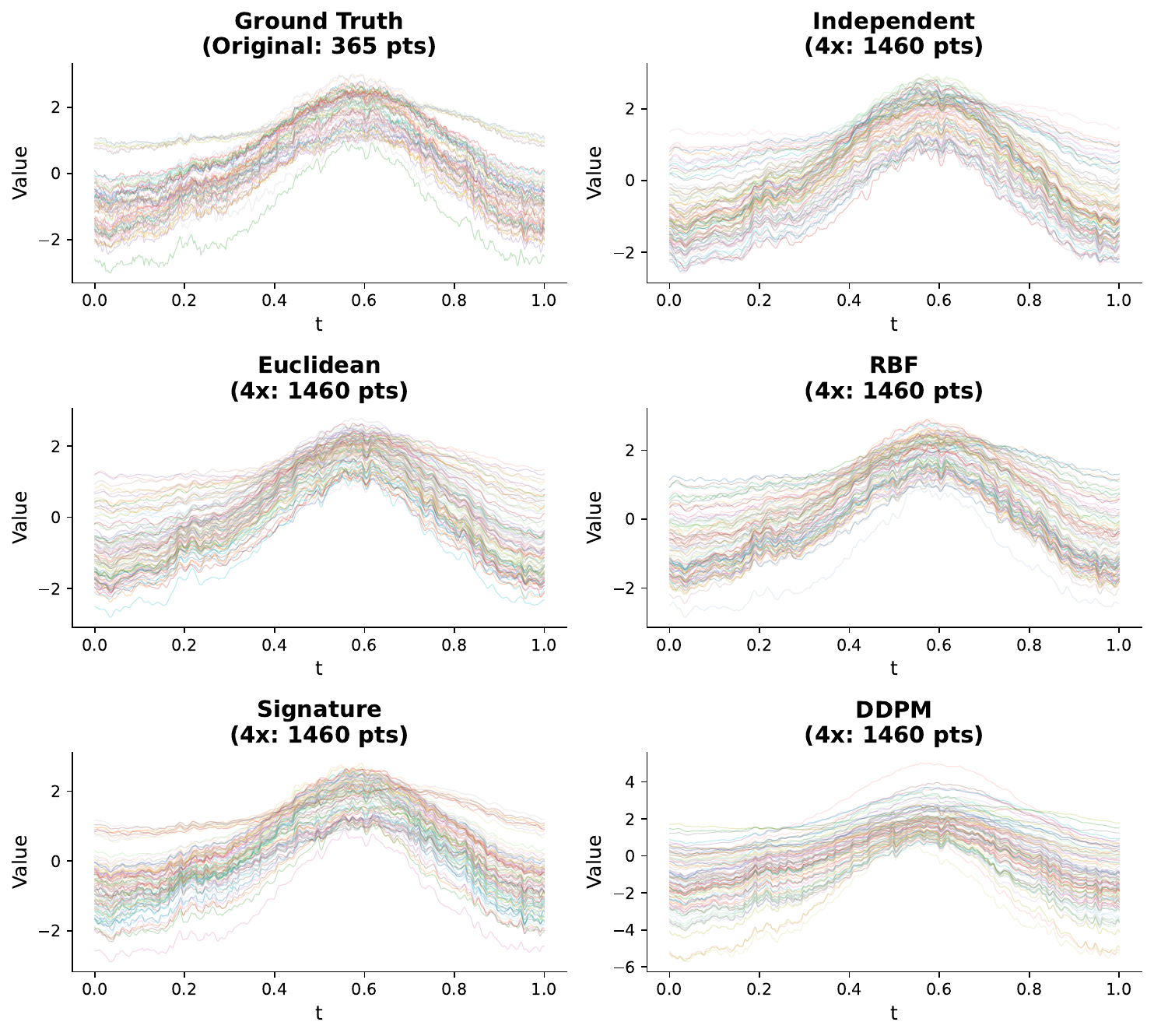}
\caption{4x temporal super-resolution.}
\end{subfigure}
\caption{Qualitative results on AEMET; panels labeled Euclidean, RBF, and Signature are kFFM variants (kFFM-Euc, kFFM-RBF, kFFM-Sig) and Independent is FFM. Kernel OT better preserves the seasonal spread and amplitude of trajectories both at the original resolution and under temporal super-resolution.}
\label{fig:aemet_main}
\end{figure}

\begin{figure}[t]
\centering
\begin{subfigure}[t]{0.49\linewidth}
\centering
\includegraphics[width=\linewidth,height=0.16\textheight,keepaspectratio]{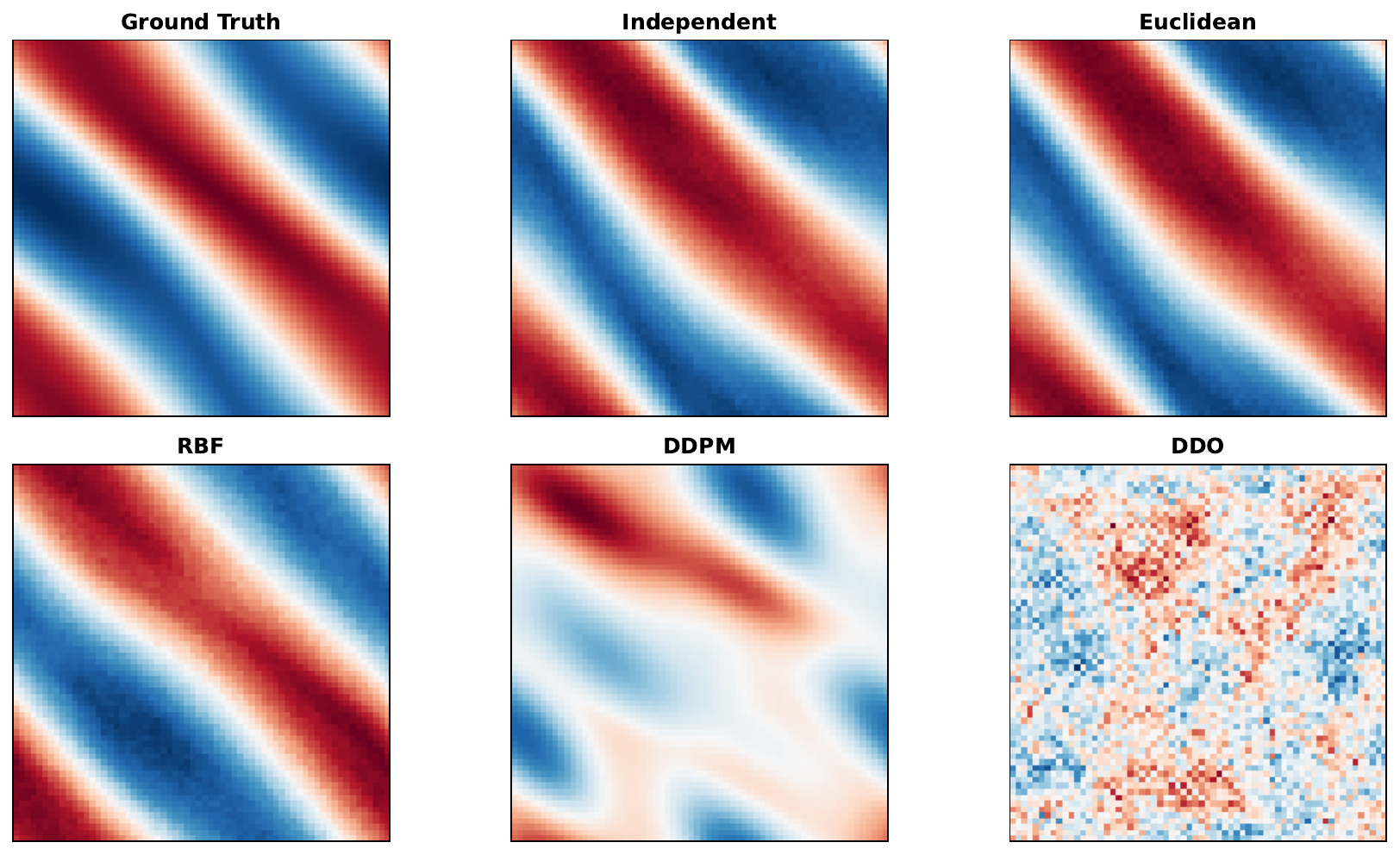}
\caption{Samples.}
\end{subfigure}\hfill
\begin{subfigure}[t]{0.49\linewidth}
\centering
\includegraphics[width=\linewidth,height=0.16\textheight,keepaspectratio]{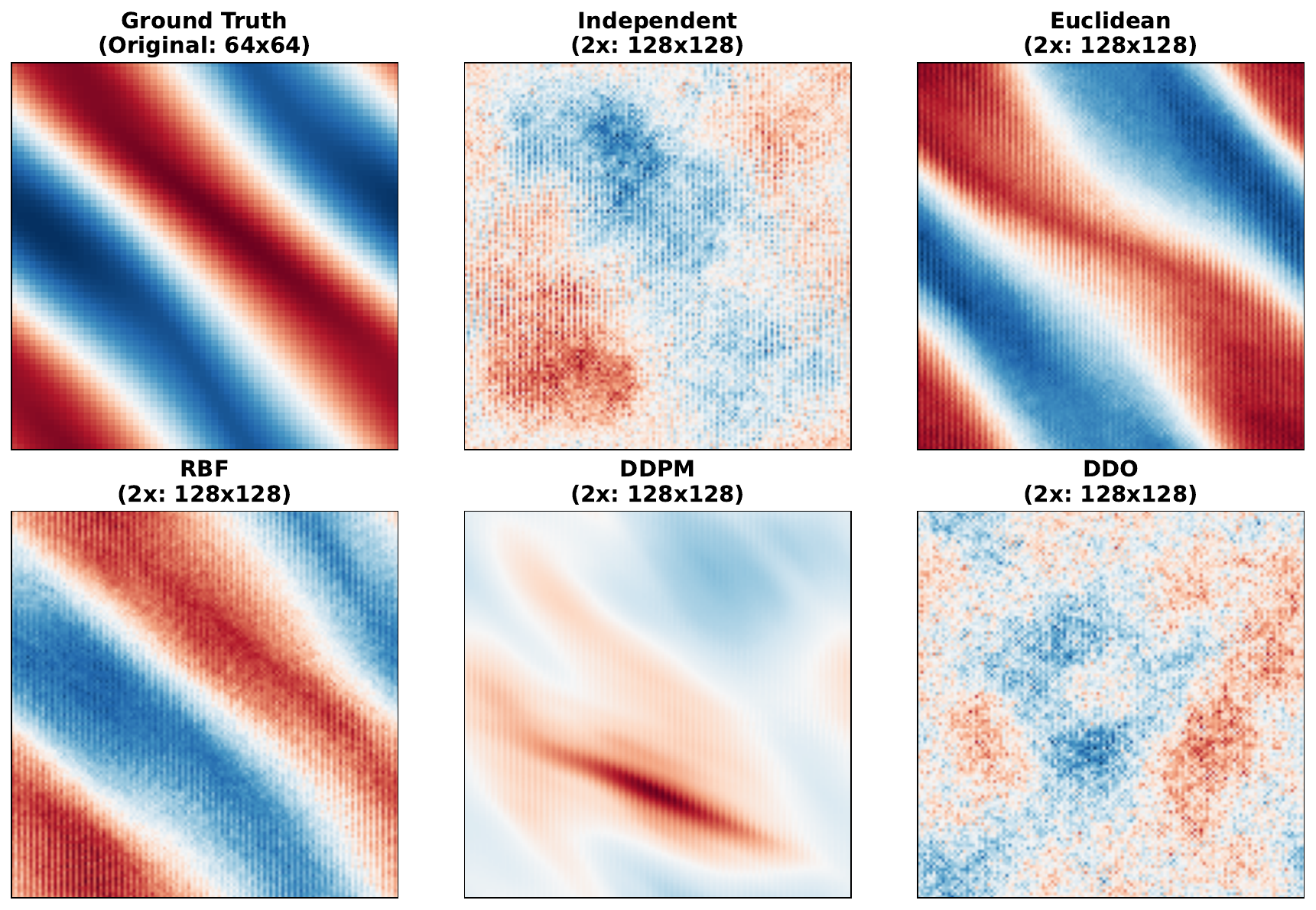}
\caption{2x spatial super-resolution.}
\end{subfigure}
\caption{Qualitative results on Navier--Stokes; panels labeled Euclidean and RBF are kFFM variants (kFFM-Euc, kFFM-RBF) and Independent is FFM. Kernel OT produces smoother and more coherent flow structures than independent coupling, and this advantage persists under spatial super-resolution.}
\label{fig:ns_main}
\end{figure}

\subsection{Full Statistics Report on Sequence Datasets}

For the sequence datasets, in addition to the mean/variance/autocorrelation tables above, we also report kurtosis and skewness in Table \ref{tab:sequence_kernel_full}, since these were also reported in \citet{kerrigan2023functionalflowmatching}. We keep them in the appendix because they are complementary to the main conclusion rather than central to it.

\begin{table}[t]
\scriptsize
\setlength{\tabcolsep}{3pt}
\centering
\caption{Higher-moment sequence-data statistics (skewness and kurtosis errors) for the models with available pointwise diagnostics, following the evaluation convention of \citet{kerrigan2023functionalflowmatching}; mean, variance, and autocorrelation errors are in Table \ref{tab:baseline_seq}, and the full baseline roster, including CFM-OT($L^2$), in Table \ref{tab:kernel_vs_l2_main}. The kFFM row uses the best kernel variant per metric. Lower is better. Best is \textcolor{ForestGreen}{$\mathbf{green}$}, second best is \textcolor{Orange}{$\mathbf{orange}$}.}
\label{tab:sequence_kernel_full}
\begin{adjustbox}{max width=\textwidth}
\begin{tabular}{llrr}
\toprule
Dataset & Model & Skewness & Kurtosis \\
\midrule
\multirow{5}{*}{AEMET}
 & DDO/NCSN & $2.73 \pm 0.24 \times 10^{-1}$ & \textcolor{Orange}{\textbf{$3.57 \pm 0.18 \times 10^{-1}$}}  \\
 & DDPM & $1.61 \pm 0.50 \times 10^{-1}$ & $3.95 \pm 1.79 \times 10^{-1}$  \\
 & GANO & $2.34 \pm 0.39 \times 10^{-1}$ & $3.84 \pm 0.83 \times 10^{-1}$  \\
 & FFM & \textcolor{Orange}{\textbf{$5.52 \pm 4.06 \times 10^{-2}$}} & $3.94 \pm 4.16 \times 10^{-1}$  \\
 & kFFM & \cellcolor[gray]{0.9}\textcolor{ForestGreen}{\textbf{$8.12 \pm 5.35 \times 10^{-3}$}} & \cellcolor[gray]{0.9}\textcolor{ForestGreen}{\textbf{$4.22 \pm 2.61 \times 10^{-2}$}}  \\
\midrule
\multirow{5}{*}{Gene Expr.}
 & DDO/NCSN & $2.02 \pm 0.13 \times 10^{-1}$ & \textcolor{ForestGreen}{\textbf{$3.75 \pm 0.44 \times 10^{-1}$}}  \\
 & DDPM & $3.72 \pm 1.17 \times 10^{-1}$ & $1.57 \pm 1.31 \times 10^{1}$  \\
 & GANO & $8.62 \pm 4.56 \times 10^{-1}$ & $2.88 \pm 2.57 \times 10^{0}$  \\
 & FFM & \textcolor{Orange}{\textbf{$1.75 \pm 0.36 \times 10^{-1}$}} & $1.64 \pm 0.57 \times 10^{0}$  \\
 & kFFM & \cellcolor[gray]{0.9}\textcolor{ForestGreen}{\textbf{$7.43 \pm 0.76 \times 10^{-2}$}} & \cellcolor[gray]{0.9}\textcolor{Orange}{\textbf{$4.87 \pm 1.17 \times 10^{-1}$}}  \\
\midrule
\multirow{5}{*}{Economy}
 & DDO/NCSN & $4.91 \pm 1.24 \times 10^{-1}$ & $4.70 \pm 1.87 \times 10^{0}$  \\
 & DDPM & $1.26 \pm 3.23 \times 10^{0}$ & $9.00 \pm 34.10 \times 10^{1}$  \\
 & GANO & $1.05 \pm 0.57 \times 10^{0}$ & $3.78 \pm 1.71 \times 10^{0}$  \\
 & FFM & \textcolor{Orange}{\textbf{$1.80 \pm 0.80 \times 10^{-1}$}} & \textcolor{Orange}{\textbf{$1.75 \pm 0.69 \times 10^{0}$}}  \\
 & kFFM & \cellcolor[gray]{0.9}\textcolor{ForestGreen}{\textbf{$7.65 \pm 2.94 \times 10^{-2}$}} & \cellcolor[gray]{0.9}\textcolor{ForestGreen}{\textbf{$9.77 \pm 4.01 \times 10^{-1}$}}  \\
\midrule
\multirow{5}{*}{Heston}
 & DDO/NCSN & $1.36 \pm 0.01 \times 10^{1}$ & $1.58 \pm 0.00 \times 10^{3}$  \\
 & DDPM & \textcolor{Orange}{\textbf{$3.97 \pm 0.91 \times 10^{0}$}} & \textcolor{Orange}{\textbf{$9.53 \pm 2.04 \times 10^{2}$}}  \\
 & GANO & $6.19 \pm 0.99 \times 10^{0}$ & $1.33 \pm 0.06 \times 10^{3}$  \\
 & FFM & $5.30 \pm 2.79 \times 10^{0}$ & $1.59 \pm 1.12 \times 10^{3}$  \\
 & kFFM & \cellcolor[gray]{0.9}\textcolor{ForestGreen}{\textbf{$3.31 \pm 0.61 \times 10^{0}$}} & \cellcolor[gray]{0.9}\textcolor{ForestGreen}{\textbf{$8.38 \pm 1.60 \times 10^{2}$}}  \\
\bottomrule
\end{tabular}
\end{adjustbox}
\end{table}

\subsection{Implementation Details and Good Practices}\label{app:practices}

Algorithm \ref{alg:mbotcfm} (Section \ref{sec:coupling}) summarizes kFFM training; here we collect the implementation choices behind it.

Unless stated otherwise, the Sinkhorn regularization is $\epsilon=0.05$ on the median-normalized cost matrix and the RBF bandwidth is $\sigma=1$ on the same normalized scale (i.e.\ the median heuristic). The coupling is computed per minibatch with the following choices, which we found to matter for stability and recommend as defaults. (i) \emph{Median-normalize the per-batch cost matrix} before Sinkhorn; this is the most important stability lever and makes the entropic regularization $\epsilon$ scale-free. (ii) Use \emph{log-domain Sinkhorn} \citep{peyre2020computationaloptimaltransport}. (iii) \emph{Sample endpoint pairs from the plan} rather than using the barycentric projection. (iv) \emph{Use a bounded (kernel) cost}: results are then invariant to $\epsilon$ across three orders of magnitude and Sinkhorn converges faster ($15$--$25\%$ faster end-to-end training than with the unbounded $L^2$ cost; Table \ref{tab:compute}); with the raw $L^2$ cost keep $\epsilon\le0.1$ (we observed a $\approx10\times$ degradation at $\epsilon\ge0.5$ on long volatility paths). (v) Per-batch re-coupling was stable on all datasets. (vi) \emph{Batch size.} Minibatch kernel OT induces the expected minibatch plan $\bar\pi_b$ (Section \ref{sec:coupling}); in a sweep over $b\in\{64,128,256,512\}$ on Heston and Stoch.\ KdV, kFFM tracks FFM across the $8\times$ range, and the only seed-consistent edge favors kFFM at the largest batch (Stoch.\ KdV, $b=512$: $3/3$ seeds, $0.11$ vs.\ $0.52\times10^{-3}$ MMD-RBF).

\subsection{Compute and Memory Cost}\label{app:compute}

Table \ref{tab:compute} summarizes the measured cost of the coupling step, profiled on Heston, Stoch.\ KdV, and Navier--Stokes at $64^2$ and $128^2$ for batch sizes $b\in\{64,128,256,512\}$. The end-to-end overhead is $5$--$7\%$ at the largest resolution we train ($128^2$), and because the $O(b^2)$ coupling cost is independent of model size and resolution, its relative weight shrinks as the backbone grows. The bounded RBF cost also conditions the Sinkhorn iterations (the Gibbs kernel $e^{-c_\kappa/\epsilon}$ stays bounded away from $0$ and $1$ across batches; Theorem \ref{thm:regularity}), which is why it is about $8\times$ faster than unbounded-$L^2$ Sinkhorn at $b=512$. Signature-kernel memory grows with path length, which is why the selection rule assigns RBF or $L^2$ costs to very long sequences.

\begin{table}[t]
\small
\centering
\caption{Measured training cost of the kernel-OT coupling step across benchmarks. The $O(b^2)$ Sinkhorn solve is independent of model size and resolution.}
\label{tab:compute}
\begin{tabular}{p{0.30\linewidth}p{0.62\linewidth}}
\toprule
Quantity & Measurement \\
\midrule
End-to-end training overhead vs.\ FFM & $5$--$7\%$ at the largest trained resolution ($128^2$ Navier--Stokes) \\
Peak coupling memory & $80$--$151$\,MB across benchmarks and $b\in\{64,\dots,512\}$; $1.5$--$9\%$ of the FNO forward+backward peak (e.g.\ $116$\,MB vs.\ $7.62$\,GB on Navier--Stokes $64^2$) \\
Per-batch coupling time & $1$--$80$\,ms, about one FNO step \\
Bounded RBF vs.\ unbounded $L^2$ Sinkhorn ($b=512$) & $9$ vs.\ $73$\,ms per batch ($\approx8\times$); $15$--$25\%$ faster end-to-end training \\
Signature-kernel memory & $0.63$\,GB at $100$ time points vs.\ $12.5$\,GB at $512$; the selection rule therefore assigns RBF/$L^2$ costs to very long sequences \\
Larger-model control & Giving FFM the $+6\%$ budget as extra epochs does not close the gap (loss curves plateau early) \\
\bottomrule
\end{tabular}
\end{table}

\subsection{Non-Kernel Metrics}\label{app:nonkernel}

Table \ref{tab:nonkernel} (Section \ref{sec:quality}) reports sliced Wasserstein distance and marginal $W_1$ for kFFM versus FFM on the same runs. Coupling and evaluation kernels already differ for our strongest configurations (signature coupling, RBF-MMD evaluation), and the gains persist under every non-kernel metric. Table \ref{tab:nonkernel_full} adds the autocorrelation and log-spectrum errors on the same matched runs for the sequence datasets; the pointwise diagnostics by kernel variant on the original 10-seed runs are in Appendix \ref{app:detailed}.

\begin{table}[t]
\small
\setlength{\tabcolsep}{5pt}
\centering
\caption{Non-kernel and structural metrics for kFFM-Sig vs.\ FFM on the sequence datasets (same prior and architecture; lower is better), computed on the matched runs of Table~\ref{tab:paired_tests}: mean$\pm$std over 10 shared seeds (Economy averages the population and GDP series within each seed). Every comparison is individually significant when paired by seed (Wilcoxon $p\le0.027$; kFFM better on 9--10 of the 10 seeds on both datasets); the wide autocorrelation spread on Gene Expression is across-seed variation shared by both methods, which the pairing removes. These runs are distinct from the 10-seed runs behind Table~\ref{tab:sequence_kernel}, so the autocorrelation values differ between the two tables; both favor kFFM.}
\label{tab:nonkernel_full}
\begin{tabular}{llcc}
\toprule
Dataset & Metric & FFM & kFFM-Sig \\
\midrule
\multirow{4}{*}{Gene Expr.} & Sliced-$W$ & $0.106\pm0.006$ & $\mathbf{0.087\pm0.006}$ \\
 & Marginal-$W_1$ & $0.074\pm0.005$ & $\mathbf{0.059\pm0.004}$ \\
 & Autocorr.\ MSE ($\times10^{-5}$) & $9.3\pm6.8$ & $\mathbf{3.3\pm3.6}$ \\
 & log-Spectrum MSE ($\times10^{-2}$) & $1.76\pm0.42$ & $\mathbf{0.86\pm0.22}$ \\
\midrule
\multirow{4}{*}{Economy} & Sliced-$W$ & $0.0291\pm0.0019$ & $\mathbf{0.0233\pm0.0012}$ \\
 & Marginal-$W_1$ & $0.0171\pm0.0016$ & $\mathbf{0.0137\pm0.0009}$ \\
 & Autocorr.\ MSE ($\times10^{-4}$) & $1.59\pm0.39$ & $\mathbf{1.12\pm0.35}$ \\
 & log-Spectrum MSE & $1.09\pm0.07$ & $\mathbf{0.97\pm0.04}$ \\
\bottomrule
\end{tabular}
\end{table}

\subsection{Turbulent Navier--Stokes: Samples and Statistics}\label{app:turbulent}

Table \ref{tab:turbulent_ns_full} gives the turbulent benchmark of Section \ref{sec:turbulent} with explicit configurations and unscaled values, and Figure \ref{fig:turbulent_ns} shows vorticity samples (pooled vorticity PDFs and energy spectra are in Figure \ref{fig:turbulent_stats}). We also re-ran the standard Navier--Stokes configurations of Section \ref{sec:quality} with the same physics diagnostics (3 seeds): the flow-matching family leads (enstrophy $W_1$ of $0.16$ vs.\ $0.28$ for DDPM; vorticity-PDF $W_1$ of $0.14$--$0.15$ vs.\ $0.39$), and kFFM matches FFM within noise on the physics metrics while retaining its distributional edge from Table \ref{tab:kernel_vs_l2_main}. The kernel-cost mismatch is larger in the turbulent regime: prior-data $L^2$ distances more than double relative to the standard benchmark (median $91$ vs.\ $40$), pushing the RBF cost deeper into saturation, the regime where Theorem \ref{thm:approx_error_correct} is weakest; the empirical results carry the method there.

\begin{table}[t]
\scriptsize
\setlength{\tabcolsep}{3pt}
\centering
\caption{Full version of Table \ref{tab:turbulent_ns} with explicit configurations (cost/base measure) and unscaled values. Turbulent 2D Navier--Stokes ($\nu=10^{-5}$, $1200$ trajectories, $\mathrm{Re}\approx2000$): distributional and physics diagnostics (lower is better), mean$\pm$std over 5 seeds, identical FNO backbone and training budget for all flow/diffusion methods. Best per column in bold (on MMD, kFFM and CFM-OT($L^2$) tie within std). Vort-PDF $W_1$: $W_1$ between pooled vorticity distributions; Enstrophy $W_1$: $W_1$ between per-snapshot enstrophy distributions; Skew/Kurt: absolute errors of pooled-vorticity skewness and excess kurtosis; log-Spec: log-spectral error.}
\label{tab:turbulent_ns_full}
\begin{adjustbox}{max width=\linewidth}
\begin{tabular}{lccccccc}
\toprule
Method & MMD & Sliced-W & Vort-PDF $W_1$ & Enstrophy $W_1$ & Skew err. & Kurt err. & log-Spec \\
\midrule
kFFM (selected: Euclidean-RBF/wn) & $\mathbf{0.00148\pm0.0006}$ & $\mathbf{0.193\pm0.01}$ & $\mathbf{0.0375\pm0.01}$ & $\mathbf{0.155\pm0.02}$ & $\mathbf{0.0027\pm0.003}$ & $0.0777\pm0.01$ & $\mathbf{0.0204\pm0.008}$ \\
CFM-OT($L^2$) (raw $L^2$/wn) & $\mathbf{0.00147\pm0.001}$ & $0.194\pm0.01$ & $0.0408\pm0.01$ & $0.157\pm0.01$ & $0.00404\pm0.002$ & $0.0752\pm0.01$ & $0.0223\pm0.007$ \\
kFFM-Sob (Sobolev-RBF/gp) & $0.111\pm0.02$ & $0.612\pm0.05$ & $0.104\pm0.03$ & $0.194\pm0.05$ & $0.0456\pm0.02$ & $\mathbf{0.0679\pm0.05}$ & $0.0652\pm0.01$ \\
kFFM-Euc (raw $L^2$/gp) & $0.109\pm0.02$ & $0.623\pm0.03$ & $0.104\pm0.04$ & $0.197\pm0.06$ & $0.0495\pm0.03$ & $0.0774\pm0.04$ & $0.067\pm0.01$ \\
FFM (gp, independent) & $0.108\pm0.02$ & $0.601\pm0.02$ & $0.107\pm0.04$ & $0.201\pm0.06$ & $0.0487\pm0.03$ & $0.0723\pm0.05$ & $0.0646\pm0.01$ \\
DDO/NCSN & $0.569\pm0.009$ & $1.39\pm0.06$ & $1.22\pm0.01$ & $1.07\pm0.02$ & $0.0083\pm0.01$ & $1.29\pm0.01$ & $0.937\pm0.02$ \\
DDPM & $0.38\pm0.02$ & $1.18\pm0.1$ & $0.848\pm0.03$ & $0.889\pm0.02$ & $0.0789\pm0.05$ & $3.37\pm2$ & $0.117\pm0.03$ \\
\bottomrule
\end{tabular}
\end{adjustbox}
\end{table}

\begin{figure}[t]
\centering
\includegraphics[width=0.62\linewidth]{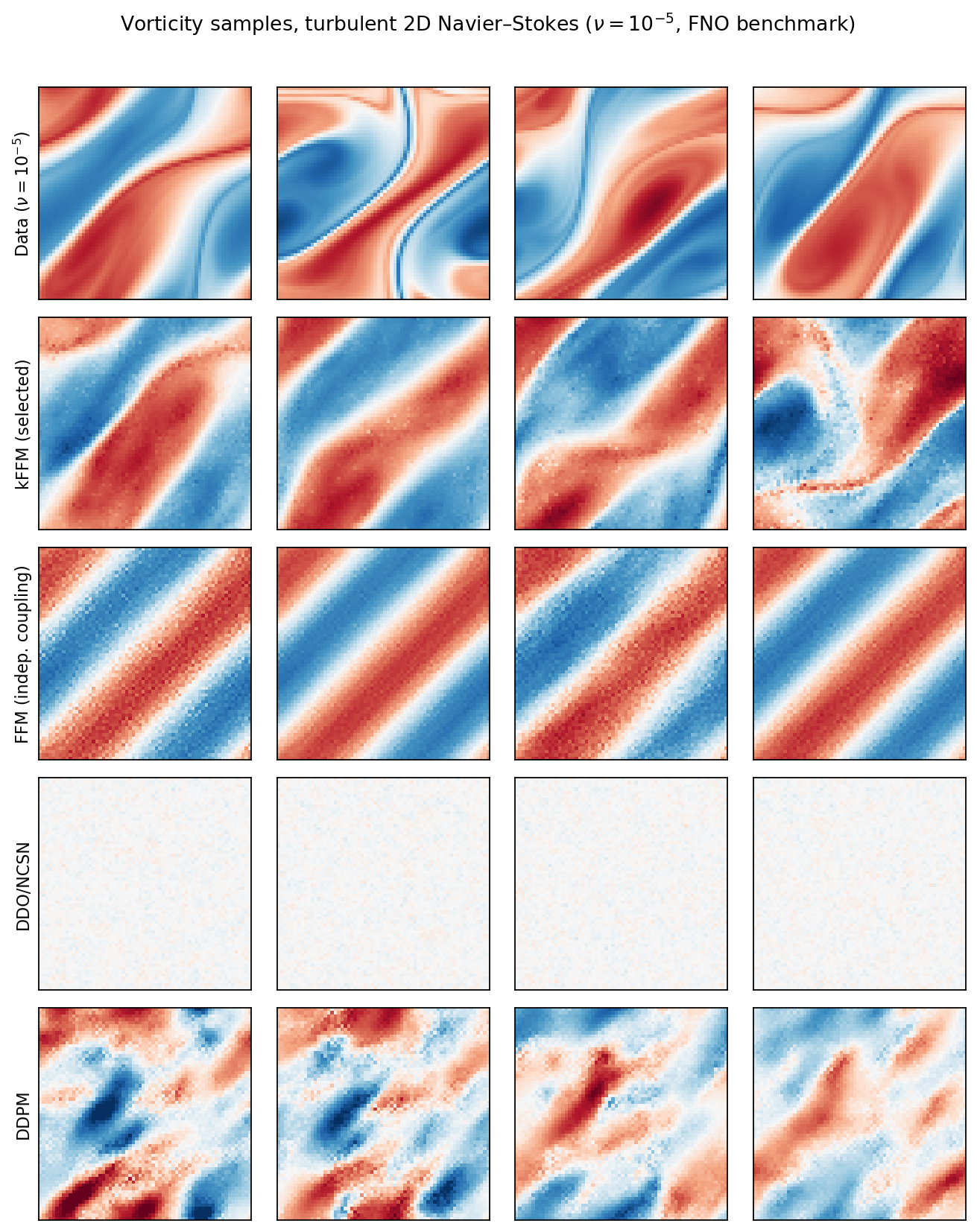}
\caption{Turbulent 2D Navier--Stokes ($\nu=10^{-5}$): vorticity samples on a shared color scale (for each method, the samples whose enstrophy is closest to the data median). kFFM (selected) and FFM reproduce the data's large-scale vortical organization, DDO/NCSN produces near-noise fields, and DDPM coarse blobs. The corresponding pooled vorticity PDFs and energy spectra are in Figure \ref{fig:turbulent_stats}.}
\label{fig:turbulent_ns}
\end{figure}

\subsection{Measuring the Kernel-Cost Mismatch}\label{app:mismatch}

Analytically, $c_{\mathrm{RBF}}(f,g)=2(1-e^{-\|f-g\|^2/2\sigma^2})$ is a rescaled $L^2$ cost up to relative error $O((D/\sigma)^2)$, with $D$ the batch diameter; median normalization cancels the rescaling, so $\Delta_\kappa$ is small at large bandwidth. We measured the mismatch on training-faithful batches through the rank correlation between kernel and $L^2$ cost matrices and the total variation (TV) distance between the induced plans. (1) Where the RBF cost is not saturated it is a monotone transform of $L^2$ (rank correlation $\approx1$: Gene Expression $1.000$, Heston $0.98$), i.e.\ the same geometry with compressed outliers. (2) The signature kernel genuinely re-ranks pairs (rank correlation with $L^2$ of only $0.33$--$0.37$ on AEMET and Gene Expression), the data class where its wins are largest and most significant; $\Delta_\kappa$ is large there by design, and Theorem \ref{thm:approx_error_correct} controls the deviation from Wasserstein with the chosen cost rather than proximity to $L^2$ OT. (3) On smooth Navier--Stokes fields the plans agree across costs (plan TV $\le0.02$), the small-$\Delta_\kappa$ regime where the bound is tightest and the bounded cost's stability and speed are the operative benefits.

\subsection{Synthetic Couplings: Where \texorpdfstring{$L^2$}{L2} Pairing Fails}\label{app:toy}

Figure \ref{fig:toy_coupling} isolates the two failure modes of $L^2$ pairing discussed in Section \ref{sec:quality} on synthetic data at the training coupling protocol (stable over 6 seeds): a conditioning failure under heavy-tailed nuisance bursts, cured by any bounded cost, and a geometry failure on paths matched by volatility, cured only by the signature kernel.

\begin{figure}[t]
\centering
\includegraphics[width=\linewidth]{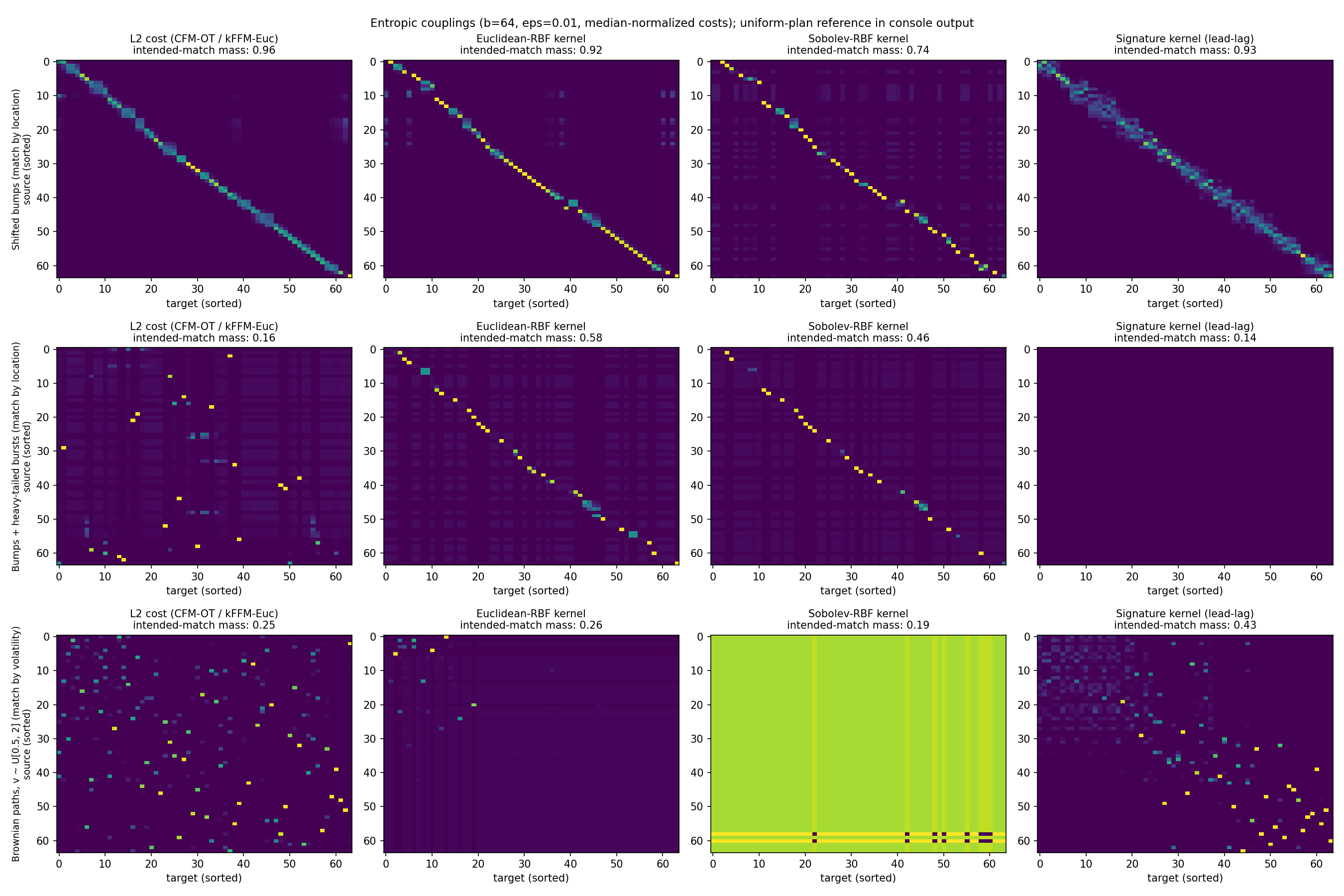}
\caption{Entropic minibatch couplings ($b=64$, $\epsilon=0.01$, median-normalized costs as in training) for three synthetic scenarios (rows) and four costs (columns). Sources and targets are sorted so that the intended matching is the diagonal, and each panel reports the plan mass on intended matches (a uniform plan gives $\approx0.14$--$0.19$). Top: shifted bumps matched by location; every cost recovers the matching. Middle: bumps where $40\%$ of samples carry large, independent nuisance bursts. Burst pairs dominate the median normalization of the unbounded $L^2$ cost, collapsing clean-pair contrasts below $\epsilon$, so its plan degenerates to near uniform ($0.16$); the bounded RBF cost saturates the bursts and recovers the matching ($0.58$), invariant to burst amplitude (the conditioning effect of Theorem~\ref{thm:regularity}). Bottom: Brownian paths $f=vW(t)$, $v\sim U[0.5,2]$, matched by volatility. The expected $L^2$ cost is separable, $\mathbb{E}\|f-g\|^2=(v^2+w^2)\int t\,dt$, so it carries no assortative signal in expectation and its plan stays near uniform ($0.25$); the signature kernel (lead-lag) reads volatility through quadratic variation and concentrates on the intended matching ($0.43$). No single cost wins every scenario, which motivates the per-data-class selection rule of Appendix~\ref{app:selection}.}
\label{fig:toy_coupling}
\end{figure}

\subsection{Why Straightness Arguments Do Not Directly Transfer}\label{app:straightness}

Finite-dimensional OT-CFM is often motivated by the Euclidean $W_2$ picture: a better coupling, combined with linear interpolation, can lead to straighter paths and hence lower numerical integration cost. That intuition depends on geometric facts that are specific to the finite-dimensional Wasserstein setting and should not be treated as automatic in function space.

\begin{proposition}[Coupling-only role of kernel OT]\label{prop:coupling_only}
Fix a coupling $\pi \in \Pi(\mu_0,\mu_1)$. In kFFM, conditional on $(f_0,f_1)\sim \pi$, the path is
\[
g_t^{f_0,f_1} = t f_1 + \sigma_t f_0,\qquad \sigma_t=1-(1-\sigma_{\min})t,
\]
so that
\[
\partial_t g_t^{f_0,f_1} = f_1-(1-\sigma_{\min})f_0,\qquad \partial_t^2 g_t^{f_0,f_1}=0:
\]
the conditional path is affine in $t$ and deterministic given the pair, and its geometry does not depend on $\pi$. Moreover, the training objective can be written as
\[
\mathcal L_{CFM}^{\pi}(\theta)= \mathbb E_{t,\,(f_0,f_1)\sim \pi}
\left\| \partial_t g_t^{f_0,f_1}- v_\theta\big(g_t^{f_0,f_1},t\big)\right\|_{\mathcal F}^2 .
\]
Therefore, optimizing over $\pi$ changes only the distribution of endpoint pairs presented to the model; it does not introduce any explicit pathwise straightness, geodesic, or action-minimization term.
\end{proposition}

\begin{proof}
The path and its velocity are exactly those of Section \ref{sec:coupling}; they coincide with the FFM path $tf+\sigma_t\xi$ and conditional field $f-(1-\sigma_{\min})\xi$ of Section \ref{sec:ffm_background} with the base noise $\xi$ replaced by the paired prior sample $f_0$. Differentiating the affine path gives the displayed derivatives. Substituting the conditional path into the conditional flow matching objective yields the displayed loss. Since $\pi$ enters only through the sampling law of $(f_0,f_1)$ and no additional functional of the path $t\mapsto g_t^{f_0,f_1}$ appears in the objective, the method does not explicitly optimize any global path-geometry quantity.
\end{proof}

\begin{corollary}[Coupling-invariance of ambient path straightness]\label{cor:straightness_invariance}
For a twice differentiable path $m:[0,1]\to\mathcal F$, define the ambient acceleration and speed-variation functionals
\[
\mathsf A(m):=\int_0^1 \|\partial_t^2 m_t\|_{\mathcal F}^2\,dt,\qquad
\mathsf V(m):=\int_0^1\left(\|\partial_t m_t\|_{\mathcal F}-\int_0^1 \|\partial_s m_s\|_{\mathcal F}\,ds\right)^2dt.
\]
Then for every coupling $\pi\in \Pi(\mu_0,\mu_1)$ and every endpoint pair $(f_0,f_1)\sim \pi$, the kFFM conditional path $g^{f_0,f_1}$ satisfies
\[
\mathsf A(g^{f_0,f_1})=0,\qquad \mathsf V(g^{f_0,f_1})=0.
\]
Hence these natural ambient-space straightness measures of the prescribed conditional path are identical for all couplings.
\end{corollary}

\begin{proof}
By Proposition \ref{prop:coupling_only}, $\partial_t g_t^{f_0,f_1}=f_1-(1-\sigma_{\min})f_0$ is constant in $t$ and $\partial_t^2 g_t^{f_0,f_1}=0$. Therefore $\mathsf A(g^{f_0,f_1})=0$. The speed $\|\partial_t g_t^{f_0,f_1}\|_{\mathcal F}=\|f_1-(1-\sigma_{\min})f_0\|_{\mathcal F}$ is also constant in $t$, so it coincides with its own time average and $\mathsf V(g^{f_0,f_1})=0$.
\end{proof}

Proposition \ref{prop:coupling_only} and Corollary \ref{cor:straightness_invariance} formalize the internal reason that the usual OT-CFM straightness argument does not transfer directly to kFFM: kernel OT acts only at the level of endpoint coupling, while the interpolation template remains the same affine FFM path. Any improvement in convergence or sampling cost could therefore only arise indirectly through better endpoint pairings and the learned velocity field, not from an explicit path-straightness objective, and we do not claim one.

Our setting is different in a second, geometric way. FFM is formulated on Gaussian measures over separable Hilbert spaces, where absolute continuity and transport already depend on Cameron--Martin and Feldman--Hajek type conditions rather than Lebesgue-density arguments. Even when one restricts attention to Gaussian optimal transport, the appropriate geometry is the Bures--Wasserstein geometry rather than the flat Euclidean geometry used by standard OT-CFM heuristics.

Recent work by \citet{yun2025gaussianoptimaltransportbreniers} makes this distinction explicit. For Gaussian measures on separable Hilbert spaces, they show that the optimal transport map can fail to admit the usual Brenier variational interpretation: formally, it is the subgradient of a convex function that is infinite almost everywhere, and the geodesic structure between degenerate Gaussian measures is richer than the classical finite-dimensional McCann interpolation picture. This does not mean that efficient flows are impossible in function space; rather, it means that the usual ``OT implies straighter paths'' argument is not a theorem in the present setting.

For this reason, our paper does not claim that kFFM is geodesic or pathwise optimal in a Bures--Wasserstein sense. The theoretical guarantees in this paper concern the well-posedness, error decomposition, and discretization invariance of the kernel OT surrogate, not global straightness of the learned ODE.

This is also why Section \ref{sec:exp} makes no efficiency claim: the main supported claim of the paper is better distributional matching from geometry-aware coupling, and Appendix \ref{app:convergence} reports only the training-time convergence of the target metric.

\subsection{Target-Metric Convergence During Training}\label{app:convergence}

In finite-dimensional OT-CFM, optimal transport is often motivated by improved convergence and straighter flows \citep{tong2024improvinggeneralizingflowbasedgenerative}. As discussed in Appendix \ref{app:straightness}, we do not treat such properties as theoretical consequences in function space and make no efficiency claim. We report only the convergence of the target metric during training: Table \ref{tab:convergence_summary} and Figure \ref{fig:convergence_main} show the MMD-RBF of FFM and kFFM over the final phase of training, with both models sharing the GP prior and FNO backbone so that the only change is the coupling. The curves are the best-so-far value (cumulative minimum) of the across-seed mean; on KdV, where the unbiased estimate is noisy at the level of single checkpoints, the curves are instead the running mean over the final phase. kFFM ends training at a better MMD-RBF on every dataset, with the largest gains on Stoch.\ KdV, Gene Expr., AEMET, Economy, and Heston.

\begin{table}[t]
\small
\setlength{\tabcolsep}{4pt}
\centering
\caption{Relative MMD-RBF improvement of kFFM over FFM at the end of the convergence curves of Figure \ref{fig:convergence_main} (higher is better; values above $100\%$ occur where the unbiased estimate for kFFM is negative). Both methods share the GP prior and FNO backbone, isolating the coupling as the only changed component.}
\label{tab:convergence_summary}
\resizebox{\linewidth}{!}{%
\begin{tabular}{lcccccccc}
\toprule
Dataset & Heston & AEMET & KdV & Economy & Gene Expr. & Stoch.\ KdV & Stoch.\ NS & Navier--Stokes \\
\midrule
$\Delta$MMD-RBF (\%) & $\mathbf{+32\%}$ & $\mathbf{+53\%}$ & $\mathbf{+8\%}$ & $\mathbf{+52\%}$ & $\mathbf{+64\%}$ & $\mathbf{+225\%}$ & $\mathbf{+7\%}$ & $\mathbf{+9\%}$ \\
\bottomrule
\end{tabular}}
\end{table}

\begin{figure}[t]
\centering
\includegraphics[width=\linewidth]{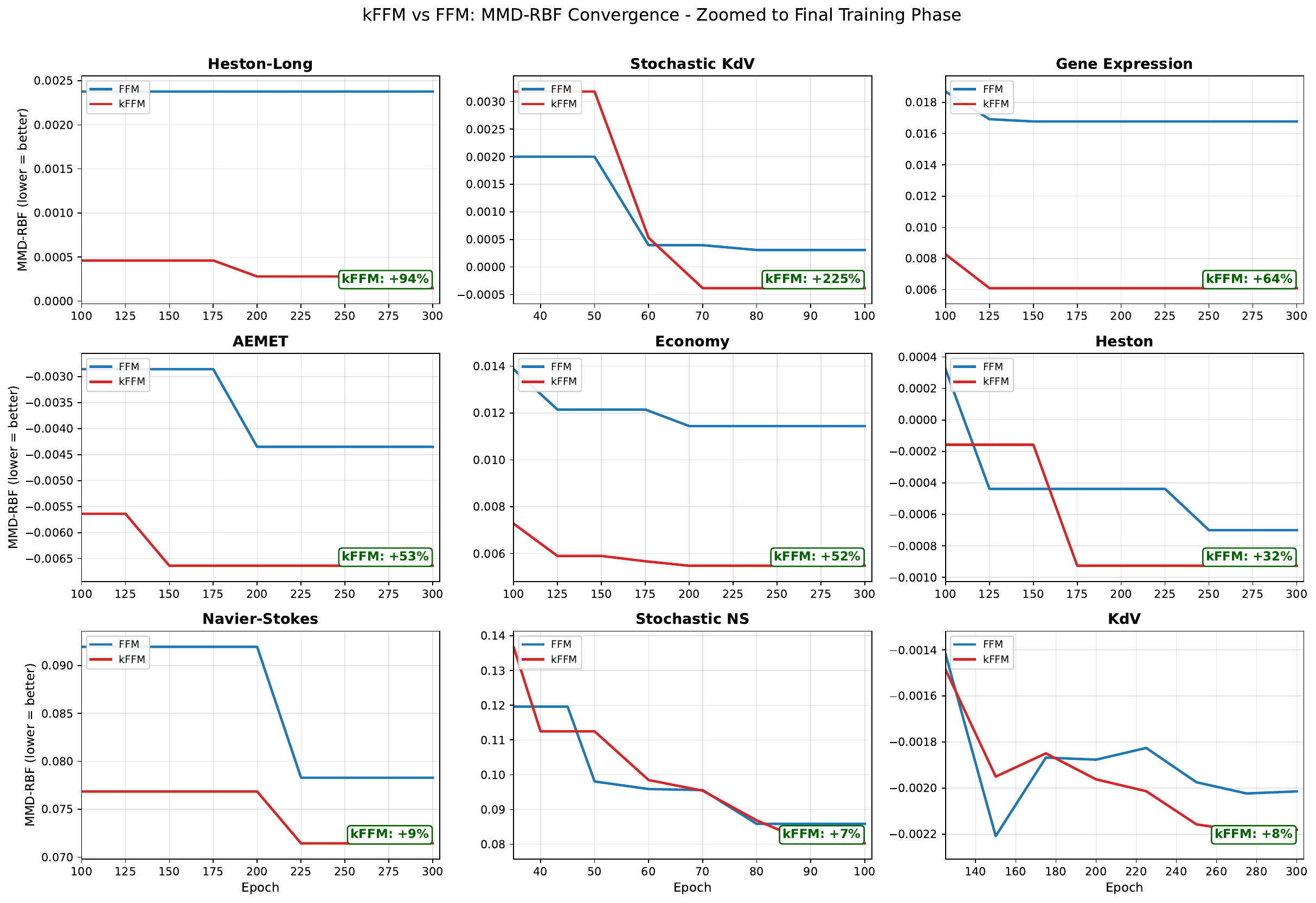}
\caption{Per-dataset target-metric convergence using MMD-RBF (lower is better) corresponding to the headline summary in Table \ref{tab:convergence_summary}. The plot shows the best-so-far across-seed mean MMD-RBF over the final phase of training (on KdV, the running mean over the final phase, from epoch 125) for kFFM versus the original FFM under matched GP prior and FNO backbone, so the only change is the coupling. The kFFM curve uses the signature kernel on Stoch.\ KdV, Gene Expr., AEMET, Economy, and Heston, the Sobolev-RBF cost on KdV, the Euclidean--RBF cost on Stoch.\ NS, and the raw $L^2$ cost on Navier--Stokes and Heston-Long. kFFM ends training at a better MMD-RBF on every dataset, with especially large gains on Stoch.\ KdV, Gene Expr., AEMET, Economy, and Heston. The additional Heston-Long panel is a long-horizon variant of Heston ($1000$ time steps) used only in this diagnostic and in the sensitivity study; it is not one of the eight benchmark datasets.}
\label{fig:convergence_main}
\end{figure}

\subsection{Hyperparameter Sensitivity}\label{app:sensitivity}

Unless otherwise noted, we keep the training setup, neural architectures, Gaussian prior, and remaining OT settings fixed at the default values from the main experiments, and vary only one hyperparameter at a time. We use the \texttt{POT}\footnote{\url{https://pythonot.github.io/}} package \citep{flamary2021pot} for optimal transport. The figures below are direct one-at-a-time sweeps: they vary $\sigma$ or $\epsilon$ while holding the other settings fixed. We additionally report the sensitivity to the Gaussian-prior length scale on the 2D PDE datasets, where it has the clearest effect.

The kernel-specific hyperparameters are: 
\begin{enumerate}
    \item RBF kernel: the bandwidth $\sigma\in\{0.1, 0.2, 0.5, 1, 2, 5, 10\}$.
    \item Signature kernel: 
    \begin{enumerate}
        \item Boolean switches for the lead-lag and time augmentations (\verb|lead_lag|, \verb|time_aug|).
        \item Dyadic order in $\{0, 1, 2, 3\}$.
        \item Static-kernel bandwidth between $0.1$ and $5$.
        \item Maximum sequence length (subsampling) in $\{32, 50, 64, 128\}$.
    \end{enumerate}
\end{enumerate}

The hyperparameters for the Sinkhorn algorithm are:
\begin{enumerate}
    \item \verb|ot_reg|, the regularization parameter $\epsilon$; the sweeps of Figure \ref{fig:sensitivity_sweeps} use $\{0.01, 0.05, 0.1, 0.5, 1.0\}$, and the full grids extend this range.
    \item \verb|ot_method|: Sinkhorn or exact OT; Sinkhorn generally performs better.
    \item \verb|ot_coupling|: sampling from the plan (default) or barycentric projection; sampling usually performs better.
\end{enumerate}

Figure \ref{fig:sensitivity_summary} summarizes the aggregate MMD-RBF spread across datasets for the main tunable kernel families, where the spread for a given dataset is computed as $(\max - \min)/|\mathrm{mean}|$ over the tested settings. The dominant pattern is that the Sinkhorn regularization is typically quite stable, while the RBF bandwidth can matter more on selected datasets. Figure \ref{fig:sensitivity_sweeps} then shows the direct sweeps on representative datasets, run at the white-noise base measure: the top panel varies the RBF bandwidth, and the bottom panel varies the Sinkhorn regularization. We focus these direct sweeps on the axes that showed the most variation; the signature-kernel settings were substantially flatter on the tested sequence datasets and are summarized in Figure \ref{fig:sensitivity_summary}. Re-aggregating the full grids: results are indistinguishable within seed-to-seed variation for RBF bandwidths $\sigma\in[0.1,2]$ on every dataset (the wider sweeps in Figure \ref{fig:sensitivity_sweeps} show where larger bandwidths start to matter on selected datasets); the Sinkhorn regularization is flat across three orders of magnitude ($\epsilon\in[0.001,1]$) for bounded costs, whereas with the raw $L^2$ cost $\epsilon\ge0.5$ degrades long volatility paths by $\approx10\times$; and the signature hyperparameters (dyadic order $0$--$3$, static bandwidth $0.1$--$5$) leave the results unchanged. The kernel \emph{family} and the base measure are the choices that matter: Euclidean costs, raw or bounded, still lose to the signature kernel by $2.6$--$2.9\times$ on Gene Expression and Economy (Section \ref{sec:quality}).

\begin{figure}[t]
\centering
\includegraphics[width=0.82\textwidth]{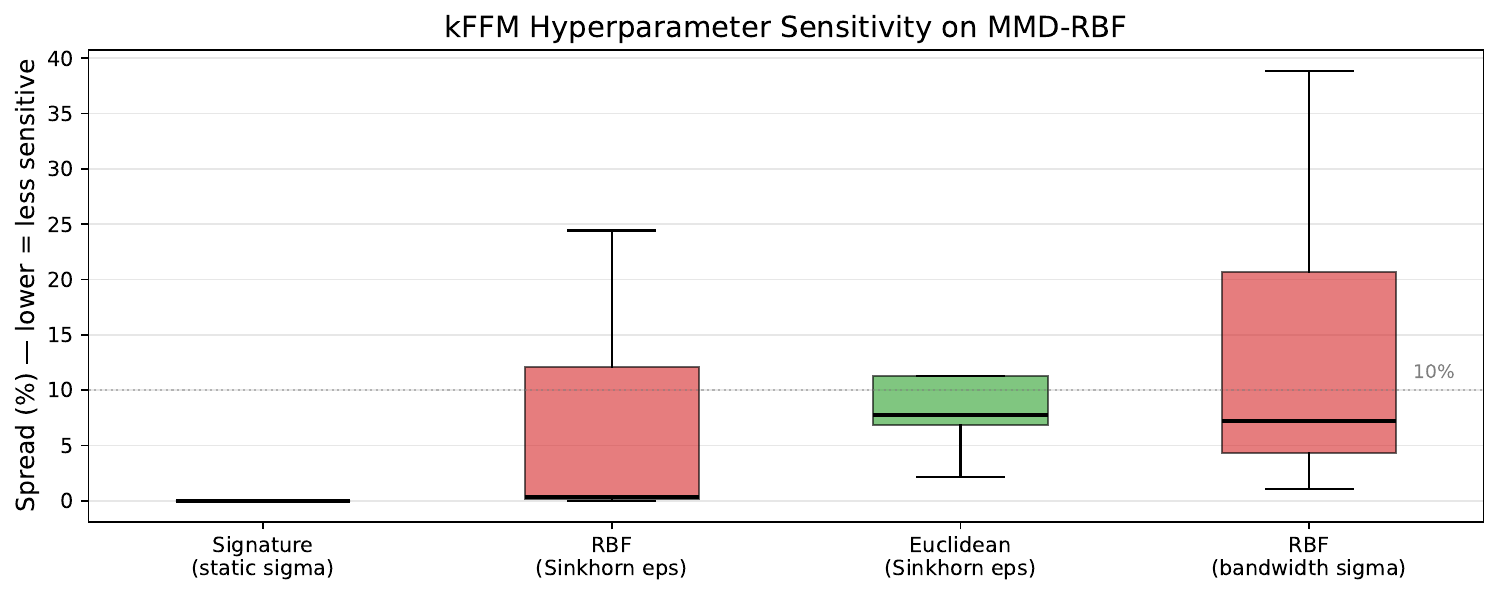}
\caption{Aggregate MMD-RBF sensitivity of kFFM hyperparameters across datasets; the spread is $(\max-\min)/|\mathrm{mean}|$ over the tested settings of the hyperparameter in parentheses, and lower is less sensitive. Signature-kernel settings are nearly invariant on the tested sequence datasets, the Sinkhorn regularization is typically less influential than the RBF bandwidth, and the results depend more on $\epsilon$ under the raw $L^2$ (Euclidean) cost than under the bounded RBF cost (median spread of about $8\%$ vs.\ near zero).}
\label{fig:sensitivity_summary}
\end{figure}

\begin{figure}[p]
\centering
\includegraphics[width=\textwidth]{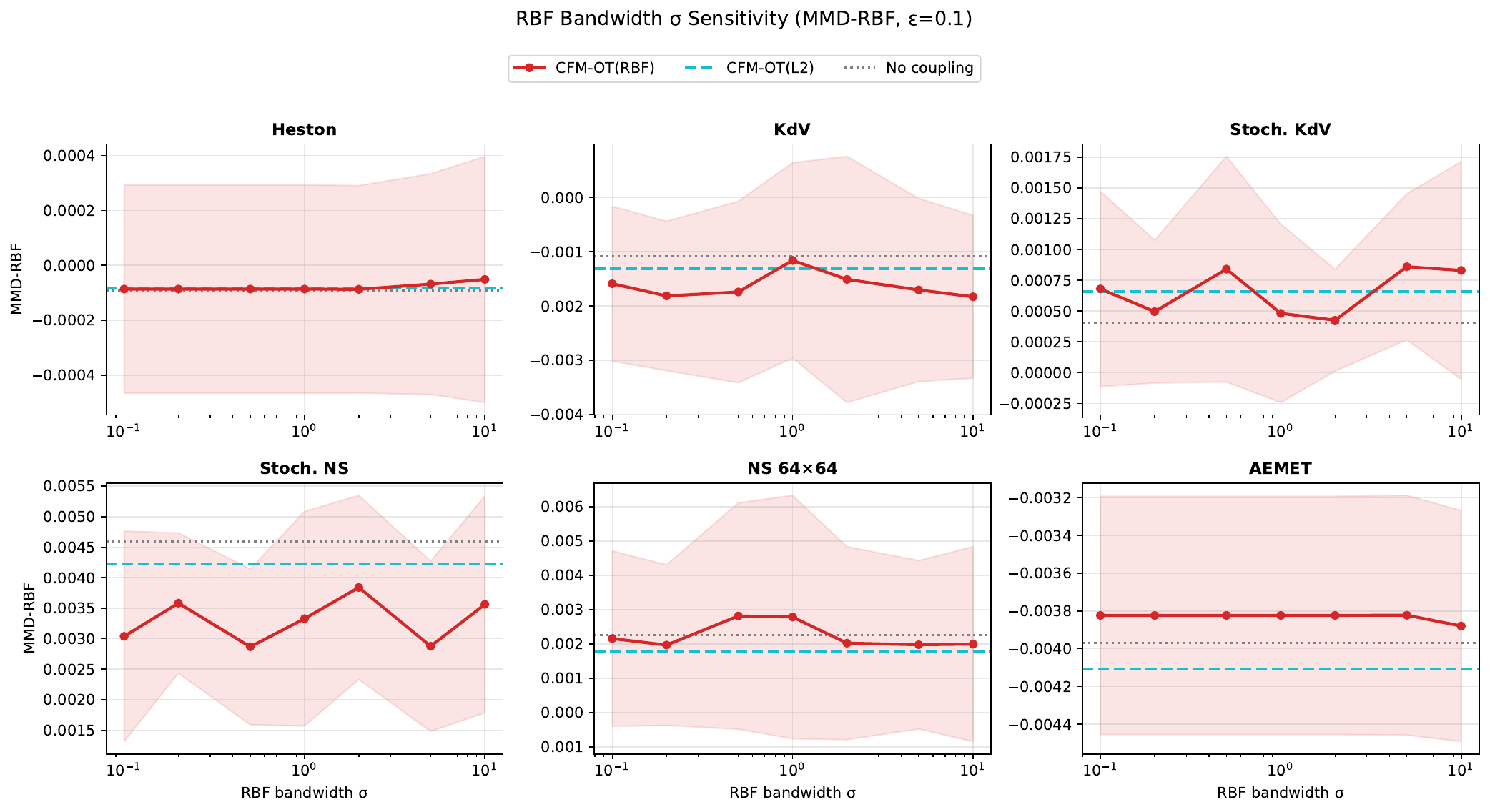}
\vspace{6pt}
\includegraphics[width=\textwidth]{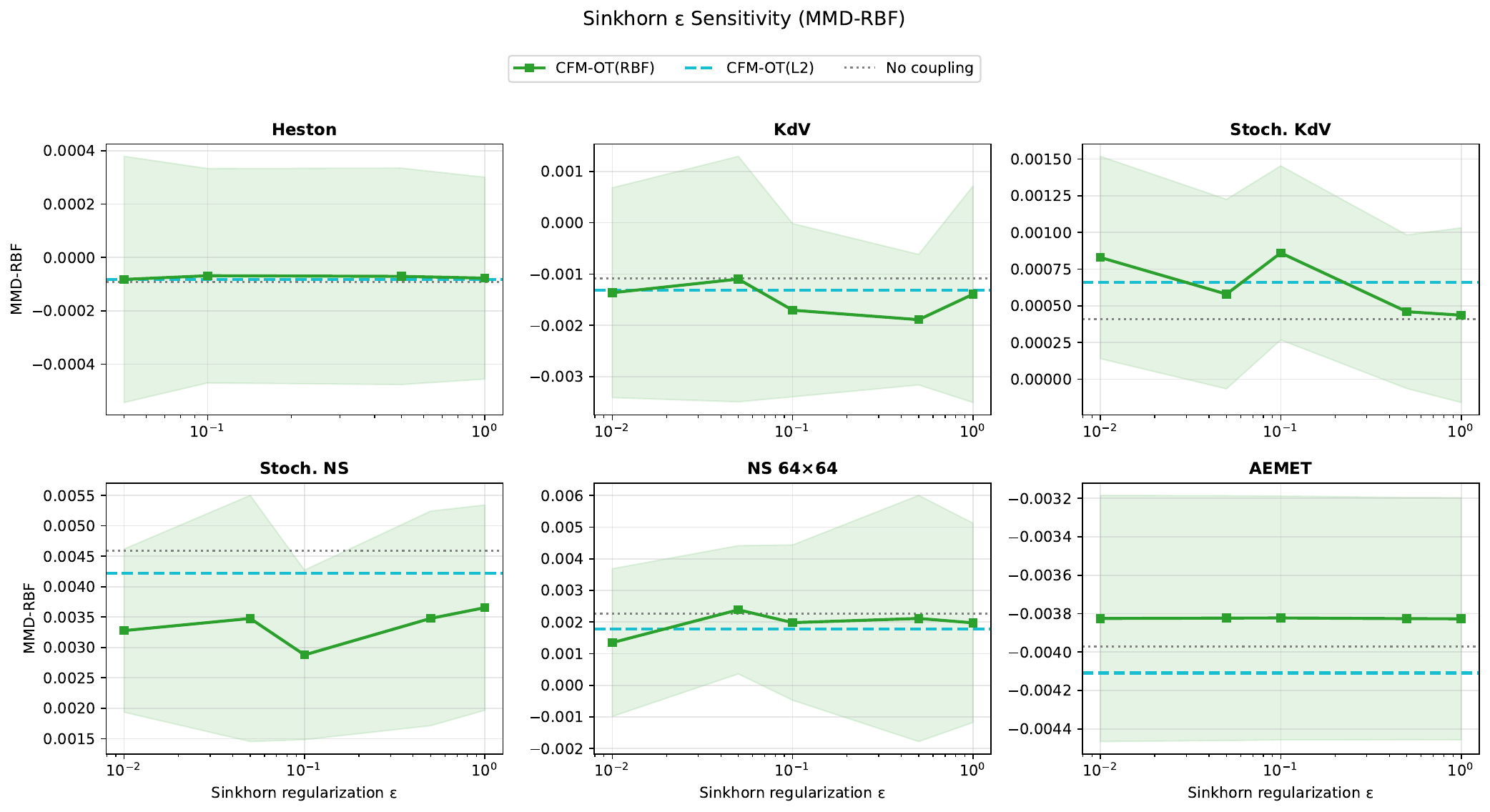}
\caption{Direct one-at-a-time MMD-RBF sweeps on representative datasets (mean and one standard deviation across seeds). Top: varying the RBF bandwidth $\sigma$ at fixed Sinkhorn regularization ($\epsilon=0.1$). Bottom: varying $\epsilon$ at fixed bandwidth. All arms in this figure use the white-noise base measure: the solid curve (legend: CFM-OT(RBF)) is the Euclidean--RBF kernel coupling, the dashed line is CFM-OT($L^2$), and the dotted line (legend: No coupling) is independent coupling at the same base. Across the tested values the swept arm moves within its seed-to-seed variation.}
\label{fig:sensitivity_sweeps}
\end{figure}

The choice of Gaussian prior matters more on the 2D PDE datasets (Navier--Stokes and Stoch.\ NS), independently of the kernel. Figure \ref{fig:gp_sensitivity} shows the sensitivity of these two datasets to the length scale of the Gaussian prior, which encodes a prior belief about the smoothness of the data.

\begin{figure}[t]
    \centering
    \includegraphics[width=\linewidth]{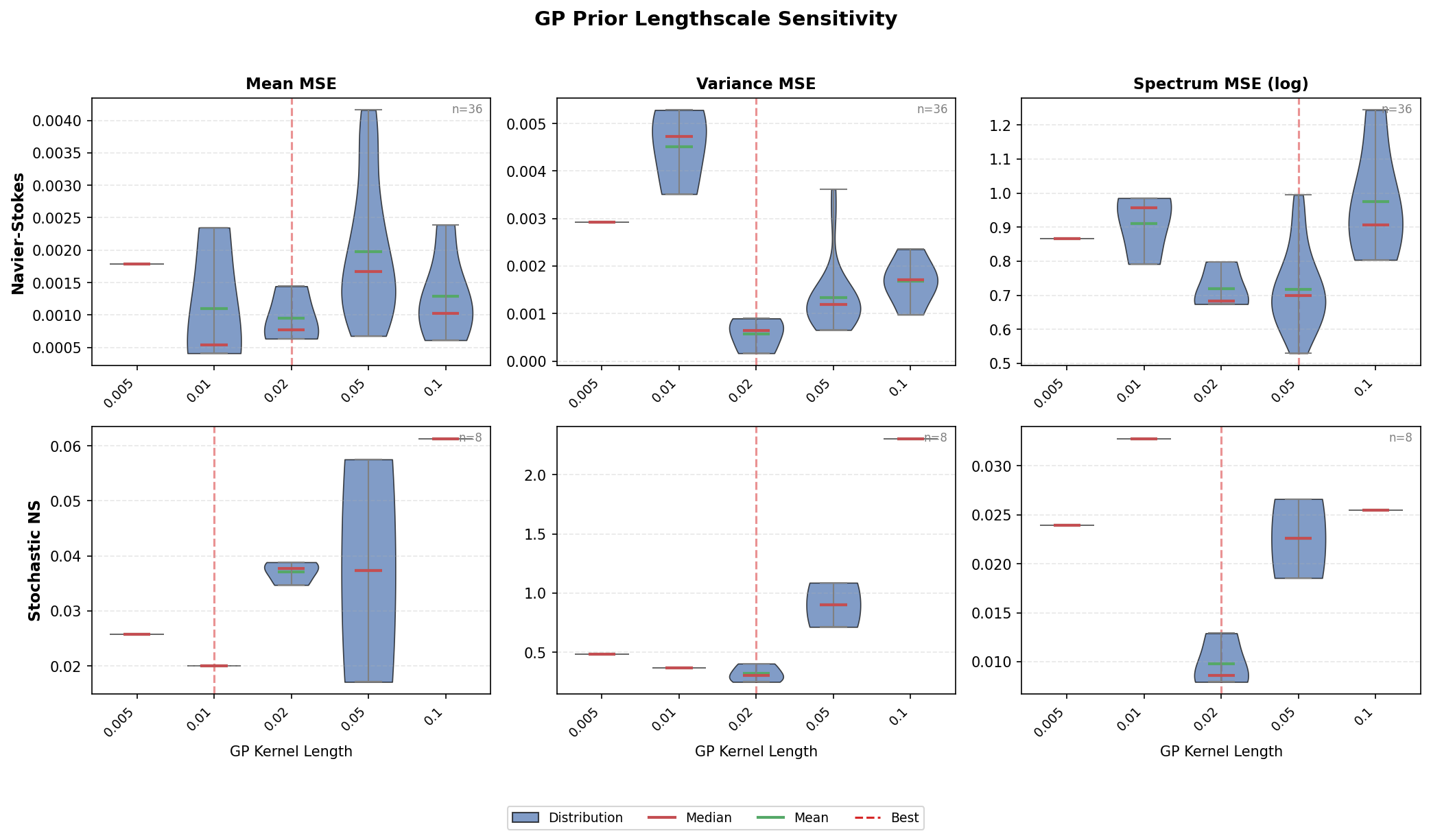}
    \caption{Distribution of the pointwise metrics (mean MSE, variance MSE, log-spectrum error) on the two 2D PDE datasets, Navier--Stokes and Stoch.\ NS, as a function of the length scale of the Gaussian prior.}
    \label{fig:gp_sensitivity}
\end{figure}

\section{Extended Related Work}\label{app:d}

This section expands the discussion of Section \ref{sec:related}. The study of using optimal transport \citep{Villani2008OptimalTO} to improve flow matching was first proposed by \citet{tong2024improvinggeneralizingflowbasedgenerative, pooladian2023multisampleflowmatchingstraightening}, which add an OT-induced coupling of each minibatch before the interpolation and training step. Several follow-up works \citep{kornilov2024optimalflowmatchinglearning, yue2025oatfmoptimalaccelerationtransport} have aimed to further expand the framework. All have assumed that the sample space is $\mathbb{R}^d$. In order to better capture the resolution-invariance of grid data and continuity of path data, \citet{kerrigan2023functionalflowmatching} first proposed Functional Flow Matching, a framework to perform flow matching in Hilbert space, relying on Gaussian measure theory. However, OT-based endpoint coupling had not been brought to this space.

The line of research that embeds probability measures in RKHS has been proposed by the pioneering work of \citet{Sriperumbudur2010Hilbert, gretton2008kernelmethodtwosampleproblem}, where kernel-induced metrics, such as Integral Probability Metrics (IPM), have been studied extensively \citep{Dudley_2002, Sriperumbudur2010Hilbert} and used for two-sample  hypothesis testing \citep{gretton2008kernelmethodtwosampleproblem}, most recently in path space \citep{chevyrev2022signaturemoments} with the signature kernel. The IPM in the form of the Maximum Mean Discrepancy (MMD) has been linked with both probability divergences and Wasserstein distances \citep{feydy2018interpolatingoptimaltransportmmd}, and it has an entropic OT formulation in the Hilbert Sinkhorn Divergence of \citet{Li_2021_CVPR}. It has also been used as the objective of gradient flows \citep{galashov2024deepmmdgradientflow, arbel2019maximummeandiscrepancygradient}, which have been adopted in deep generative models. The extension of these frameworks to function space is an active area of research, and to our knowledge our work is the first to explore it in the context of flow matching.

\paragraph{Concurrent function-space flow matching.} Three concurrent works are closest to ours; none studies the endpoint coupling. \citet{kollovieh2025flowmatchinggp} use flow matching with GP priors for probabilistic time-series forecasting, a conditional task with an independent prior-data coupling; this is closest in spirit to our GP-prior setup and orthogonal on the axis we study, and our results suggest that framework could itself benefit from kernel-OT coupling. \citet{zhang2025functionalrectifiedflow} straighten flows in Hilbert space via rectification, iteratively re-training on model-generated endpoint pairs; rectification changes pairs across training rounds using the model, whereas kernel OT changes pairs within each batch using data geometry in a single run and extends beyond Hilbert spaces (signature kernel on path space), so the two can in principle be composed. \citet{lee2025operatorflowmatching} propose operator flow matching for time-series forecasting, again conditional prediction with independent coupling. On the finite-dimensional side, \citet{fatras2020minibatchwasserstein, fatras2021minibatchot} analyze minibatch OT plans and the expected-minibatch coupling that Section \ref{sec:coupling} relies on.

\ifarxiv\else
\clearpage
\input{sections/checklist}
\fi

\end{document}